\documentclass{article}

    \usepackage[final]{neurips_2021}

\usepackage[utf8]{inputenc} 
\usepackage[T1]{fontenc}    
\usepackage{hyperref}       
\usepackage{url}            
\usepackage{booktabs}       
\usepackage{amsfonts}       
\usepackage{nicefrac}       
\usepackage{microtype}      
\usepackage{xcolor}         
\usepackage[font=small]{caption}
\usepackage[font=small]{subcaption}
\usepackage[inline]{enumitem}
\usepackage{wrapfig}

\usepackage{footnote}
\usepackage{etoolbox}
\BeforeBeginEnvironment{definition}{\savenotes}
\AfterEndEnvironment{definition}{\spewnotes}

\title{Stochastic Gradient Descent-Ascent and \\ Consensus Optimization for Smooth Games: \\Convergence Analysis under Expected Co-coercivity}
\usepackage{pgfplotstable} 
\usetikzlibrary{automata, positioning, arrows, shapes, fit, calc, intersections}
\usepgfplotslibrary{statistics}
\include{figure}
\usepackage{amsmath,amsfonts,amssymb,amsthm,array}
\usepackage{algorithm}
\usepackage{algorithmic}%
\usepackage{bm}
\usepackage{color}
\usepackage{graphicx}

\newcommand{\Exp}{\mathbb{E}}

\newcommand{\E}[1]{{\mathbb{E}\left[#1\right] }}    
\newcommand{\EE}[2]{{\mathbb{E}_{#1}\left[#2\right] }} 

\newcommand{\Prob}[1]{\mathbb{P} \left[ #1\right]}
\newcommand{\R}{\mathbb{R}}

\usepackage{multirow}

\newcommand{\bA}{\mathbf{A}}
\newcommand{\bB}{\mathbf{B}}
\newcommand{\bC}{\mathbf{C}}
\newcommand{\bD}{\mathbf{D}}

\newcommand{\bI}{\mathbf{I}}

\newcommand{\bQ}{\mathbf{Q}}

\newcommand{\bJ}{\mathbf{J}}
\newcommand{\eqdef}{:=}

\newcommand{\cD}{{\cal D}}

\newcommand{\cH}{{\cal H}}

\newcommand{\cL}{{\cal L}}

\usepackage{mdframed} 
\usepackage{thmtools}

\definecolor{shadecolor}{gray}{0.9}
\declaretheoremstyle[
headfont=\normalfont\bfseries,
notefont=\mdseries, notebraces={(}{)},
bodyfont=\normalfont,
postheadspace=0.5em,
spaceabove=1pt,
mdframed={
  skipabove=8pt,
  skipbelow=8pt,
  hidealllines=true,
  backgroundcolor={shadecolor},
  innerleftmargin=4pt,
  innerrightmargin=4pt}
]{shaded}

\declaretheorem[style=shaded,within=section]{definition}
\declaretheorem[style=shaded,sibling=definition]{theorem}
\declaretheorem[style=shaded,sibling=definition]{proposition}
\declaretheorem[style=shaded,sibling=definition]{assumption}
\declaretheorem[style=shaded,sibling=definition]{corollary}

\declaretheorem[style=shaded,sibling=definition]{lemma}
\declaretheorem[style=shaded,sibling=definition]{remark}

\providecommand{\norm}[1]{\left\| #1\right\|}
\newcommand{\dotprod}[1]{\left< #1\right>}

\usepackage{hyperref}

\DeclareMathOperator{\nulls}{null}

\usepackage[colorinlistoftodos,bordercolor=orange,backgroundcolor=orange!20,linecolor=orange,textsize=scriptsize]{todonotes}

\newcommand{\remove}[1]{}

\usepackage{hyperref}       
\hypersetup{ 
	pdftitle={},
	pdfkeywords={},
	pdfborder=0 0 0,
	pdfpagemode=UseNone,
	colorlinks=true,
	linkcolor=blue, 
	citecolor=blue, 
	filecolor=blue, 
	urlcolor=blue, 
	pdfview=FitH,
	pdfauthor={Anonymous},
}

\author{%
  Nicolas Loizou\thanks{Corresponding author: nicolasloizou1@gmail.com. \, $^\dagger$Canada CIFAR AI Chair.} \\
  Mila and DIRO,\\
   Universit\'{e} de Montr\'{e}al\\
  \And
   Hugo Berard \\
   Mila and DIRO,\\
   Universit\'{e} de Montr\'{e}al\\
   \And
      Gauthier Gidel$^\dagger$ \\
   Mila and DIRO,\\
   Universit\'{e} de Montr\'{e}al\\
   \And
  Ioannis Mitliagkas$^\dagger$ \\
   Mila and DIRO,\\
   Universit\'{e} de Montr\'{e}al\\
   \And
   Simon Lacoste-Julien$^\dagger$ \\
   Mila and DIRO,\\
   Universit\'{e} de Montr\'{e}al\\
}

\begin{document}

\maketitle
\begin{abstract}
Two of the most prominent algorithms for solving unconstrained smooth games are the classical stochastic gradient descent-ascent (SGDA) and the recently introduced stochastic consensus optimization (SCO)~\citep{mescheder2017numerics}.
SGDA is known to converge to a stationary point for specific classes of games, but current convergence analyses require a bounded variance assumption. SCO is used successfully for solving large-scale adversarial problems, but its convergence guarantees are limited to its deterministic variant. In this work, we introduce the \textit{expected co-coercivity} condition, explain its benefits, and provide the first last-iterate convergence guarantees of SGDA and SCO under this condition for solving a class of stochastic variational inequality problems that are potentially non-monotone. We prove linear convergence of both methods to a neighborhood of the solution when they use constant step-size, and we propose insightful stepsize-switching rules to
guarantee convergence to the exact solution. In addition, our convergence guarantees hold under the arbitrary sampling paradigm, and as such, we give insights into the complexity of minibatching. 
\end{abstract}
\section{Introduction}
\vspace{-1mm}
Motivated from the recent interest in solving adversarial formulations in machine learning such as generative adversarial networks (GANs)~\citep{goodfellow2014generative}, we consider in this paper a more abstract formulation of the problem and focus on solving the following \emph{unconstrained} stochastic variational inequality (VI) problem:\footnote{While our presentation focuses on this finite-sum structure, most of our convergence results can easily be adapted to the general stochastic setting (see App.~\ref{Appendix_BeyondFiniteSum}). Also, we do not use the full power of variational inequalities that usually have constraints~\citep{harker1990finite}, but standard algorithms for~\eqref{VI} are coming from this literature~\citep{gidel2018variational}.}
\vspace{-3mm}
\begin{equation}
\label{VI}
\text{Find} \quad  x^* \in \R^d \quad \text{such that} \quad \xi(x^*)=\frac{1}{n}\sum_{i=1}^n \xi_i(x^*)=0,
\end{equation}
where each $\xi_i: \R^d \rightarrow \R^d$ is Lipschitz continuous. Further, we assume that the problem \eqref{VI} has a unique\footnote{This assumption can be relaxed; but for simplicity of exposition we enforce it.} solution $x^*$ and that the operator $\xi$ is $\mu$-\emph{quasi-strongly monotone}: there is a $\mu \geq 0$ such that:
\begin{equation}
\label{QSM}
\left\langle\xi(x),  x-x^*\right\rangle \geq \mu \|x-x^*\|^2 \quad \forall x \in \R^d
\end{equation} 
If $\mu=0$, then we say that $\xi$ satisfies the variational stability condition: $\langle \xi(x), x-x^*\rangle\geq 0$~\citep{hsieh2020explore}. In the variational inequality literature, condition~\eqref{QSM} is also known as strong stability condition~\citep{mertikopoulos2019games} or as strong Minty variational inequality (MVI) \citep{diakonikolas2021efficient, song2020optimistic}.

Problem~\eqref{VI} generalizes the solution of several types of \emph{stochastic smooth games}~\citep{facchinei2007games, scutari2010games, mertikopoulos2019games}. The simplest example is the unconstrained min-max optimization problem (also called a \emph{zero-sum} game):
\vspace{-2mm}
\begin{equation}
\label{Eq:MinMax}
\min_{x_1 \in\R^{d_1}} \max_{x_2 \in\R^{d_2}}\frac{1}{n} \sum_{i=1}^n g_i(x_1,x_2) \, ,
\vspace{-1mm}
\end{equation}
where each component function $g_i:\R^{d_1} \times \R^{d_2}\rightarrow \R$ is assumed to be smooth. Here, $\xi_i$ represents the appropriate concatenation of the block-gradients of $g_i$: $\xi_i(x) := (\nabla_{x_1} g_i(x_1,x_2); -\nabla_{x_2} g_i(x_1,x_2) )$, where $x := (x_1; x_2)$. Solving~\eqref{VI} then amounts to finding a stationary point $x^* = (x_1^*; x_2^*)$ for~\eqref{Eq:MinMax}, which under a convex-concavity assumption for $g_i$ for example, implies that it is a global solution for the min-max problem. More generally, we might seek the pure Nash equilibrium of a $k$-players game, where each player $j$ is simultaneously trying to find the action $x_j^*$ which minimizes with respect to $x_j \in \R^{d_j}$ their own cost function $\frac{1}{n} \sum_{i=1}^n f^{(j)}_i(x_j, x_{-j})$, while the other players are playing $x_{-j}$, which represents $x = (x_1,\ldots,x_k)$ with the component~$j$ removed. Here, $\xi_i(x)$ is the concatenation over all possible $j$'s of $\nabla_{x_j} f_i^{(j)} (x_j, x_{-j})$.

Such finite sum formulations appear in several machine learning applications such as generative adversarial networks (GANs)~\citep{goodfellow2014generative}, robust learning~\citep{wen2014robust} or even some formulations of reinforcement learning~\citep{pfau2016connecting}. A standard algorithm that has been used to solve~\eqref{VI} is the stochastic version of the classical gradient method~\citep{demyanov1972GradientMethodSP, Nemirovski-Juditsky-Lan-Shapiro-2009} or its variance reduced version~\citep{balamurugan2016stochastic}, that we call stochastic gradient descent-ascent (SGDA) in this paper.\footnote{We use this suggestive name motivated from the min-max formulation~\eqref{Eq:MinMax}, though we also call SGDA the simple update $x^{k+1} = x^k - \alpha \xi_{i_k}(x^k)$ to solve~\eqref{VI} in the more general non-zero sum game scenario.} More recently, \citet{mescheder2017numerics} analyzed some limitations of the gradient method in the context of GAN training and proposed an alternative efficient algorithm which could be used to solve~\eqref{VI} that they called consensus optimization (CO), which combines gradient updates with the minimization of $\| \xi(x)\|^2$. While the practical version of their algorithm for large $n$ is stochastic (SCO, that randomly samples $i$'s) and displayed good performance~\citep{mescheder2017numerics}, the only global convergence rate guarantees existing in the literature so far is only for the deterministic variant~\citep{azizian2019tight,abernethy2021last}. 

The classical results from the stochastic VI literature are inappropriate for several reasons. First, a \emph{uniform bound} over $x$ on the variance $\E{\|\xi_i(x) - \xi(x)\|^2}$ is typically assumed to get convergence guarantees (see e.g.~\citet{Nemirovski-Juditsky-Lan-Shapiro-2009,gidel2018variational,mertikopoulos2019games,yang2020global,lin2020finite}), but this is not compatible with the \emph{unconstrained} aspect of~\eqref{VI}. For example, suppose $g_i$ is a quadratic function, then the variance typically goes to infinity as $x \rightarrow \infty$. More appropriate relaxed assumptions have been considered to prove the convergence of other algorithms for~\eqref{VI} such as the stochastic extragradient method~\citep{hsieh2020explore,mishchenko2020revisiting} and its variance-reduced version~\citep{chavdarova2019reducing}, or the stochastic Hamiltonian gradient method~\citep{loizou2020stochastic}, but not yet to the best of our knowledge for SGDA nor SCO. Second, the classical analysis for SGDA~\citep{Nemirovski-Juditsky-Lan-Shapiro-2009} typically considers the convergence of the \emph{average of the iterates} rather than for the \emph{last-iterate}. However, as pointed out among others by~\citet{daskalakis2017training} and \citet{chavdarova2019reducing}, getting last-iterate convergence is important to apply the methods on potentially non-monotone problems such as GANs, where averaging is not appropriate. The only non-asymptotic last iterate convergence result for SGDA that we are aware of is~\cite{lin2020finite}, which focuses on a different class of problems (not assuming quasi-strong monotonicity) but it relies on strong assumptions on $\E{\|\xi_i(x)\|^2}$ (see Section~\ref{Section_SGDA}). 

In this paper, we address both of these issues. We generalize the recent improved analysis of SGD~\citep{gower2019sgd} to the case of unconstrained stochastic variational inequality~\eqref{VI}, and prove the last-iterate convergence for both SGDA and SCO without requiring any bounded variance assumption. We focus on quasi-strongly monotone VI problems, a class of structured non-monotone operators for which we are able to provide tight convergence guarantees and avoid the standard issues (cycling and divergence of the methods) appearing in the more general non-monotone regime.

\paragraph{Main Contributions.}
The key contributions of this work are summarized as follows:
\vspace{-1mm}
\begin{itemize}[leftmargin=*]
\item We propose the \emph{expected co-coercivity} (EC) assumption, which is the appropriate generalization of the expected smoothness assumption from~\citet{gower2019sgd} to Problem~\eqref{VI}. We explain the benefits of EC and show that is strictly weaker than the bounded variance assumption and ``growth conditions'' previously used for the analysis of stochastic algorithms for~\eqref{VI}.
\item Using the EC assumption, we prove the first last-iterate convergence guarantees for stochastic gradient descent-ascent (SGDA) on~\eqref{VI} without any unrealistic noise assumption. We show a linear convergence rate to a neighborhood of $x^*$ when constant step-size is used, and a $O(1/k)$ rate to the exact solution when using a decreasing step-size rule. For the latter, we propose a theoretically motivated switching rule from a constant to a decreasing step-size to get faster convergence.
\item Using the EC assumption, we provide the first convergence analysis of a stochastic variant (SCO) of the consensus optimization (CO) algorithm proposed by~\cite{mescheder2017numerics} and previously used to trained GANs. In particular, we prove last-iterate convergence for SCO, for both constant and decreasing step-sizes. As a corollary of our results, we obtain an improved convergence analysis for the deterministic CO. Furthermore, we explain how the update rule of the stochastic Hamiltonian gradient descent~\citep{loizou2020stochastic} is a special case of the SCO and show that in this scenario, our analysis matches the theoretical guarantees presented in~\citet{loizou2020stochastic}. 
\item Inspired by recent results from the optimization literature~\citep{gower2019sgd}, we give the first stochastic reformulation of the variational inequality problem~\eqref{VI} which enables us to provide convergence guarantees of SGDA and SCO under the \emph{arbitrary sampling paradigm}~\citep{richtarik2016optimal}. This allows us to give insights into the complexity of minibatching.
\end{itemize}
\section{Arbitrary Sampling: Stochastic Reformulation of Problem~\eqref{VI}}
\label{Sec:sampling}
In this work, we provide theorems through which we can analyze all minibatch variants of the two algorithms under study, SGDA and SCO. To do this, we construct a so-called ``stochastic reformulation'' of the variational inequality problem~\eqref{VI}. Our approach is inspired by recently proposed stochastic reformulations of standard optimization problems, like the empirical risk minimization in \cite{gower2019sgd} and linear systems in \citet{richtarik2020stochastic,loizou2020convergence,loizou2020momentum}.

In each step of our algorithms, we assume we are given access to unbiased estimates $g(x) \in \R^d$ of the operator such that $\E{g(x)}  = \xi(x).$  For example, we can use a minibatch to form an estimate of the operator such as $g(x) = \frac{1}{b}\sum_{i\in S}\xi_i(x),$ where $S \subset \{1,\ldots, n\}$ will be chosen uniformly at random and $|S|=b.$ 
To allow for any form of minibatching, we use the \emph{arbitrary sampling} notation $g(x) = \xi_v(x) \eqdef \frac{1}{n} \sum _{i=1}^n v_i \xi_i(x),$ where $v\in\R^n_+$ is a random \emph{sampling vector} drawn from some distribution $\cD$ such that $\Exp_{\cD}[v_i]  = 1, \,\mbox{for }i=1,\ldots, n$. 
Note that the unbiasedness follows immediately from this definition of the sampling vector: $\Exp_{\cD}[\xi_v(x)] =\frac{1}{n} \sum _{i=1}^n \Exp_{\cD}[v_i] \xi_i(x) = \xi(x).$ 

Thus, with each user-defined distribution $\cD$, we are able to introduce a stochastic reformulation of problem~\eqref{VI} as follows:
\vspace{-3mm}
\begin{equation}
\label{Reformulation}
\text{Find} \quad  x^* \in \R^d \quad \text{such that} \quad \Exp_\cD\left[\xi_v(x^*)\eqdef\frac{1}{n}\sum_{i=1}^n v_i \xi_i(x^*) \right]=0.
\end{equation}
Since $\xi_v(x)$ as an unbiased estimate of the operator $\xi(x)$, we can now use stochastic (simultaneous) gradient descent-ascent (SGDA) to solve~\eqref{Reformulation} as follows:
 \begin{eqnarray}  
 \label{SGDA_UpdateRule}
 x^{k+1} =  x^k- \alpha_k  \xi_{v^k}(x^k) \, ,  
\end{eqnarray} 
where $v^k \sim \cD$ is sampled i.i.d at each iteration and $\alpha_k >0$ is a stepsize. We highlight that in our analysis, we allow to select \emph{any} distribution $\cD$ that satisfies $\Exp_{\cD}[v_i]  = 1$ $\forall i$, and for different selection of $\cD$, \eqref{SGDA_UpdateRule} yields different interpretation as an SGDA method for solving the original problem \eqref{VI}.  

In this work, we mostly focus on the $b$--minibatch sampling, however note that our analysis holds for every form of minibatching and for several choices of sampling vectors $v$.
\begin{definition}[Minibatch sampling]\label{def:minibatch}
Let $b \in [n]$. We say that $v \in \R^n$ is a $b$--minibatch sampling if
for every subset $S \in [n]$ with $|S| =b$, we have that $\Prob{v=\frac{n}{b}\sum_{i \in S} e_i}=\left.1 \right/\binom{n}{b} \eqdef \frac{b!(n-b)!}{n!}$
\end{definition}
By using a double counting argument, one can show that if $v$ is a $b$--minibatch sampling, it is also a valid sampling vector ($\Exp_{\cD}[v_i]  = 1$)~\citep{gower2019sgd}. See~\cite{gower2019sgd} for other choices of sampling vectors $v$.

\section{Expected Co-coercivity and Connection to Other Assumptions}
\label{Section_ExpectedCoCo}
Before introducing the condition of expected co-coercivity, we first review some details on co-coercivity, an intermediate notion between monotonicity and strong monotonicity~\citep{zhu1996co}, and explain where it belongs as assumption in the literature of variational inequalities and min-max optimization.
\paragraph{Co-coercive operators.} The co-coercive condition is relatively standard in operator splitting literature~\citep{davis2017three,vu2013splitting} and for variational inequalities~\citep{zhu1996co}. It was used to analyze the celebrated forward-backward algorithm (a.k.a, proximal gradient)~\citep{lions1979splitting,chen1997convergence,palaniappan2016stochastic} that is known not to converge for general monotone operators~\citep{bauschke2011convex}.
\begin{definition}[Co-coercivity / Co-coercive \emph{around} $w^*$] \label{def:cocoercivity}
We say that an operator $\xi$ is $\ell$--co-coercive if there exist $\ell>0$ such that,\footnote{Note that in our definition we consider the inverse of the co-coercive constant from~\citet{lions1979splitting} which is the constant $\ell$ such that $\langle\xi(x)-\xi(y),x-y\rangle \geq \ell \|\xi(x)-\xi(y)\|^2 .$}
$\|\xi(x)-\xi(y)\|^2 \leq \ell \langle\xi(x)-\xi(y),x-y\rangle \quad \forall x , y \in \R^d.$ \\If there exist $w^* \in \R^d$ and $\ell>0$ such that $\|\xi(x)-\xi(w^*)\|^2 \leq \ell \langle\xi(x)-\xi(w^*),x-w^*\rangle \quad \forall x \in \R^d\,.$ then we say that the operator $\xi$ is $\ell$--co-coercive \emph{around} $w^*$. Note that in the last definition the point $w^*$ is not necessarily a point where $\xi(w^*)=0$.
\end{definition}
\vspace{-2mm}
Note that from Cauchy-Swartz's inequality, one can get that a $\ell$-co-coercive operator is $\ell$-Lipschitz.
In single-objective minization, one can show the converse statement by using convex duality. Thus, a gradient of a function is $L$--co-coercive if and only if the function is convex and $L$-smooth (i.e. $L$-Lipschitz gradients)~\citep{bauschke2011convex}. However, in general, a $L$-Lipchitz operator is \emph{not} $L$--co-coercive. What we can show instead is that  a $L$-Lipschitz and $\mu$-strongly monotone operator is $\ell$--co-coercive with $\ell \in [L, L^2/\mu]$~\citep{facchinei2007finite}. Note that both ranges of the spectrum may occur. For instance, \citet{chavdarova2019reducing} present a sufficent condition in zero-sum games to have $\ell = O(L)$.  Note that one can easily show that a sum of co-coercive operators is also co-coercive. 
Let us now provide a proposition summarizing the implications between (strong) monotonicity and co-coercivity. 
\begin{proposition}
\label{PropositionCocoMonotone}
For a $L$-Lipschitz operator $\xi$, the following implications hold:
\[
\begin{array}
[c]{cccccccc}
& \mu\text{-strongly monotone} & \Longrightarrow & \frac{L^2}{\mu}\text{-co-coercive}
 & \Longrightarrow  & \text{monotone} \\
&  \Downarrow & & \Downarrow && \Downarrow \\
& \mu\text{-quasi-strongly monotone}   & \Longrightarrow & \frac{L^2}{\mu}\text{-co-coercive around } x^*& \Longrightarrow & \begin{tabular}{c}\text{variational} \\ \text{stability condition} \end{tabular} 
\end{array}
\]
\end{proposition}
Let us also note that while a $\ell$-co-coercive operator is always $\ell$-Lipschitz continuous, it is possible for an operator to be $\ell$-co-coercive \emph{around $x^*$} and \emph{not} be Lipschitz continuous. This highlights the wider applicability of the $\ell$-co-coercivity around $x^*$ assumption that is all we need for several of our convergence results, in contrast to the Lipschitz continuity of $\xi$ which is typically assumed in the variational inequality literature. 
In Appendix~\ref{App:coolNonMonotoneExample}, we provide such example of a $\mu$-quasi strongly monotone operator that is $\ell$-co-coercive around $x^*$, which is not monotone nor Lipschitz continuous. 
\subsection{Expected Co-coercivity (EC)}
In our analysis of SGDA and SCO, we rely on a generic and remarkably weak assumption that we call \emph{expected co-coercivity} (EC).  In this section, we formally define EC, provide sufficient conditions for it to hold and relate it to the existing gradient assumptions.
\begin{assumption}[Expected Cocoercivity]
\label{ass:ExpCoCo} We say that $\xi$ is $\ell_{\xi}$--co-coercive in expectation with respect to a distribution $\cD$ if there exists  $\ell_{\xi}>0$  such that
\begin{equation}
\label{eq:ExpCoCo}
\Exp_{\cD}\left[\norm{\xi_v(x)-\xi_v(x^*)}^2 \right]\leq \ell_{\xi}\langle \xi (x),x-x^*\rangle \tag{EC} \quad \forall x \in \R^d \, .
\end{equation}
For simplicity, we will write $\xi  \in EC(\ell_\xi)$ to say that \ref{eq:ExpCoCo} holds and we will refer to $\ell_{\xi}$ as the expected co-coercivity constant.
\end{assumption}
The convergence results in this paper will depend on the following operator noise at $x^*$ that is finite for any reasonable sampling distribution $\cD$ for the sampling vector $v$:
\begin{equation}
\label{Sigma}
\sigma^2  \eqdef \Exp_{\cD}[\norm{\xi_v(x^*)}^2] < \infty .
\end{equation}
As we discuss below, common assumptions used to prove convergence of stochastic algorithms for solving the VI problem is uniform boundedness of the stochastic operator $\Exp\|\xi_i(x)\|^2 \leq c$ or uniform boundedness of the variance  $\Exp\|\xi_i(x)-\xi(x)\|^2 \leq c$. However these assumptions either do not hold or are true only for restrictive set of problems. In our work we do not assume such bounds. Instead we use the following direct consequence of Assumption~\ref{ass:ExpCoCo}.
\begin{lemma}
\label{MainLemma}
If $\xi  \in EC(\ell_\xi)$, then $\Exp \| \xi_{v} (x)\|^2 \leq 2\ell_{\xi}\langle \xi (x),x-x^*\rangle  + 2 \sigma^2$.
\end{lemma}
\vspace{-2mm}
Let us now provide some more familiar sufficient conditions which
guarantee that the \ref{eq:ExpCoCo} condition holds and give closed form expression for the expected co-coercivity parameter.
\begin{proposition}
\label{PropositionMinibatch}
Let $\xi_i$ be $\ell_i$ co-coercive (or $\ell_i$ co-coercive around $x^*$), then $\xi \in EC(\ell_\xi)$. Let $\ell_{\max}=\max \{\ell_i\}_{i=1}^n$ and $\ell$ be the co-coercive constant of $\xi$, if we let $v$ to be a $b$-minibatch sampling, then
$\ell_\xi = \frac{n}{b}\frac{b-1}{n-1}\ell+\frac{1}{b}\frac{n-b}{n-1} \ell_{\max}$ and $\sigma^2 =\frac{1}{b} \frac{n-b}{n-1} \sigma_1^2$,
where $\sigma_1^2 \eqdef  \frac{1}{n} \sum_{i=1}^n \norm{\xi_i(x^*)}^2$.
\end{proposition}
\vspace{-2mm}
In the above Proposition~\ref{PropositionMinibatch}, we show how co-coercivity of $\xi_i$ implies expected co-coercivity. However, the opposite implication does not necessarily hold. Indeed the expected co-coercivity can hold even when we do not assume that $\xi_i$ are co-coercive, as we show in the next proposition.
\begin{proposition}
\label{PropositionExtra}
Let $\xi$ be quasi-strongly monotone and let $\xi_i$ be $L_i$-Lipschitz continuous for all $i \in [n]$. Then $\xi  \in EC(\ell_\xi)$.
\end{proposition}
\vspace{-2mm}
\paragraph{Connection to Other Assumptions}
In the optimization literature, the standard convergence analysis of stochastic gradient algorithms like SGD relied on bounded gradient ($\Exp\|\nabla f_i(x)\|^2 \leq c$) or bounded variance assumptions ($\Exp\|\nabla f_i(x)-\nabla f(x)\|^2 \leq c$) \citep{recht2011hogwild, hazan2014beyond, rakhlin2012making} or growth condition ($\Exp\|\nabla f_i(x)\|^2 \leq c_1 \|\nabla f(x)\|^2 +c_2$) \citep{bottou2018optimization, schmidt2017minimizing}. However, a recent line of work shows that these assumptions might be restrictive or never be satisfied\footnote{For example, the bounded gradient assumption and strong convexity contradict each other in the unconstrained setting (see \citep{pmlr-v80-nguyen18c} for more details).} and proposed alternative conditions~ \citep{pmlr-v80-nguyen18c, vaswani2018fast, gower2019sgd, gower2021sgd, khaled2020unified, khaled2020better, assran2018stochastic, koloskova2020unified,patel2021stochastic, loizou2020stochastic, loizou2020stochasticB}. One of the weakest assumptions used for the convergence analysis of SGD in the smooth setting, is expected smoothness (ES) proposed in \citet{gower2019sgd} (see last row of Table~\ref{TableAssumptions}). Our expected co-coercivity condition \eqref{eq:ExpCoCo} can be seen as the generalization of ES in the operator setting.

In the literature of stochastic methods for solving the variational inequality problem and min-max optimization problem, similar assumptions have been made. In particular for the analysis of stochastic algorithms, papers assume either bounded operators \citep{Nemirovski-Juditsky-Lan-Shapiro-2009,abernethy2021last} or bounded variance \citep{juditsky2011solving,yang2020global,lin2020gradient,luo2020stochastic,tran2020hybrid} and growth condition \citep{lin2020finite}. In all of the these conditions, the values of parameters $c,$ $c_1$ and $c_2$ (see Table~\ref{TableAssumptions}) usually do not have a closed form expression -- they are simply assumed to exist. However, to the best of our knowledge, there is no analysis using a concept similar to our expected co-coercivity. All existing analyses of SGDA for (quasi)-strongly monotone and co-coercive operators require the much stronger extra assumptions of ``bounded noise” or ``bounded variance” to guarantee convergence, while for SCO, there are no known convergence guarantees in the literature. Note that through Lemma~\ref{MainLemma}, \eqref{eq:ExpCoCo} implies bounds on the gradient with closed-form problem-depended expressions for these constants. We also mention that other appropriate relaxed assumptions have been considered to prove the convergence of algorithms for~\eqref{VI} ~\citep{hsieh2020explore,mishchenko2020revisiting,chavdarova2019reducing, loizou2020stochastic}, but not yet for SGDA nor SCO. For a wider literature review in the area, see Appendix~\ref{Appendix_MoreRelatedWork}.

In Table~\ref{TableAssumptions}, we illustrate the correspondence between conditions used in the stochastic optimization literature and the stochastic VI problem. For further connections between these conditions, see Appendix~\ref{Appendix_Connections}. There, for example, we show why assuming bounded gradients together with strong monotonicity lead to an empty set of operators and explain why ES and EC (see last row of Table~\ref{TableAssumptions}) are equivalent for convex and smooth single-objective optimization problems.
\begin{table}[tb]
\caption{Correspondence of Assumptions between Optimization and Variational Inequalities}
\label{TableAssumptions}
\begin{center}
\resizebox{\linewidth}{!}{
\begin{tabular}{ |c|c|c|}
 \hline
\textbf{Assumptions} & \begin{tabular}{c}\textbf{Stochastic Optimization} \\ $\min_x f(x)=\frac{1}{n} \sum_{i=1}^n f_i(x)$\end{tabular}& \begin{tabular}{c}\textbf{Stochastic Variational Inequality} \\ $\text{Find} \quad  x^* \text{such that} \quad \xi(x^*)=\frac{1}{n}\sum_{i=1}^n \xi_i(x^*)=0$\end{tabular}\\
 \hline
 \hline
Bounded Gradient & $\Exp\|\nabla f_i(x)\|^2 \leq c$ &$\Exp\|\xi_i(x)\|^2 \leq c$ \\
 \hline
Bounded Variance  & $\Exp\|\nabla f_i(x)-\nabla f(x)\|^2 \leq c$ & $\Exp\|\xi_i(x)-\xi(x)\|^2 \leq c$  \\
 \hline
Growth Condition &$\Exp\|\nabla f_i(x)\|^2 \leq c_1 \|\nabla f(x)\|^2 +c_2$ & $\Exp\|\xi_i(x)\|^2 \leq c_1 \|\xi(x)\|^2 +c_2$ \\
 \hline 
\begin{tabular}{c}Expected Smoothness (ES) / \\ Expected Cocoercivity (EC) \end{tabular} & $\Exp \left[\|\nabla f_v(x)-\nabla f_v(x^*)\|^2\right] \leq 2 \cL (f(x) -f(x^*))$ & \colorbox{blue!20}{$\Exp\left[\norm{\xi_v(x)-\xi_v(x^*)}^2 \right]\leq \ell_{\xi}\langle \xi (x),x-x^*\rangle$ }  \\
 \hline 
\end{tabular}}
\vspace{-5mm}
\end{center}
\end{table}
\section{Stochastic Gradient Descent-Ascent}
\label{Section_SGDA}
Having presented the update rule of SGDA~\eqref{SGDA_UpdateRule} for solving the stochastic reformulation~\eqref{Reformulation} of the original unconstrained stochastic variational inequality problem~\eqref{VI}, let us now provide theorems for its convergence guarantees. We highlight that our theorems hold for any selection of distributions $\cD$ over the random sampling vectors $v$ and as such they are able to describe the convergence of an infinite array of variants of SGDA each of which is associated with a specific probability law governing the data selection rule used to form minibatches.
\begin{theorem}[Constant Step-size]
\label{SGDA_ConstantStep}
Assume that $\xi$ is $\mu-$quasi strongly monotone and that $\xi \in EC( \ell_{\xi})$. Choose $\alpha_k=\alpha \leq \frac{1}{2\ell_\xi}$ for all k. Then, the iterates of SGDA, given by \eqref{SGDA_UpdateRule}, satisfy:
\vspace{-3mm}
\begin{eqnarray}
\label{nakns}
\Exp \left[ \|x^{k}-x^*\|^2 \right]&\leq& \left(1-\alpha \mu \right)^k \|x^0-x^*\|^2  + \frac{2 \alpha \sigma^2}{ \mu} \, ,
\end{eqnarray}
\end{theorem}
Note that we do not assume that $\xi$ or $\xi_i$ are monotone operators in Theorem~\ref{SGDA_ConstantStep}. SGDA converges by only assuming that $\xi$ is quasi-strongly monotone and that \ref{eq:ExpCoCo} holds. Theorem~\ref{SGDA_ConstantStep} states that SGDA converges linearly to a neighborhood of $x^*$ which is proportional to the step-size $\alpha$ and the noise at the optimum $\sigma^2$.  We highlight, that since we control distribution $\cD$ we also control the values of $\ell_\xi$ and $\sigma^2$, and in the case of $b$-minibatch sampling these values have a closed-form expressions as shown in Proposition~\ref{PropositionMinibatch}.  To the best of our knowledge, Theorem~\ref{SGDA_ConstantStep} is the first last-iterate non-asymptotic convergence guarantee for SGDA for solving quasi-strongly monotone problems without assuming extra conditions on the noise. It is worth mentioning that \citet{lin2020finite} also prove last-iterate convergence of SGDA for different class of problems (they do not assume quasi-strong monotonicity), but the proposed analysis requires much stronger noise conditions. In particular, a bound on the variance with vanishing constants is needed, which, as far as we know, can only be satisfied by running SGDA with growing mini-batch size~\citep{friedlander2012hybrid} (see also App.~\ref{Appendix_MoreRelatedWork} for a more detailed discussion).
To highlight further the generality of Theorem~\ref{SGDA_ConstantStep}, we note that for the deterministic GDA, $\sigma^2=0$. Thus, we can obtain the following corollary.
\begin{corollary}[Deterministic GDA]
\label{CorollaryGDA}
Let all assumptions of Theorem~\ref{SGDA_ConstantStep} be satisfied. Let $|S|=n$ with probability one (each iteration of SGDA uses a full batch gradient). Then by selecting $\alpha_k=\alpha \leq \frac{1}{2\ell}$ for all k, the iterates of deterministic GDA satisfy: $ \|x^{k}-x^*\|^2 \leq \left(1-\alpha \mu \right)^k \|x^0-x^*\|^2$.
\end{corollary}
Even if Corollary~\ref{CorollaryGDA} looks trivial, to the best of our knowledge, Theorem~\ref{SGDA_ConstantStep} is the first convergence theorem of SGDA that includes the deterministic GDA originally provided by~\citet{chen1997convergence} as a special case.   

\textbf{Optimal $b$-Minibatch Size:} Using standard computations, the convergence rate presented in Theorem~\ref{SGDA_ConstantStep} can be equivalently expressed as iteration complexity result as follows: If we are given any accuracy $\epsilon>0$, choosing stepsize $\alpha  = \min \left\{ \frac{1}{2\ell_\xi},\; \frac{\epsilon\mu}{4 \sigma^2}\right\}$ and 
$
k\geq  \max \left\{ \tfrac{2\ell_{\xi}}{\mu },\; \tfrac{4 \sigma^2}{\epsilon\mu^2}\right\} \log\left(\tfrac{ 2 \|x^0 - x^*\|^2 }{  \epsilon }\right),
$
implies $\mathbb{E} \| x^k - x^* \|^2  \leq \epsilon.$ By combining the lower bound on $k$ with the expressions of $\ell_\xi$ and $\sigma^2$ of Proposition~\ref{PropositionMinibatch}, we have that the iteration complexity (by ignoring the logarithmic terms) becomes $k\geq \frac{2}{\mu} \max\{\ell_\xi, \frac{2\sigma^2}{\epsilon \mu} \}$ where $\ell_\xi = \frac{n}{b}\frac{b-1}{n-1}\ell+\frac{1}{b}\frac{n-b}{n-1} \ell_{\max} $ and $\sigma^2 =\frac{1}{b} \frac{n-b}{n-1} \sigma_1^2$. Thus, the total complexity of the algorithm as a function of the minibatch size $b$ is given by $TC(b)\; \leq \; \frac{2}{\mu} \max\{b \ell_\xi, b \frac{2\sigma^2}{\epsilon \mu} \}$.
By following the same steps with \cite{gower2019sgd}, it can be shown that $b \ell_\xi$ is linearly increasing term in $b$ while  $b\frac{2\sigma^2}{\epsilon \mu}$ is a linearly decreasing term in $b$. Hence, if we define $b^*$ to be the minibatch size that minimize the total complexity $TC(b)$ (optimal $b$-Minibatch Size) we have that if $\sigma_1^2\leq \ell_{\max}$ then $b^*=1$ otherwise $b^* = n \frac{\ell  - \ell_{\max} +\frac{2}{\epsilon \mu} \cdot  \sigma_1^2}{n\ell  - \ell_{\max} +\frac{2}{\epsilon \mu} \cdot \sigma_1^2}.$

In the next theorem, we provide an insightful stepsize-switching rule that
describes when one should switch from a constant
to a decreasing step-size regime to guarantee convergence to $x^*$ and not to a neighborhood, providing the first convergence analysis of SGDA under such a switching rule.
\begin{theorem}
\label{SGDA_DecreasingStep}
Assume $\xi$ is $\mu$-quasi-strongly monotone and that $\xi \in EC( \ell_{\xi})$. Let  $\mathcal{K} \eqdef \left.\ell_\xi\right/\mu$ and let $\alpha_k=\frac{1}{2 \ell_\xi} $ for $k \leq 4\lceil\mathcal{K} \rceil $ and $\alpha_k=\frac{2k+1}{(k+1)^2 \mu}$ for $K>4\lceil\mathcal{K} \rceil$. If $k \geq 4 \lceil\mathcal{K} \rceil$, then iterates of SGDA, given by \eqref{SGDA_UpdateRule} satisfy:
\vspace{-3mm}
\begin{equation}
\mathbb{E}\| x^{k} - x^*\|^2 \le   \frac{\sigma^2 }{\mu^2 }\frac{8 }{k} + \frac{16 \lceil\mathcal{K} \rceil^2}{e^2 k^2 }  \|x^0 - x^*\|^2 = O\left(\frac{1}{k}\right) \, .
\end{equation}
\end{theorem}
\section{Stochastic Consensus Optimization}
\label{sec:SCO}
The Consensus Optimization (CO) Algorithm is a computationally-light\footnote{At each step, CO requires only the computation of a Jacobian-vector product which can be efficiently evaluated in tasks like training neural networks with comparable computation time of a gradient \citep{pearlmutter1994fast}.} second order methods which has been introduced in \cite{mescheder2017numerics} and it was shown to be an effective method for training GANs in a variety of settings. \cite{liang2019interaction} show first that CO converges linearly in the bilinear case. \cite{abernethy2021last} show that CO can be viewed as a perturbation of the deterministic Hamiltonian gradient descent (HGD) and explain how CO converges at the same rate as HGD, while \cite{azizian2019tight} prove convergence of CO for $\mu$-strongly monotone operators with positive singular values of the Jacobian matrix $J=\nabla\xi$.

However, even if CO is used explicitly in the stochastic setting and in practice only minibatch variants are implemented, to the best of our knowledge all existing analysis focus only on the deterministic setting. Thus, our work is the first that provide convergence guarantees for the Stochastic Consensus Optimization (SCO) and due to our framework, our analysis includes the convergence of the deterministic update as special case. Impressively, our analysis provides tighter rates than previous analysis even in the deterministic setting. 

For the results of this section, we assume that each $\xi_i$ in problem \eqref{VI} is differentiable. That is, we have access to the Jacobian matrices $\bJ_i(x)=\nabla\xi_i(x)$. Following \cite{loizou2020stochastic}, in our analysis we will also assume that the Hamiltonian function is quasi-strongly convex and that it satisfies expected smoothness (see last row of Table~\ref{TableAssumptions}). These are not strong assumptions, and as an example, they are satisfied for smooth bilinear min-max optimization problems \cite{loizou2020stochastic}. In Appendix~\ref{Appendix_Experiments}, by extending the results of \cite{loizou2020stochastic}, we explain how these assumptions can be satisfied for the quadratic min-max problems.

\subsection{Setting}
The consensus optimization (CO) algorithm as presented in \cite{mescheder2017numerics} has the following update rule:
\begin{equation}
\label{CO_Rule}
x^{k+1}=x^k - \alpha \xi(x^k) - \gamma \nabla \cH(x^k)
\end{equation}
where $\cH(x)= \frac{1}{2} \|\xi(x)\|^2$ is the Hamiltonian function and $\alpha, \gamma >0$ are the step-sizes. From its definition, it is clear that the update rule is essentially a weighted combination of GDA and the Hamiltonian gradient descent (HGD) of~\cite{balduzzi2018mechanics}.
In practice, implementing CO in a mini-batch setting (stochastic) leads to biased estimates of the gradient of the Hamiltonian function \citep{mescheder2017numerics}. This is one of the main reason that existing analysis was not able to capture the behavior of the method in the stochastic setting. However, recently \cite{loizou2020stochastic} proposed a way to obtain unbiased estimators of the gradient of the Hamiltonian function by expressing the Hamiltonian of a stochastic game as a finite-sum problem. In this work, we adopt the finite-sum structure and the unbiased estimators proposed in \cite{loizou2020stochastic} for the Hamiltonian part and we extend the formulation to capture the arbitrary sampling paradigm. That is, the Hamiltonian function can be expressed as, $\cH(x)= \frac{1}{2} \|\xi(x)\|^2 = \frac{1}{n} \sum_{i=1}^n \frac{1}{n} \sum_{j=1}^n \frac{1}{2} \langle  \xi_i(x),  \xi_j(x)\rangle$. 
In addition, by following the stochastic reformulation setting presented in Section~\ref{Sec:sampling}, let us have two independent random sampling vectors $u\sim \cD$ and $v\sim \cD$ and let us define: $\cH_{u,v}(x)=\frac{1}{n} \sum_{i=1}^n \frac{1}{n} \sum_{j=1}^n \frac{1}{2} \langle u_i \xi_i(x),  v_i \xi_j(x)\rangle.$ Since vectors $u$ and $v$ are \emph{independent} sampling vectors, it is clear that $\Exp_{u,v} [\cH_{u,v}(x)]= \cH(x)$, where $\Exp_{u,v}$ denotes the expectation with respect to distribution $\cD$ on both vectors $u$ and $v$. Let us also use $\bJ_i(x)=\nabla \xi_i(x)$ to express the Jacobian matrices for $i$, then the gradient of $\cH_{u,v}(x)$ has the following form:
\vspace{-4mm}
 \begin{eqnarray}
 \label{StochasticHamiltonianGradient}
\nabla \cH_{u,v}(x)=\frac{1}{2} \left[ \bJ_u^\top(x) \xi_v(x) +   \bJ_v^\top(x) \xi_u(x) \right]
\end{eqnarray}
where  $\bJ_u(x)= \frac{1}{n}\sum_{i=1}^n u_i \bJ_i(x)$ and $\bJ_v(x)= \frac{1}{n}\sum_{i=1}^n v_i \bJ_i(x)$.
Similar to \cite{loizou2020stochastic}, it can be shown that $\nabla \cH_{u,v}(x)$  is an unbiased estimator of $\nabla \cH(x)=\bJ^\top(x) \xi(x)$. That is, $\Exp_{u,v} [\nabla \cH_{u,v}(x)]=\nabla \cH(x)$.
Throughout this section, we will denote the gradient noise of the stochastic Hamiltonian function with $\sigma_{\cH}^2  \eqdef \Exp_{u,v}[\norm{\nabla \cH_{u,v}(x^*)}^2]$, which is finite for any reasonable sampling distribution $\cD$.

\subsection{Stochastic Consensus Optimization and its Special Cases}
\begin{wrapfigure}{r}{0.54\textwidth}
\begin{minipage}{0.54\textwidth}
\vspace{-5ex}
\begin{algorithm}[H]
   \caption{Stochastic Consensus Opt. (SCO)}
   \label{SCO_Algorithm}
\begin{algorithmic}
   \STATE {\bfseries Input:} Starting step-size $\alpha_0, \gamma_0>0$. Choose initial point $x^0 \in \R^d$. Distribution $\cD$ of samples.\hspace{-10mm}   
   \FOR{$k=0,1,2,\cdots, K$}
   \STATE Sample independently $v^k \sim {\cal D}$ and $u^k \sim {\cal D}$
   \STATE Set step-sizes $\alpha_k$ and $\gamma_k$ according to a pre-selected step-size rule
   \STATE Set  $x^{k+1}=x^k - \alpha_k \xi_{v^k}(x^k) - \gamma_k \nabla \cH_{v^k,u^k}(x^k)$
   \ENDFOR
   \STATE {\bf Output:} The last iterate $x^k$
\end{algorithmic}
\end{algorithm}
\end{minipage}
\vspace{-2ex}
\end{wrapfigure}
Having explained the basic setting and background on consensus optimization algorithm and the Hamiltonian function, let us now present as Algorithm~\ref{SCO_Algorithm} the proposed stochastic consensus optimization (SCO) algorithm. Note that in each iteration $k$, two random sampling vectors $u$ and $v$ are sampled independently from a user-defined distribution $\cD$. These vectors are used to evaluate $\xi_{v^k}$ and $\nabla \cH_{v^k,u^k}(x^k)$, the unbiased estimators of $\xi(x^k)$ and $\nabla \cH(x^k)$ at point $x^k$ respectively. 
Note also that the update rule of SCO is a weighted combination of SGDA \eqref{SGDA_UpdateRule} and the stochastic Hamiltonian gradient descent (SHGD) of~\citet{loizou2020stochastic}. Thus, it is clear, that if one selects $\alpha^k=0, \forall k >0 $ then the method is equivalent to SHGD and if $ \gamma^k=0, \forall k >0 $ then the method becomes equivalent to the SGDA.
In addition if we select sampling vectors $u=v = (1,1, \dots, 1) \in R^n$ with probability 1, then from the definition of $\xi_{v^k}$ and $\nabla \cH_{v^k,u^k}(x^k)$ we obtain the deterministic CO \eqref{CO_Rule} as special case of our update rule.
\subsection{Convergence Analysis}
Let us now present our main theoretical results describing the performance of SCO. Similar to the previous section, we provide two main theorems for two different step-size selection. 
\begin{theorem}[Constant Step-size]
\label{SCO_ConstantStep}
Assume $\xi$ is $\mu$-quasi-strongly monotone with $\mu \geq 0$ and that $\xi \in EC( \ell_{\xi})$. 
Assume that the Hamiltonian function $\cH$ is $\mu_{\cH}$-quasi strongly convex and $\cL_{\cH}$-expected smooth. Then, for $\gamma_k=\gamma \leq \frac{1}{4\cL_{\cH}}$ and $\alpha_k=\alpha \leq  \frac{1}{4\ell_\xi}$, the iterates of SCO satisfy:
\vspace{-2mm}
\begin{eqnarray}
\Exp\left[ \|x^{k}-x^*\|^2 \right]\leq(1-\gamma\mu_{\cH}  - \alpha\mu )^k\|x^0-x^*\|^2 + \frac{4[\alpha^2 \sigma^2 + \gamma^2  \sigma_{\cH}^2]}{\gamma\mu_{\cH}  +\alpha\mu}.
\end{eqnarray}
If $\mu=0$, that is $\xi$ only satisfies the variational stability condition $\langle\xi(x),x-x^*\rangle \geq0$, then 
$$\Exp \left[ \|x^{k}-x^*\|^2 \right]\leq (1-\gamma\mu_{\cH} )^k \|x^0-x^*\|^2  + \frac{4 [\alpha^2 \sigma^2 +  \gamma^2  \sigma_{\cH}^2]}{\gamma \mu_{\cH} }.$$
\end{theorem}
Theorem~\ref{SCO_ConstantStep} is quite informative, as it highlights that both SGDA and SHG parts should coexist in the update rule of the SCO to guarantee faster convergence for $\mu$-quasi-strongly monotone operators (i.e. both step-sizes $\alpha$ and $\gamma$ should be positive up to specific values). However, if $\xi$ simply satisfies the variational stability condition $\langle\xi(x),x-x^*\rangle \geq0$ (when $\mu=0$), then the convergence rate of SCO, does not depend on the step-size $\alpha$ (the SGDA part) and the neighborhood of convergence $\frac{4 [\alpha^2 \sigma^2 +  \gamma^2  \sigma_{\cH}^2]}{\gamma \mu }$ is smaller when $\alpha=0$. Thus, in this case one needs to simply run the SHGD.

To appreciate the generality of Theorem~\ref{SCO_ConstantStep}, let us present some corollaries and compare the rates with existing results in the literature. First, we get a rate for the deterministic CO algorithm. 
\begin{corollary}[Deterministic CO]
\label{DeterministicCO}
Let all assumptions of Theorem~\ref{SCO_ConstantStep} be satisfied. Then, for $\gamma_k=\gamma \leq \frac{1}{4 L_{\cH}}$ and $\alpha_k=\alpha \leq \frac{1}{4\ell}$, the iterates of CO satisfy:
$
	\|x^{k}-x^*\|^2 \leq(1-\gamma\mu_{\cH}  - \alpha\mu )^k\|x^0-x^*\|^2.
$
\end{corollary}
The result of Corollary~\ref{DeterministicCO} should be compared to the convergence guarantees for CO as provided in~\citet{abernethy2021last}, where the authors viewed CO as a perturbation of the deterministic Hamiltonian gradient descent (HGD). In particular, \citet{abernethy2021last} gives the rate $1-\frac{\mu_{\cH}}{4 L_\cH}$ for CO under similar assumptions,\footnote{For this convergence rate, \citet{abernethy2021last} assumed that the Hamiltonian function satisfies the Polyak-Lojasiewicz condition but focus on strongly-convex and strongly-concave min-max optimization problems.} which is clearly slower than our $1-\frac{\mu_{\cH}}{4 L_\cH}  - \frac{\mu}{4\ell}$ rate. The authors had explicitly noted that treating the GDA part as an adversarial perturbation most likely should be improved upon. With our analysis, we provide a different, more natural analysis of CO and answer this open problem.
In addition, note that by setting $\gamma_k=0$ and $\alpha_k=\alpha < 1/2  \ell_{\xi}$, then SCO becomes equivalent to SGDA and Theorem~\ref{SCO_ConstantStep} matches the convergence guarantees presented in Theorem~\ref{SGDA_ConstantStep}. On the other hand, by setting $\alpha^k=0$, SCO yields equivalent updates to SHGD and our result matches the theoretical guarantees of \citet{loizou2020stochastic} as we show in the next corollary.
\begin{corollary}
\label{CorollarySHGD}
Under the assumptions of Thm.~\ref{SCO_ConstantStep}, set $\alpha^k=0$ and $\gamma_k=\gamma \leq 1/2\cL_{\cH}$. Then SCO is equivalent to SHGD and its iterates satisfy: $\Exp\left[ \|x^{k}-x^*\|^2 \right]\leq(1-\gamma\mu_{\cH} )^k\|x^0-x^*\|^2 + \frac{2 \gamma \sigma_{\cH}^2}{\mu_{\cH}}$.
\end{corollary}
All previous results for SCO show convergence to a neighborhood of $x^*$. In the next theorem, by selecting decreasing step-sizes (switching strategy) for the values of $\alpha$ and $\gamma$, we are able to  guarantee a sublinear convergence to the exact solution for SCO. To the best of our knowledge this is the first result analyzing SCO with decreasing step-sizes.
\begin{theorem}
\label{SCO_DecreasingStep}
Assume $\xi$ is $\mu$-quasi-strongly monotone and that $\xi \in EC( \ell_{\xi})$. Assume that the Hamiltonian function $\cH$ is $\mu_{\cH}$-quasi strongly convex and $\cL_{\cH}$-expected smooth. Let $\alpha_k=\gamma_k$,  $\psi= \max\{  \ell_{\xi},\cL_{\cH}\}$ and $k^* \eqdef 8 \frac{\psi}{\mu_{\cH}+\mu}$. Let also, $\gamma_k= \frac{1}{4 \psi}$ for $k \leq \lceil k^* \rceil$ and $\gamma_k= \frac{2k+1}{(k+1)^2 [\mu_{\cH}+\mu]}$  for  $k >  \lceil k^* \rceil$. If $k \geq  \lceil k^* \rceil$, then SCO iterates satisfy:
\vspace{-2mm}
\begin{equation}
\mathbb{E}\| x^{k} - x^*\|^2 \le   \frac{\sigma_{\cH}^2 +\sigma^2}{[\mu + \mu_{\cH}]^2}\frac{16 }{k} + \frac{(k^*)^2 }{e^2 k^2}  \|x^0 - x^*\|^2 = O\left(\frac{1}{k}\right)
\end{equation}
If $\mu=0$, that is $\xi$ only satisfies the variational stability condition $\langle\xi(x),x-x^*\rangle \geq0$, then SCO is still able to converge sublinearly with $O\left(\frac{1}{k}\right)$, to $x^*$.
\end{theorem}
\section{Numerical Evaluation}
\label{sec:numerical_eval}
\vspace{-2mm}
The purpose of this experimental section is to corroborate our theoretical results, which form the main contributions of this paper. To do so, we focus on strongly-monotone quadratic games of the following form:
 \vspace{-3mm}
\begin{equation}
  \label{eq:quadratic_games}
      \min_{x_1 \in \mathbb{R}^{d}} \max_{x_2 \in \mathbb{R}^{p}} \frac{1}{n}\sum_i \frac{1}{2} x_1^\top \bA_i x_1 + x_1^\top \bB_i x_2 - \frac{1}{2} x_2^\top \bC_i x_2 + a_i^\top x_1 - c_i^\top x_2
\end{equation}
\begin{figure}[t]
  \centering
  \begin{subfigure}[b]{0.24\textwidth}
  \includegraphics[width=1.05\textwidth]{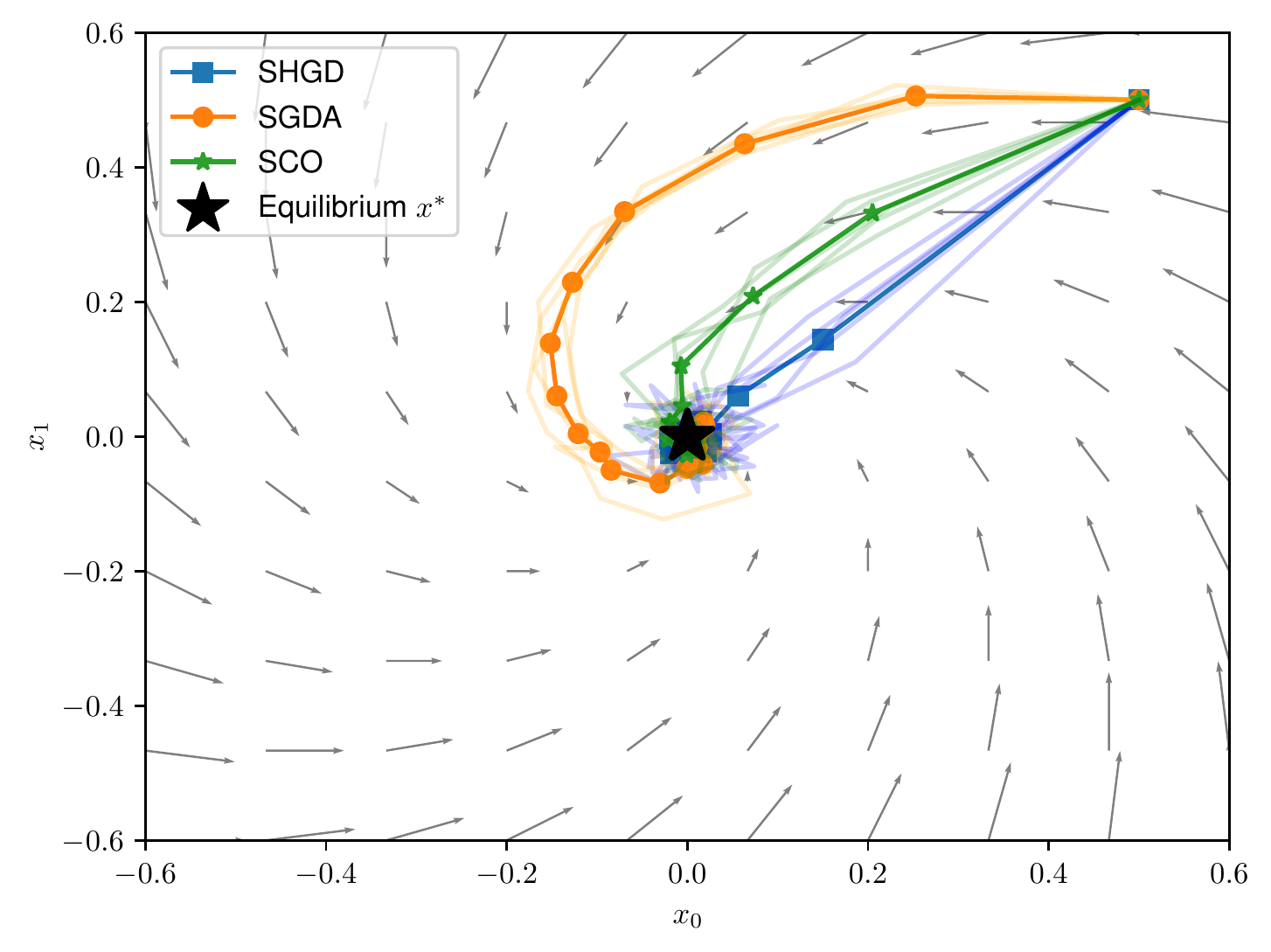}
  \caption{2d trajectories}
  \label{fig:2d_game}
  \end{subfigure}
  \begin{subfigure}[b]{0.24\textwidth}
    \includegraphics[width=\textwidth]{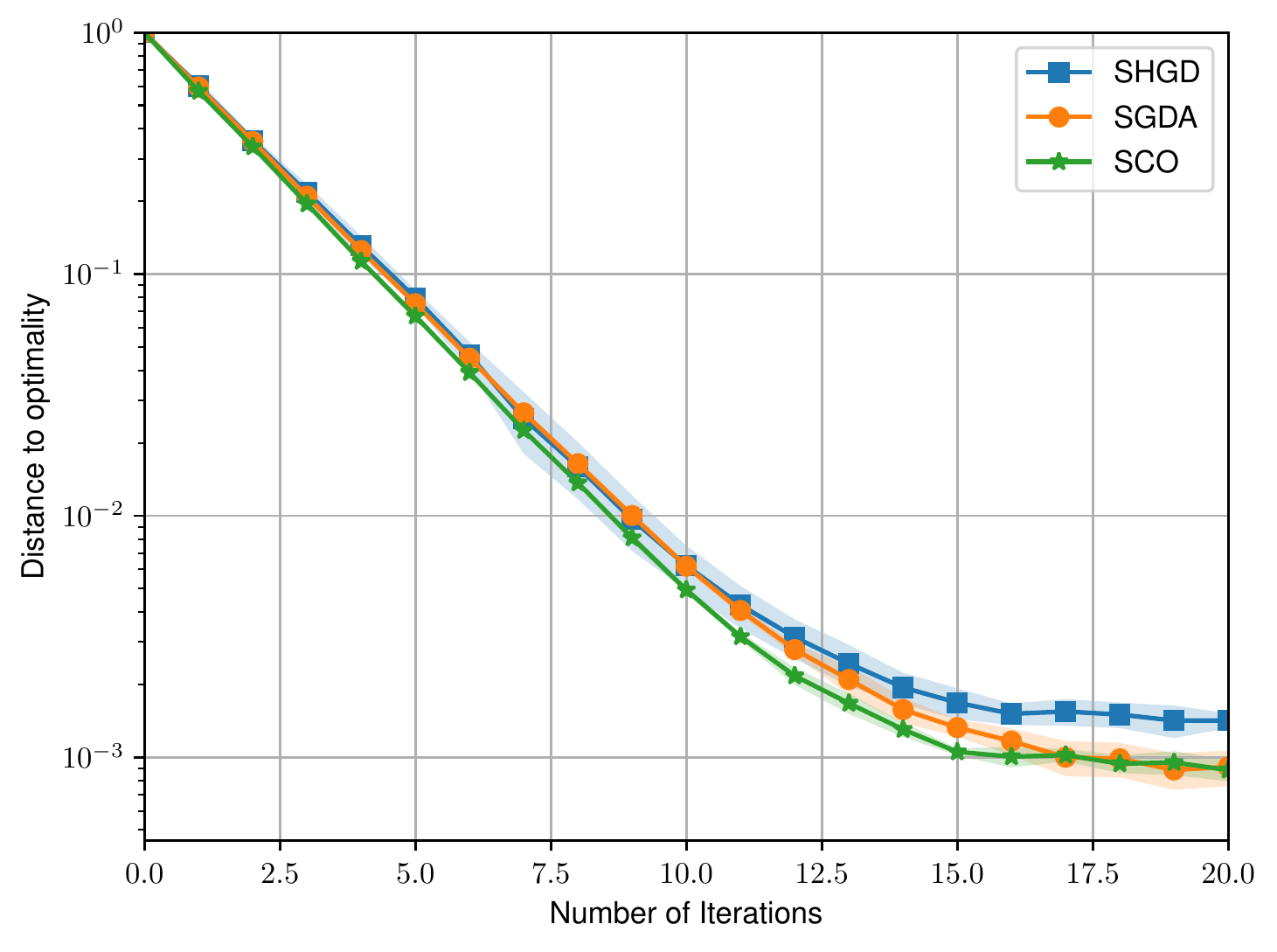}
    \caption{$\kappa_G=\frac{\ell_\xi}{\mu}=1$}
      \label{adnsonaoda}
  \end{subfigure}
  \begin{subfigure}[b]{0.24\textwidth}
    \includegraphics[width=\textwidth]{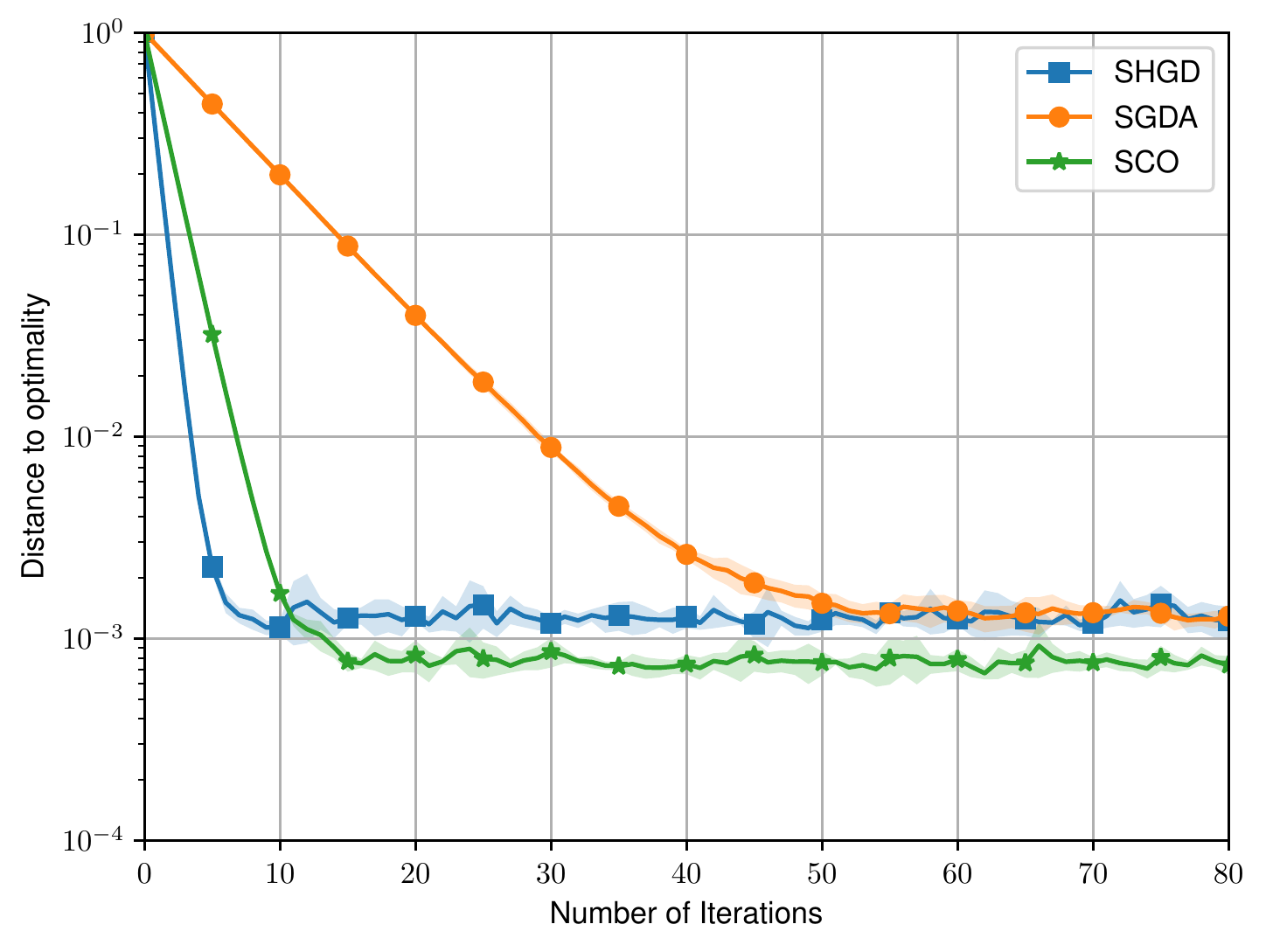}
    \caption{$\kappa_G=\frac{\ell_\xi}{\mu}=5$}
  \end{subfigure}
  \begin{subfigure}[b]{0.24\textwidth}
    \includegraphics[width=\textwidth]{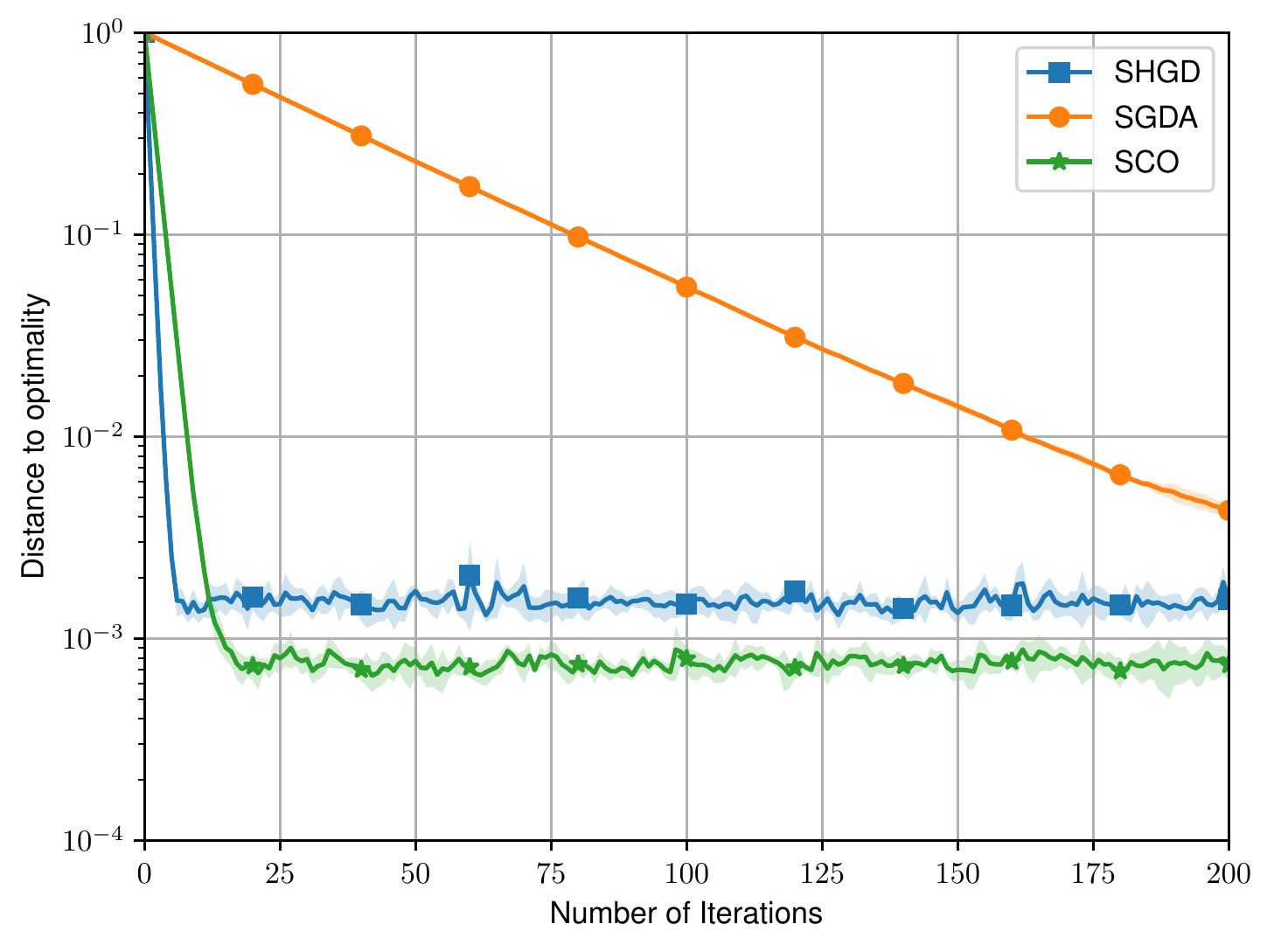}
    \caption{$\kappa_G=\frac{\ell_\xi}{\mu}=25$}
  \end{subfigure}
  \vspace{-2mm}
  \caption{\small Convergence of SGDA, SCO and SHGD on different quadratic games. \textbf{(a)} Trajectories of SCO, SHGD and SGDA on a 2d quadratic game. The arrows represent the direction defined by $\xi(x)$ at a particular point $x$. \textbf{(b-d)} Distance to optimality $\frac{\|x^k-x^*\|^2}{\|x^0-x^*\|^2}$ as a function of the number of iterations. Each plot corresponds to a game with a particular condition number $\kappa_G=\frac{\ell_\xi}{\mu}$. The solid lines represent the average performance over the 5 runs and the colored area represent the 95\% confidence intervals.}
  \label{fig:quadratic_game}
\end{figure}
We show that the Hamiltonian function of such a game is $\mu_{\cH}$-quasi-strongly convex and $L_{\cH}$-smooth in App.~\ref{app:quadratic_game_proof}. For the game to also be strongly-monotone and co-coercive, we sample the matrices such that $\mu_A \bI \preceq\bA_i \preceq L_A\bI$, $\mu_C\bI \preceq\bC_i \preceq L_C\bI$, and $\mu_B^2\bI \preceq \bB_i^\top\bB_i \preceq L_B^2\bI$, where $\bI$ is the identity matrix; the exact sampling is described in App.~\ref{app:experimental_details}. The bias terms $a_i, c_i$ are sampled from a normal distribution. We pick the step-size for the different methods according to our theoretical findings. That is,  
for constant step-size, we select $\alpha=\frac{1}{2\ell_\xi}$ for SGDA (Theorem~\ref{SGDA_ConstantStep}), $\alpha=\frac{1}{4\ell_\xi}, \gamma=\frac{1}{4\cL_{\cH}}$ for SCO (Theorem~\ref{SCO_ConstantStep}), and $\gamma=\frac{1}{2\cL_{\cH}}$ for SHGD (Corollary~\ref{CorollarySHGD}). For the stepsize-switching rule that guarantees convergence to $x^*$, we use the step-sizes proposed in Theorem~\ref{SGDA_DecreasingStep} for SGDA and Theorem~\ref{SCO_DecreasingStep} for SCO.
\begin{figure}[t]
	\centering
	\begin{subfigure}[b]{0.25\textwidth}
		\includegraphics[width=\textwidth]{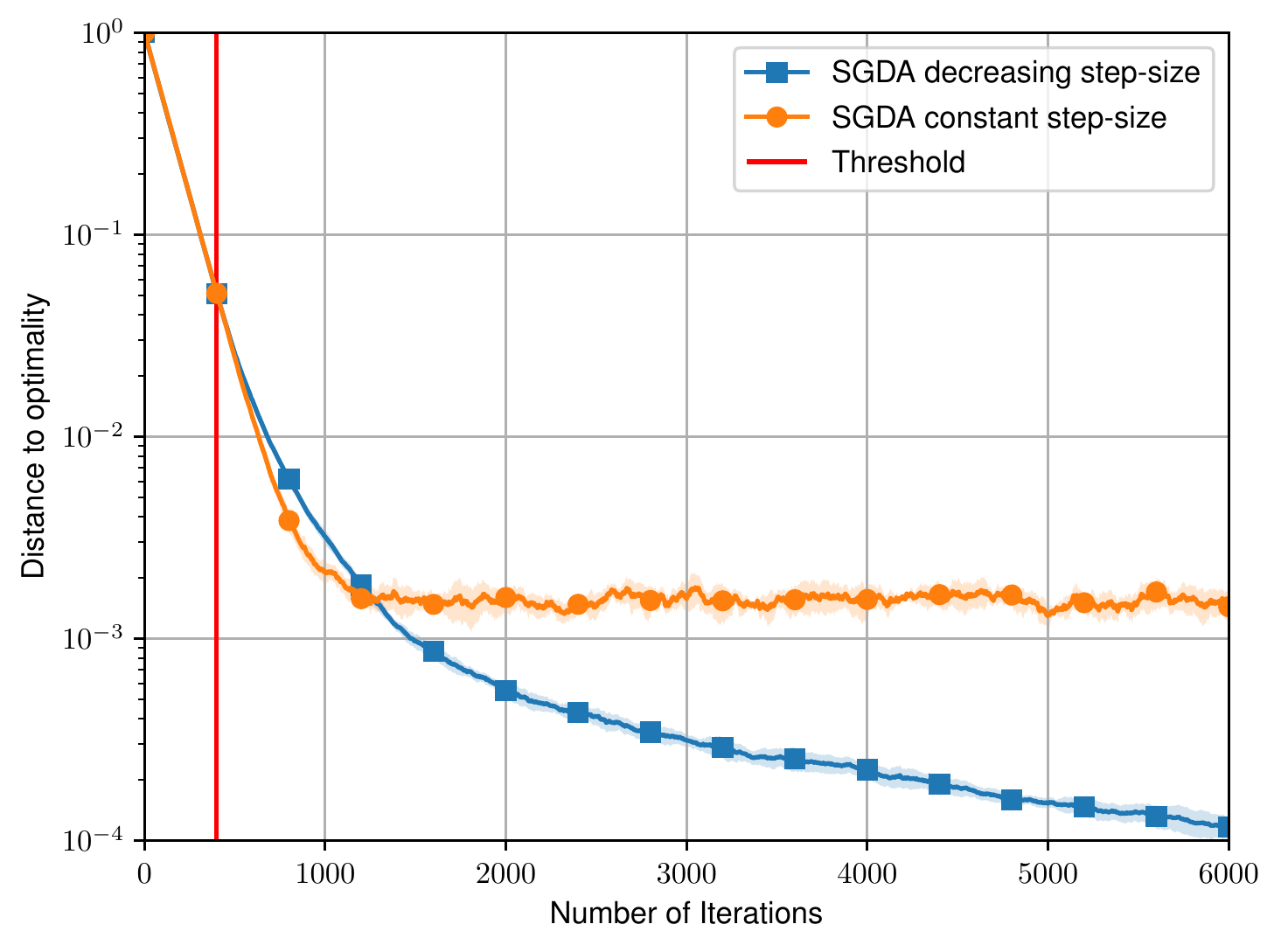}
		\caption{$\kappa_G=\frac{\ell_\xi}{\mu}=100$}
		\label{SGDAplot_switch}
	\end{subfigure}
		\begin{subfigure}[b]{0.25\textwidth}
		\includegraphics[width=\textwidth]{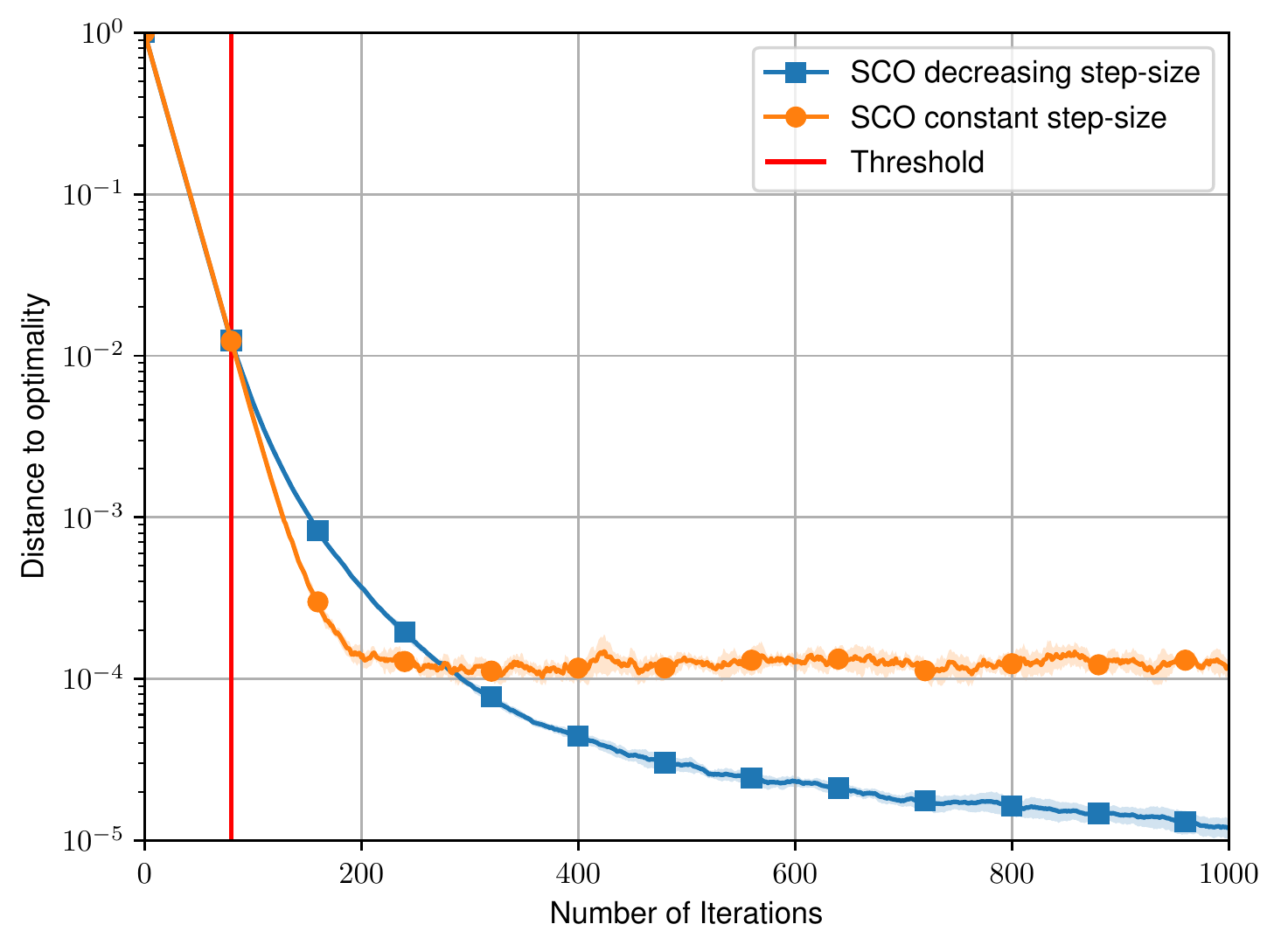}
		\caption{$\kappa_G=\frac{\ell_\xi}{\mu}=5$}
		\label{SCOplot_switch}
	\end{subfigure}
	\vspace{-2mm}
	\caption{\small Comparison between constant and decreasing step size regimes of SGDA and SCO. The vertical red lines correspond to the moment we switch from a constant to a decreasing step-size. The solid lines represent the average performance over the 5 runs and the colored area represent the 95\% confidence intervals.}
	\label{fig:quadratic_game_decreasing}
\end{figure}
Let us first look at the qualitative behavior of SGDA, SCO and SHGD on a simple 2d example where $x_1,x_2 \in \R$. We show the trajectories of the different algorithms in Fig.~\ref{fig:2d_game}. As expected we observe that the behavior of SCO is in between SGDA and SHGD. Recall that the update rule of SCO is a weighted combination of SGDA and the SHGD. The code to reproduce our results can be found at \url{https://github.com/hugobb/StochasticGamesOpt}.

\textbf{Comparison of Algorithms (Constant step-size).}  We look at the convergence of the methods for different games where we vary the condition number $\kappa_G=\frac{\ell_\xi}{\mu}$. In comparing the methods, we use the relative distance to optimality $\frac{\|x^k-x^*\|^2}{\|x^0-x^*\|^2}$. As predicted from our theoretical results, when constant step-size is used, all the methods converge linearly to a neighborhood of the solution (see Figure~\ref{fig:quadratic_game}). We observe that the performance of SGDA depends on the condition number: the higher the condition number, the slower the convergence. In contrast, the convergence of both SHGD and SCO is less affected by a larger condition number. We also observe that SGDA is slower than SHGD, but converges to a smaller neighborhood of the solution (see e.g. Fig.~\ref{adnsonaoda}). SCO achieves a good trade-off; it converges fast like SHGD and to a small neighborhood like SGDA. An important note is that the size of the neighborhood heavily depends on the selection of the learning rate; we explore this dependence in App.~\ref{app:step_size_study}.

\textbf{Constant vs Decreasing step-size.} We also compare the performance of SGDA and SCO in the constant and decreasing step-size regimes considered in Theorems~\ref{SGDA_ConstantStep} and~\ref{SGDA_DecreasingStep} for SGDA  and Theorems~\ref{SCO_ConstantStep} and \ref{SCO_DecreasingStep} for SCO. We present our results in Figure~\ref{fig:quadratic_game_decreasing}. As predicted from our theoretical analysis for both methods, the decreasing step-size (switching step-size rule) reaches higher precision compare to the constant step-size. In Figure~\ref{SGDAplot_switch} the vertical red line denotes the value $4\lceil \left.\ell_\xi\right/\mu \rceil$ predicted in Theorem~\ref{SGDA_DecreasingStep}  while in Figure~\ref{SCOplot_switch} the red line denotes the value $\left\lceil 8 \frac{\max\{  \ell_{\xi},\cL_{\cH}\}}{\mu_{\cH}+\mu} \right\rceil$ predicted in Theorem~\ref{SCO_DecreasingStep}. Note that for both algorithms the red line is a good approximation of the point where SGDA and SCO need to change their update rules from constant to decreasing step-size.
\vspace{-2mm}
\section{Conclusion and Future Directions of Research}
\vspace{-2mm}
We provided the first last-iterate convergence analysis of SGDA and SCO without requiring any strong bounded noise assumption, by introducing the much weaker expected co-coercivity assumption. We proved last-iterate convergence for both methods for a class of unconstrained variational inequality problems that are potentially non-monotone (quasi-strongly monotone problems), with both constant and decreasing step-sizes. Future work includes extending our results beyond the $\mu$-quasi strongly monotone assumption for SGDA (assuming only co-coercivity), where we expect to obtain a slower sublinear rate. We also believe the proposal of expected co-coercivity to be of independent interest; it could be used to provide an efficient analysis of other algorithms to solve~\eqref{VI} under the arbitrary sampling paradigm, and it would also be interesting to generalize the analysis to the constrained formulation of variational inequalities. 

\section*{Acknowledgments and Disclosure of Funding}

\textbf{Funding.} Nicolas Loizou acknowledges support by the IVADO Postdoctoral Funding Program. This research was partially supported by the Canada CIFAR AI Chair Program and by a Google Focused Research award. Simon Lacoste-Julien is a CIFAR
Associate Fellow in the Learning in Machines \& Brains program.

\textbf{Competing interests.}  Simon Lacoste-Julien additionally works part time as the head of the SAIT AI Lab, Montreal from Samsung. Hugo Berard was working part time as a research intern at Facebook, Montreal.

{
\small
\bibliographystyle{abbrvnat}
\bibliography{SGDA_SCO}

\begin{thebibliography}{75}
\providecommand{\natexlab}[1]{#1}
\providecommand{\url}[1]{\texttt{#1}}
\expandafter\ifx\csname urlstyle\endcsname\relax
  \providecommand{\doi}[1]{doi: #1}\else
  \providecommand{\doi}{doi: \begingroup \urlstyle{rm}\Url}\fi

\bibitem[Abernethy et~al.(2021)Abernethy, Lai, and Wibisono]{abernethy2021last}
J.~Abernethy, K.~A. Lai, and A.~Wibisono.
\newblock Last-iterate convergence rates for min-max optimization: Convergence
  of {H}amiltonian gradient descent and consensus optimization.
\newblock In \emph{ALT}, 2021.

\bibitem[Assran et~al.(2019)Assran, Loizou, Ballas, and
  Rabbat]{assran2018stochastic}
M.~Assran, N.~Loizou, N.~Ballas, and M.~Rabbat.
\newblock Stochastic gradient push for distributed deep learning.
\newblock \emph{ICML}, 2019.

\bibitem[Azizian et~al.(2020)Azizian, Mitliagkas, Lacoste-Julien, and
  Gidel]{azizian2019tight}
W.~Azizian, I.~Mitliagkas, S.~Lacoste-Julien, and G.~Gidel.
\newblock A tight and unified analysis of gradient-based methods for a whole
  spectrum of differentiable games.
\newblock In \emph{AISTATS}, 2020.

\bibitem[Balamurugan and Bach(2016)]{balamurugan2016stochastic}
P.~Balamurugan and F.~Bach.
\newblock Stochastic variance reduction methods for saddle-point problems.
\newblock In \emph{NeurIPS}, 2016.

\bibitem[Balduzzi et~al.(2018)Balduzzi, Racaniere, Martens, Foerster, Tuyls,
  and Graepel]{balduzzi2018mechanics}
D.~Balduzzi, S.~Racaniere, J.~Martens, J.~Foerster, K.~Tuyls, and T.~Graepel.
\newblock The mechanics of n-player differentiable games.
\newblock In \emph{ICML}, 2018.

\bibitem[Bauschke et~al.(2011)Bauschke, Combettes, et~al.]{bauschke2011convex}
H.~H. Bauschke, P.~L. Combettes, et~al.
\newblock \emph{Convex analysis and monotone operator theory in Hilbert
  spaces}, volume 408.
\newblock Springer, 2011.

\bibitem[Bottou et~al.(2018)Bottou, Curtis, and
  Nocedal]{bottou2018optimization}
L.~Bottou, F.~E. Curtis, and J.~Nocedal.
\newblock Optimization methods for large-scale machine learning.
\newblock \emph{SIAM Review}, 60\penalty0 (2):\penalty0 223--311, 2018.

\bibitem[Brighi and John(2002)]{brighi2002characterizations}
L.~Brighi and R.~John.
\newblock Characterizations of pseudomonotone maps and economic equilibrium.
\newblock \emph{Journal of Statistics and Management Systems}, 5\penalty0
  (1-3):\penalty0 253--273, 2002.

\bibitem[Chavdarova et~al.(2019)Chavdarova, Gidel, Fleuret, and
  Lacoste-Julien]{chavdarova2019reducing}
T.~Chavdarova, G.~Gidel, F.~Fleuret, and S.~Lacoste-Julien.
\newblock Reducing noise in {GAN} training with variance reduced extragradient.
\newblock In \emph{NeurIPS}, 2019.

\bibitem[Chen and Rockafellar(1997)]{chen1997convergence}
G.~H. Chen and R.~T. Rockafellar.
\newblock Convergence rates in forward--backward splitting.
\newblock \emph{SIAM Journal on Optimization}, 7\penalty0 (2):\penalty0
  421--444, 1997.

\bibitem[Choi et~al.(1990)Choi, DeSarbo, and Harker]{choi1990product}
S.~C. Choi, W.~S. DeSarbo, and P.~T. Harker.
\newblock Product positioning under price competition.
\newblock \emph{Management Science}, 36\penalty0 (2):\penalty0 175--199, 1990.

\bibitem[Daskalakis et~al.(2018)Daskalakis, Ilyas, Syrgkanis, and
  Zeng]{daskalakis2017training}
C.~Daskalakis, A.~Ilyas, V.~Syrgkanis, and H.~Zeng.
\newblock Training {GANs} with optimism.
\newblock In \emph{ICLR}, 2018.

\bibitem[Daskalakis et~al.(2021)Daskalakis, Skoulakis, and
  Zampetakis]{daskalakis2021complexity}
C.~Daskalakis, S.~Skoulakis, and M.~Zampetakis.
\newblock The complexity of constrained min-max optimization.
\newblock In \emph{Proceedings of the 53rd Annual ACM SIGACT Symposium on
  Theory of Computing}, pages 1466--1478, 2021.

\bibitem[Davis and Yin(2017)]{davis2017three}
D.~Davis and W.~Yin.
\newblock A three-operator splitting scheme and its optimization applications.
\newblock \emph{Set-valued and variational analysis}, 25\penalty0 (4):\penalty0
  829--858, 2017.

\bibitem[Dem'yanov and Pevnyi(1972)]{demyanov1972GradientMethodSP}
V.~F. Dem'yanov and A.~B. Pevnyi.
\newblock Numerical methods for finding saddle points.
\newblock \emph{USSR Computational Mathematics and Mathematical Physics},
  12\penalty0 (5):\penalty0 11--52, 1972.

\bibitem[Diakonikolas et~al.(2021)Diakonikolas, Daskalakis, and
  Jordan]{diakonikolas2021efficient}
J.~Diakonikolas, C.~Daskalakis, and M.~Jordan.
\newblock Efficient methods for structured nonconvex-nonconcave min-max
  optimization.
\newblock In \emph{AISTATS}, 2021.

\bibitem[Elizarov and Kalimullina(2009)]{elizarov2009maximization}
A.~M. Elizarov and A.~Kalimullina.
\newblock Maximization of the lift/drag ratio of airfoils with a turbulent
  boundary layer: Sharp estimates, approximation, and numerical solutions.
\newblock \emph{Computational Mathematics and Mathematical Physics},
  49\penalty0 (3):\penalty0 559--572, 2009.

\bibitem[Facchinei and Kanzow(2007)]{facchinei2007games}
F.~Facchinei and C.~Kanzow.
\newblock Generalized {N}ash equilibrium problems.
\newblock \emph{4OR}, 5\penalty0 (3):\penalty0 173--210, 2007.

\bibitem[Facchinei and Pang(2007)]{facchinei2007finite}
F.~Facchinei and J.-S. Pang.
\newblock \emph{Finite-dimensional variational inequalities and complementarity
  problems}.
\newblock Springer Science \& Business Media, 2007.

\bibitem[Friedlander and Schmidt(2012)]{friedlander2012hybrid}
M.~Friedlander and M.~Schmidt.
\newblock Hybrid deterministic-stochastic methods for data fitting.
\newblock \emph{SIAM Journal on Scientific Computing}, 34\penalty0
  (3):\penalty0 A1380--A1405, 2012.

\bibitem[Gidel et~al.(2018)Gidel, Berard, Vignoud, Vincent, and
  Lacoste-Julien]{gidel2018variational}
G.~Gidel, H.~Berard, G.~Vignoud, P.~Vincent, and S.~Lacoste-Julien.
\newblock A variational inequality perspective on generative adversarial
  networks.
\newblock In \emph{ICLR}, 2018.

\bibitem[Golowich et~al.(2020{\natexlab{a}})Golowich, Pattathil, and
  Daskalakis]{golowich2020tight}
N.~Golowich, S.~Pattathil, and C.~Daskalakis.
\newblock Tight last-iterate convergence rates for no-regret learning in
  multi-player games.
\newblock In \emph{NeurIPS}, 2020{\natexlab{a}}.

\bibitem[Golowich et~al.(2020{\natexlab{b}})Golowich, Pattathil, Daskalakis,
  and Ozdaglar]{golowich2020last}
N.~Golowich, S.~Pattathil, C.~Daskalakis, and A.~Ozdaglar.
\newblock Last iterate is slower than averaged iterate in smooth convex-concave
  saddle point problems.
\newblock In \emph{COLT}, 2020{\natexlab{b}}.

\bibitem[Goodfellow et~al.(2014)Goodfellow, Pouget-Abadie, Mirza, Xu,
  Warde-Farley, Ozair, Courville, and Bengio]{goodfellow2014generative}
I.~Goodfellow, J.~Pouget-Abadie, M.~Mirza, B.~Xu, D.~Warde-Farley, S.~Ozair,
  A.~Courville, and Y.~Bengio.
\newblock Generative adversarial nets.
\newblock In \emph{NeurIPS}, 2014.

\bibitem[Gower et~al.(2021)Gower, Sebbouh, and Loizou]{gower2021sgd}
R.~Gower, O.~Sebbouh, and N.~Loizou.
\newblock {SGD} for structured nonconvex functions: Learning rates,
  minibatching and interpolation.
\newblock In \emph{AISTATS}, 2021.

\bibitem[Gower et~al.(2019)Gower, Loizou, Qian, Sailanbayev, Shulgin, and
  Richt{\'a}rik]{gower2019sgd}
R.~M. Gower, N.~Loizou, X.~Qian, A.~Sailanbayev, E.~Shulgin, and
  P.~Richt{\'a}rik.
\newblock {SGD}: General analysis and improved rates.
\newblock In \emph{ICML}, 2019.

\bibitem[Harker and Pang(1990)]{harker1990finite}
P.~T. Harker and J.-S. Pang.
\newblock Finite-dimensional variational inequality and nonlinear
  complementarity problems: a survey of theory, algorithms and applications.
\newblock \emph{Mathematical programming}, 48\penalty0 (1):\penalty0 161--220,
  1990.

\bibitem[Hazan and Kale(2014)]{hazan2014beyond}
E.~Hazan and S.~Kale.
\newblock Beyond the regret minimization barrier: optimal algorithms for
  stochastic strongly-convex optimization.
\newblock \emph{The Journal of Machine Learning Research}, 15\penalty0
  (1):\penalty0 2489--2512, 2014.

\bibitem[Hsieh et~al.(2020)Hsieh, Iutzeler, Malick, and
  Mertikopoulos]{hsieh2020explore}
Y.-G. Hsieh, F.~Iutzeler, J.~Malick, and P.~Mertikopoulos.
\newblock Explore aggressively, update conservatively: Stochastic extragradient
  methods with variable stepsize scaling.
\newblock In \emph{NeurIPS}, 2020.

\bibitem[Juditsky et~al.(2011)Juditsky, Nemirovski, and
  Tauvel]{juditsky2011solving}
A.~Juditsky, A.~Nemirovski, and C.~Tauvel.
\newblock Solving variational inequalities with stochastic mirror-prox
  algorithm.
\newblock \emph{Stochastic Systems}, 1\penalty0 (1):\penalty0 17--58, 2011.

\bibitem[Kannan and Shanbhag(2019)]{kannan2019optimal}
A.~Kannan and U.~V. Shanbhag.
\newblock Optimal stochastic extragradient schemes for pseudomonotone
  stochastic variational inequality problems and their variants.
\newblock \emph{Computational Optimization and Applications}, 74\penalty0
  (3):\penalty0 779--820, 2019.

\bibitem[Karimi et~al.(2016)Karimi, Nutini, and Schmidt]{karimi2016linear}
H.~Karimi, J.~Nutini, and M.~Schmidt.
\newblock Linear convergence of gradient and proximal-gradient methods under
  the {P}olyak-{\l}ojasiewicz condition.
\newblock In \emph{ECML-PKDD}, 2016.

\bibitem[Khaled and Richt{\'a}rik(2020)]{khaled2020better}
A.~Khaled and P.~Richt{\'a}rik.
\newblock Better theory for {SGD} in the nonconvex world.
\newblock \emph{arXiv preprint arXiv:2002.03329}, 2020.

\bibitem[Khaled et~al.(2020)Khaled, Sebbouh, Loizou, Gower, and
  Richtárik]{khaled2020unified}
A.~Khaled, O.~Sebbouh, N.~Loizou, R.~M. Gower, and P.~Richtárik.
\newblock Unified analysis of stochastic gradient methods for composite convex
  and smooth optimization.
\newblock \emph{arXiv preprint arXiv:2006.11573}, 2020.

\bibitem[Koloskova et~al.(2020)Koloskova, Loizou, Boreiri, Jaggi, and
  Stich]{koloskova2020unified}
A.~Koloskova, N.~Loizou, S.~Boreiri, M.~Jaggi, and S.~U. Stich.
\newblock A unified theory of decentralized {SGD} with changing topology and
  local updates.
\newblock \emph{ICML}, 2020.

\bibitem[Li and Yuan(2017)]{li2017convergence}
Y.~Li and Y.~Yuan.
\newblock Convergence analysis of two-layer neural networks with relu
  activation.
\newblock In \emph{NeurIPS}, 2017.

\bibitem[Liang and Stokes(2019)]{liang2019interaction}
T.~Liang and J.~Stokes.
\newblock Interaction matters: A note on non-asymptotic local convergence of
  generative adversarial networks.
\newblock In \emph{AISTATS}, 2019.

\bibitem[Lin et~al.(2020{\natexlab{a}})Lin, Jin, and Jordan]{lin2020gradient}
T.~Lin, C.~Jin, and M.~Jordan.
\newblock On gradient descent ascent for nonconvex-concave minimax problems.
\newblock In \emph{ICML}, 2020{\natexlab{a}}.

\bibitem[Lin et~al.(2020{\natexlab{b}})Lin, Zhou, Mertikopoulos, and
  Jordan]{lin2020finite}
T.~Lin, Z.~Zhou, P.~Mertikopoulos, and M.~Jordan.
\newblock Finite-time last-iterate convergence for multi-agent learning in
  games.
\newblock In \emph{ICML}, 2020{\natexlab{b}}.

\bibitem[Lions and Mercier(1979)]{lions1979splitting}
P.-L. Lions and B.~Mercier.
\newblock Splitting algorithms for the sum of two nonlinear operators.
\newblock \emph{SIAM Journal on Numerical Analysis}, 16\penalty0 (6):\penalty0
  964--979, 1979.

\bibitem[Loizou and Richt{\'a}rik(2020{\natexlab{a}})]{loizou2020convergence}
N.~Loizou and P.~Richt{\'a}rik.
\newblock Convergence analysis of inexact randomized iterative methods.
\newblock \emph{SIAM Journal on Scientific Computing}, 42\penalty0
  (6):\penalty0 A3979--A4016, 2020{\natexlab{a}}.

\bibitem[Loizou and Richt{\'a}rik(2020{\natexlab{b}})]{loizou2020momentum}
N.~Loizou and P.~Richt{\'a}rik.
\newblock Momentum and stochastic momentum for stochastic gradient, newton,
  proximal point and subspace descent methods.
\newblock \emph{Computational Optimization and Applications}, 77\penalty0
  (3):\penalty0 653--710, 2020{\natexlab{b}}.

\bibitem[Loizou et~al.(2020)Loizou, Berard, Jolicoeur-Martineau, Vincent,
  Lacoste-Julien, and Mitliagkas]{loizou2020stochastic}
N.~Loizou, H.~Berard, A.~Jolicoeur-Martineau, P.~Vincent, S.~Lacoste-Julien,
  and I.~Mitliagkas.
\newblock Stochastic {H}amiltonian gradient methods for smooth games.
\newblock In \emph{ICML}, 2020.

\bibitem[Loizou et~al.(2021)Loizou, Vaswani, Laradji, and
  Lacoste-Julien]{loizou2020stochasticB}
N.~Loizou, S.~Vaswani, I.~Laradji, and S.~Lacoste-Julien.
\newblock Stochastic polyak step-size for {SGD}: An adaptive learning rate for
  fast convergence.
\newblock \emph{AISTATS}, 2021.

\bibitem[Luo et~al.(2020)Luo, Ye, Huang, and Zhang]{luo2020stochastic}
L.~Luo, H.~Ye, Z.~Huang, and T.~Zhang.
\newblock Stochastic recursive gradient descent ascent for stochastic
  nonconvex-strongly-concave minimax problems.
\newblock \emph{NeurIPS}, 2020.

\bibitem[Mertikopoulos and Zhou(2019)]{mertikopoulos2019games}
P.~Mertikopoulos and Z.~Zhou.
\newblock Learning in games with continuous action sets and unknown payoff
  functions.
\newblock \emph{Mathematical Programming}, 173\penalty0 (1):\penalty0 465--507,
  2019.

\bibitem[Mescheder et~al.(2017)Mescheder, Nowozin, and
  Geiger]{mescheder2017numerics}
L.~Mescheder, S.~Nowozin, and A.~Geiger.
\newblock The numerics of {GAN}.
\newblock In \emph{NeurIPS}, 2017.

\bibitem[Mishchenko et~al.(2020)Mishchenko, Kovalev, Shulgin, Richt{\'a}rik,
  and Malitsky]{mishchenko2020revisiting}
K.~Mishchenko, D.~Kovalev, E.~Shulgin, P.~Richt{\'a}rik, and Y.~Malitsky.
\newblock Revisiting stochastic extragradient.
\newblock In \emph{AISTATS}, 2020.

\bibitem[Mokhtari et~al.(2020)Mokhtari, Ozdaglar, and
  Pattathil]{mokhtari2020unified}
A.~Mokhtari, A.~Ozdaglar, and S.~Pattathil.
\newblock A unified analysis of extra-gradient and optimistic gradient methods
  for saddle point problems: Proximal point approach.
\newblock In \emph{AISTATS}, 2020.

\bibitem[Necoara et~al.(2018)Necoara, Nesterov, and
  Glineur]{Necoara-Nesterov-Glineur-2018-linear-without-strong-convexity}
I.~Necoara, Y.~Nesterov, and F.~Glineur.
\newblock Linear convergence of first order methods for non-strongly convex
  optimization.
\newblock \emph{Math. Program.}, pages 1--39, 2018.

\bibitem[Nemirovski et~al.(2009)Nemirovski, Juditsky, Lan, and
  Shapiro]{Nemirovski-Juditsky-Lan-Shapiro-2009}
A.~Nemirovski, A.~Juditsky, G.~Lan, and A.~Shapiro.
\newblock Robust stochastic approximation approach to stochastic programming.
\newblock \emph{SIAM Journal on Optimization}, 19\penalty0 (4):\penalty0
  1574--1609, 2009.

\bibitem[Nesterov(2013)]{nesterov2013introductory}
Y.~Nesterov.
\newblock \emph{Introductory Lectures on Convex Optimization: A Basic Course},
  volume~87.
\newblock Springer Science \& Business Media, 2013.

\bibitem[Nguyen et~al.(2018)Nguyen, Nguyen, van Dijk, Richt\'{a}rik,
  Scheinberg, and Tak\'{a}\v{c}]{pmlr-v80-nguyen18c}
L.~Nguyen, P.~H. Nguyen, M.~van Dijk, P.~Richt\'{a}rik, K.~Scheinberg, and
  M.~Tak\'{a}\v{c}.
\newblock {SGD} and {H}ogwild! {C}onvergence without the bounded gradients
  assumption.
\newblock In \emph{ICML}, 2018.

\bibitem[Palaniappan and Bach(2016)]{palaniappan2016stochastic}
B.~Palaniappan and F.~Bach.
\newblock Stochastic variance reduction methods for saddle-point problems.
\newblock In \emph{NeurIPS}, 2016.

\bibitem[Patel and Zhang(2021)]{patel2021stochastic}
V.~Patel and S.~Zhang.
\newblock Stochastic gradient descent on nonconvex functions with general noise
  models.
\newblock \emph{arXiv preprint arXiv:2104.00423}, 2021.

\bibitem[Pearlmutter(1994)]{pearlmutter1994fast}
B.~A. Pearlmutter.
\newblock Fast exact multiplication by the {H}essian.
\newblock \emph{Neural computation}, 6\penalty0 (1):\penalty0 147--160, 1994.

\bibitem[Pfau and Vinyals(2016)]{pfau2016connecting}
D.~Pfau and O.~Vinyals.
\newblock Connecting generative adversarial networks and actor-critic methods.
\newblock \emph{arXiv preprint arXiv:1610.01945}, 2016.

\bibitem[Rakhlin et~al.(2012)Rakhlin, Shamir, and Sridharan]{rakhlin2012making}
A.~Rakhlin, O.~Shamir, and K.~Sridharan.
\newblock Making gradient descent optimal for strongly convex stochastic
  optimization.
\newblock In \emph{ICML}, 2012.

\bibitem[Recht et~al.(2011)Recht, Re, Wright, and Niu]{recht2011hogwild}
B.~Recht, C.~Re, S.~Wright, and F.~Niu.
\newblock Hogwild: A lock-free approach to parallelizing stochastic gradient
  descent.
\newblock In \emph{NeurIPS}, 2011.

\bibitem[Richt{\'a}rik and Tak{\'a}{\v{c}}(2016)]{richtarik2016optimal}
P.~Richt{\'a}rik and M.~Tak{\'a}{\v{c}}.
\newblock On optimal probabilities in stochastic coordinate descent methods.
\newblock \emph{Optimization Letters}, 10\penalty0 (6):\penalty0 1233--1243,
  2016.

\bibitem[Richt{\'a}rik and Tak{\'a}c(2020)]{richtarik2020stochastic}
P.~Richt{\'a}rik and M.~Tak{\'a}c.
\newblock Stochastic reformulations of linear systems: algorithms and
  convergence theory.
\newblock \emph{SIAM Journal on Matrix Analysis and Applications}, 41\penalty0
  (2):\penalty0 487--524, 2020.

\bibitem[Rosasco et~al.(2014)Rosasco, Villa, and V{\~u}]{rosasco2014stochastic}
L.~Rosasco, S.~Villa, and B.~C. V{\~u}.
\newblock A stochastic forward-backward splitting method for solving monotone
  inclusions in hilbert spaces.
\newblock \emph{arXiv preprint arXiv:1403.7999}, 2014.

\bibitem[Rousseau et~al.(2005)Rousseau, Sharer, Pagerit, and
  Das]{rousseau2005trade}
A.~Rousseau, P.~Sharer, S.~Pagerit, and S.~Das.
\newblock Trade-off between fuel economy and cost for advanced vehicle
  configurations.
\newblock In \emph{20th International Electric Vehicle Symposium (EVS20),
  Monaco}, volume~5, 2005.

\bibitem[Schmidt et~al.(2017)Schmidt, Le~Roux, and Bach]{schmidt2017minimizing}
M.~Schmidt, N.~Le~Roux, and F.~Bach.
\newblock Minimizing finite sums with the stochastic average gradient.
\newblock \emph{Math. Program.}, 162\penalty0 (1-2):\penalty0 83--112, 2017.

\bibitem[Scutari et~al.(2010)Scutari, Palomar, Facchinei, and
  Pang]{scutari2010games}
G.~Scutari, D.~P. Palomar, F.~Facchinei, and J.-S. Pang.
\newblock Convex optimization, game theory, and variational inequality theory.
\newblock \emph{IEEE Signal Processing Magazine}, 27\penalty0 (3):\penalty0
  35--49, 2010.

\bibitem[Song et~al.(2020)Song, Zhou, Zhou, Jiang, and Ma]{song2020optimistic}
C.~Song, Z.~Zhou, Y.~Zhou, Y.~Jiang, and Y.~Ma.
\newblock Optimistic dual extrapolation for coherent non-monotone variational
  inequalities.
\newblock \emph{NeurIPS}, 2020.

\bibitem[Tran~Dinh et~al.(2020)Tran~Dinh, Liu, and Nguyen]{tran2020hybrid}
Q.~Tran~Dinh, D.~Liu, and L.~Nguyen.
\newblock Hybrid variance-reduced {SGD} algorithms for minimax problems with
  nonconvex-linear function.
\newblock \emph{NeurIPS}, 2020.

\bibitem[Tseng(1995)]{tseng1995linear}
P.~Tseng.
\newblock On linear convergence of iterative methods for the variational
  inequality problem.
\newblock \emph{Journal of Computational and Applied Mathematics}, 60\penalty0
  (1-2):\penalty0 237--252, 1995.

\bibitem[Vaswani et~al.(2018)Vaswani, Bach, and Schmidt]{vaswani2018fast}
S.~Vaswani, F.~Bach, and M.~Schmidt.
\newblock Fast and faster convergence of {SGD} for over-parameterized models
  and an accelerated perceptron.
\newblock \emph{arXiv preprint arXiv:1810.07288}, 2018.

\bibitem[V{\~u}(2013)]{vu2013splitting}
B.~C. V{\~u}.
\newblock A splitting algorithm for dual monotone inclusions involving
  cocoercive operators.
\newblock \emph{Advances in Computational Mathematics}, 38\penalty0
  (3):\penalty0 667--681, 2013.

\bibitem[Wen et~al.(2014)Wen, Yu, and Greiner]{wen2014robust}
J.~Wen, C.-N. Yu, and R.~Greiner.
\newblock Robust learning under uncertain test distributions: Relating
  covariate shift to model misspecification.
\newblock In \emph{ICML}, 2014.

\bibitem[Yang et~al.(2020)Yang, Kiyavash, and He]{yang2020global}
J.~Yang, N.~Kiyavash, and N.~He.
\newblock Global convergence and variance-reduced optimization for a class of
  nonconvex-nonconcave minimax problems.
\newblock \emph{NeurIPS}, 2020.

\bibitem[Zhou et~al.(2017)Zhou, Mertikopoulos, Bambos, Boyd, and
  Glynn]{zhou2017stochastic}
Z.~Zhou, P.~Mertikopoulos, N.~Bambos, S.~Boyd, and P.~W. Glynn.
\newblock Stochastic mirror descent in variationally coherent optimization
  problems.
\newblock \emph{NeurIPS}, 2017.

\bibitem[Zhou et~al.(2021)Zhou, Mertikopoulos, Moustakas, Bambos, and
  Glynn]{zhou2021robust}
Z.~Zhou, P.~Mertikopoulos, A.~L. Moustakas, N.~Bambos, and P.~Glynn.
\newblock Robust power management via learning and game design.
\newblock \emph{Operations Research}, 69\penalty0 (1):\penalty0 331--345, 2021.

\bibitem[Zhu and Marcotte(1996)]{zhu1996co}
D.~L. Zhu and P.~Marcotte.
\newblock Co-coercivity and its role in the convergence of iterative schemes
  for solving variational inequalities.
\newblock \emph{SIAM Journal on Optimization}, 6\penalty0 (3):\penalty0
  714--726, 1996.

\end{thebibliography}
}
\newpage

\appendix 

\part*{Supplementary Material}
The supplementary material is organized as follows: In Section~\ref{AppendixEC}, we give some basic definitions and provide the proofs of the propositions, lemmas and theorems related to the expected
co-coercivity condition as presented in Section~\ref{Section_ExpectedCoCo} of the main paper. 
In Section~\ref{AppendixProofs} we present the proofs of the main
theorems and corollaries for the convergence of SGDA and SCO. 
In Section~\ref{Appendix_Experiments} we present the experimental details and provide additional experiments. Finally in Section~\ref{Appendix_BeyondFiniteSum} we explain how our convergence results can be easily adapted to the general stochastic setting and in Section~\ref{Appendix_MoreRelatedWork} we provide further related work.

\tableofcontents

\section{Proofs of Results on Co-coercivity and Expected Co-coercivity}
\label{AppendixEC}

Let us start by re-stating the main definitions of the classes of operators under study.
\begin{definition}[Lipschitz continuous]
An operator $\xi: \R^d \rightarrow \R^d$ is $L-$Lipschitz continuous if there is $L>0$ such that:
\begin{equation}
\label{DefLipschitz}
\|\xi(x)-\xi(y)\| \leq L \|x-y\|, \quad \forall x, y \in \R^d
\end{equation}
\end{definition}

\begin{definition}[Co-coercivity]
\label{DefCOCO}
We say that an operator $\xi$ is $\ell$--co-coercive if there exist $\ell>0$ such that:
$$\|\xi(x)-\xi(y)\|^2 \leq \ell \langle\xi(x)-\xi(y),x-y\rangle \quad \forall x , y \in \R^d.$$
\end{definition}

\begin{definition}[Co-coercive \emph{around} $w^*$]
\label{DefCOCOStar}
We say that an operator $\xi$ is $\ell$--co-coercive \emph{around} $w^*$ if there exist $w^* \in \R^d$ and $\ell>0$ such that $$\|\xi(x)-\xi(w^*)\|^2 \leq \ell \langle\xi(x)-\xi(w^*),x-w^*\rangle \quad \forall x \in \R^d.$$Note that in this definition, the point $w^*$ is not necessarily a point where $\xi(w^*)=0$.
\end{definition}

\begin{definition}[Strongly monotone / monotone]
\label{DefSM}
We say that an operator $\xi$ is $\mu$--strongly monotone if there exist $\mu>0$ such that $$\left\langle\xi(x)-\xi(y),  x-y\right\rangle \geq \mu \|x-y\|^2 \quad \forall x,y \in \R^d.$$ 
If $\mu=0$, that is $$\left\langle\xi(x)-\xi(y),  x-y\right\rangle \geq 0 \quad \forall x,y \in \R^d,$$ then we say that the operator is monotone.
\end{definition}

\begin{definition}[Quasi-Strongly Monotone / Variational Stability Condition]
\label{DefQSM}
We say that an operator $\xi$ is $\mu$-\emph{quasi-strongly monotone} if there exist $\mu>0$ such that $$\left\langle\xi(x),  x-x^*\right\rangle \geq \mu \|x-x^*\|^2 \quad \forall x \in \R^d.$$ 
Here $x^*$ is the solution of the stochastic variational inequality problem~\eqref{VI}.
If $\mu=0$, that is 
\begin{equation}
\label{DefVarStab}
\langle \xi(x), x-x^*\rangle\geq 0
\end{equation}
then we say that $\xi$ satisfies the variational stability condition.
\end{definition}

\subsection{Proof of Proposition~\ref{PropositionCocoMonotone}}

Before stating the proof of Proposition~\ref{PropositionCocoMonotone}, we clarify that the assumption of $L$-Lipschitzness of $\xi$ is only used for the implications where $L$ appear; it is not needed for the other implications. In particular, while a $\ell$-co-coercive operator is always $\ell$-Lipschitz continuous by using Cauchy-Schwartz,\footnote{$\| \xi(x) - \xi(x')\|^2 \leq \ell \langle \xi(x)-\xi(x'), x-x' \rangle \leq \ell  \| \xi(x)-\xi(x') \| \|  x-x' \|$ $\Longrightarrow$  $\| \xi(x) - \xi(x')\| \leq \ell \|  x-x' \|$.} it is possible for an operator to be $\ell$-co-coercive \emph{around $x^*$} and \emph{not} be Lipschitz continuous (see such an example in Section~\ref{App:coolNonMonotoneExample}). This highlights the wider applicability of the $\ell$-co-coercivity around $x^*$ assumption that is all we need for several of our convergence results, in contrast to the Lipschitz continuity of $\xi$ which is typically assumed in the variational inequality literature. 

\begin{proof}
Most of these implications can be found in~\citet{facchinei2007finite}.

$\mu\text{-strongly monotone}  \Longrightarrow  \frac{L^2}{\mu}\text{-co-coercive}$: 
The proof of this result is a direct application of strong monotonicity and Lipschitzness properties:
\begin{equation*}
  \|\xi(x)-\xi(x')\|^2 \leq L^2 \|x-x'\|^2 \leq \tfrac{L^2}{\mu} \langle \xi(x)-\xi(x'), x-x' \rangle \, , \qquad \forall x,x' \in \R^d\,.
\end{equation*}

$\ell\text{-co-coercive} \Longrightarrow \text{ monotone}$: It comes from the fact that a norm is non-negative. 
\begin{equation*}
  0 \leq \|\xi(x)-\xi(x')\|^2  \leq \ell \langle \xi(x)-\xi(x'), x-x' \rangle \, , \qquad \forall x,x' \in \R^d\,.
\end{equation*}

$\text{monotone } \Longrightarrow \text{ variational stability condition}$: It comes from the fact that monotonicity applied to $x' = x^*$ is variational stability condition.

$\text{Quasi }\mu\text{-strongly monotone}  \Longrightarrow  \frac{L^2}{\mu}\text{-co-coercive relatively to } x^*$: (It is the only implication that is not proven in~\citet{facchinei2007finite}) The proof of this result is a direct application of quasi-strong monotonicity and Lipschitzness properties:
\begin{equation*}
  \|\xi(x)-\xi(x^*)\|^2 \leq L^2 \|x-x^*\|^2 \leq \tfrac{L^2}{\mu} \langle \xi(x), x-x^* \rangle \, , \qquad \forall x \in \R^d\,.
\end{equation*}

$\frac{L^2}{\mu}\text{-co-coercive relatively to } x^* \Longrightarrow \text{ variational stability condition}$: It comes from the fact that a norm is non-negative. 
\begin{equation*}
  0 \leq \|\xi(x)-\xi(x^*)\|^2  \leq \ell \langle \xi(x), x-x^* \rangle \, , \qquad \forall x \in \R^d\,.
\end{equation*}

\end{proof}

\subsection{Proof of Lemma~\ref{MainLemma}}
\begin{proof}
\begin{eqnarray}
\label{oansxa}
\Exp_{\cD} \|\xi_{v} (x)\|^2 &=& \Exp_{\cD} \| \xi_v(x) - \xi_{v}(x^*) + \xi_{v}(x^*) \|^2 \notag\\ 
&\leq &2 \Exp_{\cD} \| \xi_{v}(x) - \xi_{v}(x^*)\|^2 + 2  \Exp_{\cD} \|\xi_{v}(x^*)\|^2 \notag\\ 
&\overset{\ref{eq:ExpCoCo}}{\leq}&2 \ell_{\xi}\langle \xi (x),x-x^*\rangle + 2 \Exp_{\cD} \|\xi_{v}(x^*)\|^2\notag\\ 
&\overset{\eqref{Sigma}}{\leq} &2 \ell_{\xi}\langle \xi (x),x-x^*\rangle + 2 \sigma^2.
\end{eqnarray}
The first inequality follows from the estimate $\|a+b\|^2 \leq 2\|a\|^2 + 2\|b\|^2$.
\end{proof}

\subsection{Proof of Proposition~\ref{PropositionMinibatch}}
Before we formally present the proof of Proposition~\ref{PropositionMinibatch}, let us first establish some random set terminology. 

Let $C\subseteq [n]$ and let $e_C \eqdef \sum_{i\in C} e_i$, where $\{e_1,\dots,e_n\}$ are the standard basis vectors in $\R^n$.  These subsets will be selected using a random set valued map $S$, in the literature referred to by the name {\em sampling}.
A sampling is uniquely characterized by choosing subset probabilities $p_C\geq 0$ for all subsets $C$ of $[n]$:
\begin{equation}
\Prob{S = C} = p_C,\quad \forall C \subset [n],
\end{equation}
where $\sum_{C\subseteq [n]} p_C =1$. In this work, following the terminology of \cite{gower2019sgd,gower2021sgd}, our results hold for \emph{proper} samplings. 
\begin{definition}\label{ProperSampling}A sampling $S$ is called proper if  $p_i \overset{\rm def}{=} \mathbb{P}[i\in S] = \sum_{C:i\in C}p_C$ is positive for all $i$. \end{definition}

As we mentioned in the main paper, in this work we focus on $b$-minibatch sampling (see Definition~\ref{def:minibatch}) however we highlight again that our results hold for the larger class of sampling vectors $v \in \R^n$ that satisfy $\Exp_{\cD}[v_i]  = 1, \,\mbox{for }i=1,\ldots, n$.

For example, the random vector $v=v(S)$ given by $v = \sum_{i\in S} \frac{1}{p_i} e_i$ is a sampling vector. This can be easily proved, by noticing that $v_i =  \mathbf{1}_{(i\in S)}/p_i,$ where $\mathbf{1}_{(i\in S)}$ is the indicator function of the event $i\in S$. Then, It follows that $\E{v_i} = \E{\mathbf{1}_{(i\in S)}}/p_i = 1$. 
Commonly used samplings that captured by our theory are the \textit{independent sampling, partition sampling, single-element sampling and importance sampling}. For more details on these different samplings check~\cite{gower2019sgd}.

By definition~\ref{def:minibatch} of $b$-minibatch sampling, it holds that $\Prob{i\in S} = p_i = \frac{b}{n}$, and  $\Prob{i,j\in S} = \frac{b}{n}\frac{b-1}{n-1} $.

All the sampling schemes presented in~\citet{gower2019sgd} had the following additional property: there exists a constant $z>0$ such that 
\begin{equation}\label{eq:indep}
 \frac{\Prob{i,j \in S}}{\Prob{i \in S} \Prob{j \in S}} = z, \quad \forall i,j \in \{1,\ldots, n\}, \; i\neq j. 
\end{equation}
For $b$-minibatch sampling, $z =\frac{n}{b}\frac{b-1}{n-1}$~\citep{gower2021sgd}.

Let us now present the proof of Proposition~\ref{PropositionMinibatch}.

\begin{proof}
Since $\xi$ is $\ell_i$--co-coercive \emph{around} $x^*$ then we have that $\xi$ is $\ell$--co-coercive \emph{around} $x^*$. That is $ \forall x \in \R^d$ it holds,
\begin{align}
\label{naosknal1}
\|\xi_i(x)-\xi_i(x^*)\|^2 \leq \ell_i \langle\xi_i(x)-\xi_i(x^*),x-x^*\rangle\\
\label{naosknal2}
\|\xi(x)-\xi(x^*)\|^2 \leq \ell \langle\xi(x)-\xi(x^*),x-x^*\rangle= \ell \langle\xi(x),x-x^*\rangle.
\end{align}

Noticing that 
\begin{eqnarray*}
\|\xi_v(x) - \xi_v(x^*)\|^2 &=& \frac{1}{n^2} \left \|\sum_{i\in S}\frac{1}{p_i}(\xi_i(x) - \xi_i(x^*)) \right \|^2 \\
&=& \sum_{i,j\in S} \left\langle \frac{1}{np_i}(\xi_i(x) - \xi_i(x^*)), \frac{1}{np_j}(\xi_j(x) - \xi_j(x^*)) \right\rangle,
\end{eqnarray*}
we have 
\begin{eqnarray*}
\mathbb{E}[\|\xi_v(x) - \xi_v(x^*)\|^2] &=& \sum_C p_C  \sum_{i,j\in C} \left\langle \frac{1}{np_i}(\xi_i(x) - \xi_i(x^*)), \frac{1}{np_j}(\xi_j(x) - \xi_j(x^*)) \right\rangle \\ 
&=& \sum_{i, j=1}^n \sum_{C: i,j\in C }p_C  \left\langle \frac{1}{np_i}(\xi_i(x) - \xi_i(x^*)), \frac{1}{np_j}(\xi_j(x) - \xi_j(x^*)) \right\rangle \\ 
&=& \sum_{i, j=1}^n \frac{\Prob{i,j \in S}}{p_ip_j} \left\langle \frac{1}{n}(\xi_i(x) - \xi_i(x^*)), \frac{1}{n}(\xi_j(x) - \xi_j(x^*)) \right\rangle,
\end{eqnarray*}
where we used a double counting argument in the 2nd equality.
Now since $\Prob{i,j \in S}/(p_ip_j) =z$ for $i \neq j$ \eqref{eq:indep} and $\Prob{i,i \in S}=p_i$ we have from the above that
\begin{eqnarray}
\mathbb{E}[\|\xi_v(x) - \xi_v(x^*)\|^2] &=& 
\sum_{i \neq j} z \left\langle \frac{1}{n}(\xi_i(x) - \xi_i(x^*)), \frac{1}{n}(\xi_j(x) - \xi_j(x^*)) \right\rangle  \notag\\
& &\qquad + \sum_{i=1}^n\frac{1}{n^2} \frac{1}{p_i} \norm{\xi_i(x) - \xi_i(x^*))}^2 \notag\\
&= &  \sum_{i,j=1}^n z \left\langle \frac{1}{n}(\xi_i(x) - \xi_i(x^*)), \frac{1}{n}(\xi_j(x) - \xi_j(x^*)) \right\rangle \notag\\
\label{nxasnxkl}
& &+ \sum_{i=1}^n\frac{1}{n^2} \frac{1}{p_i}\left(1 -p_i z \right) \norm{\xi_i(x) - \xi_i(x^*))}^2\\
& \overset{\eqref{naosknal1}}{\leq} &  z \norm{\xi(x) - \xi(x^*)}^2 \notag\\
& &+  \sum_{i=1}^n\frac{1}{n^2} \frac{\ell_i }{p_i}\left(1 -p_i z \right) \langle\xi_i(x)-\xi_i(x^*),x-x^*\rangle \notag\\
& \overset{\eqref{naosknal2}}{\leq} & \left(z \ell +\max_{i=1,\ldots, n}\frac{\ell_i}{np_i}\left(1 -p_i z \right)  \right) \langle\xi(x),x-x^*\rangle.
\end{eqnarray}
Comparing the above to the definition of expected co-coercivity~\eqref{eq:ExpCoCo} we have that
\begin{equation} \label{eq:CLinterpolc2}
\ell_{\xi} =  z \ell +\max_{i=1,\ldots, n}\frac{\ell_i}{np_i}\left(1 -p_i z \right).
\end{equation}
Using that for  $b$-minibatch sampling it holds $\Prob{i\in S} = p_i = \frac{b}{n}$ and  $z =\frac{n}{b}\frac{b-1}{n-1}$ we obtain
$$\ell_\xi = \frac{n}{b}\frac{b-1}{n-1}\ell+\frac{1}{b}\frac{n-b}{n-1} \ell_{\max},$$ 
where $\ell_{\max}=\max \{\ell_i\}_{i=1}^n$.

The specialized expressions of $\sigma^2$ for the $b$-minibatch sampling, can be obtain by following the same steps of Proposition 3.10 of \cite{gower2019sgd}. Below using our notation, we include this derivation for completeness:

\begin{eqnarray*}
		\sigma^2 = \mathbb{E}_{\cD}[\|\xi_v(x^*) \|^2] &=& \mathbb{E}\left[ \left\| \frac{1}{n}\sum_{i=1}^n\nabla f_i(x^*)v_i \right\|^2\right] = \frac{1}{n^2}\mathbb{E}\left[\left\| \sum_{i=1}^n \nabla f_i(x^*)v_i \right\|^2\right]= \frac{1}{n^2} \mathbb{E} \left[\left\| \sum_{i\in {S}} \frac{1}{p_i}\xi_i(x^*) \right\|^2\right] \\
		&=& \frac{1}{n^2} \mathbb{E} \left[ \left\| \sum_{i=1}^n 1_{i \in { S}}\frac{1}{p_i}\xi_i(x^*)\right\|^2 \right] = \frac{1}{n^2} \mathbb{E} \left[\sum_{i=1}^n \sum_{j=1}^n 1_{i \in { S}} 1_{j \in { S}} \langle \frac{1}{p_i}\xi_i(x^*), \frac{1}{p_j}\xi_j(x^*) \rangle \right]\\
		&=& \frac{1}{n^2}\sum_{i,j}\frac{\Prob{i,j \in S}}{p_ip_j} \langle \xi_i(x^*), \xi_j(x^*) \rangle .
	\end{eqnarray*}
	
Recall that for $b$-minibatch sampling, $\Prob{i\in S} = p_i = \frac{b}{n}$, and  $\Prob{i,j\in S} = \frac{b}{n}\frac{b-1}{n-1} $. Thus,

\begin{eqnarray*}
		\sigma^2 &=& \frac{1}{n^2} \sum_{i,j\in [n]} \frac{\Prob{i,j \in S}}{p_ip_j} \langle \xi_i(x^*), \xi_j(x^*) \rangle \\
		&=& \frac{1}{n^2} \sum_{i\neq j} \frac{b(b-1)}{n(n-1)}\cdot \frac{n^2}{b^2} \langle \xi_i(x^*), \xi_j(x^*) \rangle + \frac{1}{n^2} \sum_{i\in [n]} \frac{n}{b} \|\xi_i(x^*)\|^2 \\
		&=& \frac{1}{nb} \left( \sum_{i\neq j} \frac{b-1}{n-1}\langle \xi_i(x^*), \xi_j(x^*) \rangle +  \sum_{i\in [n]} \|\xi_i(x^*)\|^2 \right)\\ 
		&=& \frac{1}{nb} \left( \sum_{i,j\in [n]} \frac{b-1}{n-1}\langle \xi_i(x^*), \xi_j(x^*) \rangle +  \sum_{i\in [n]} \frac{n-b}{n-1}\|\xi_i(x^*)\|^2 \right)\\  
		&=& \frac{1}{nb} \cdot \frac{n-b}{n-1} \sum_{i\in [n]} \|\xi_i(x^*)\|^2\\ 
		&=& \frac{1}{b} \cdot \frac{n-b}{n-1}\sigma_1^2,
	\end{eqnarray*}
where $\sigma_1^2 \eqdef  \frac{1}{n} \sum_{i=1}^n \norm{\xi_i(x^*)}^2$.
\end{proof}

\newpage

\subsection{Proof of Proposition~\ref{PropositionExtra}}

\begin{proposition}
Let $\xi$ be $\mu$-quasi-strongly monotone and let $\xi_i$ be $L_i$-Lipschitz continuous for all $i \in [n]$. Then $\xi  \in EC(\ell_\xi)$.

For a general proper sampling scheme (Def.~\ref{ProperSampling}), we can provide the following (loose) bound on $\ell_\xi$:
\begin{equation} \label{eq:generalSchemeEll}
\ell_\xi \leq \frac{1}{n}\sum_{i=1}^n  \frac{\Exp_{\cD}\left[v_i^2\right] L_i^2}{\mu} \, .
\end{equation}
If, as is the case for standard sampling schemes from~\citet{gower2019sgd}, we assume that there exists a $z>0$ such that $z=\frac{\Prob{i,j \in S}}{\Prob{i \in S} \Prob{j \in S}}$  for all $i,j \in \{1,\ldots, n\}, \; i\neq j$, then we can use the tighter value:
\begin{equation}
\ell_\xi=\left(z L^2 + \sum_{i=1}^n\frac{1}{n^2} \frac{L_i^2 }{p_i}\left(1 -p_i z \right)  \right) \frac{1}{\mu} \, ,
\end{equation}
where $L$ is the Lipschitz continuous parameter of operator $\xi$ and $p_i =\Prob{i\in S} $. 
 
Finally, for $b$-minibatch sampling, it holds $p_i = \frac{b}{n}$ and $z =\frac{n}{b}\frac{b-1}{n-1}$ and thus $\xi  \in EC(\ell_\xi)$ with
\begin{equation}
\ell_\xi = \left(\frac{n}{b}\frac{b-1}{n-1}L^2+ \left( \frac{1}{n} \sum_{i=1}^n L_i^2 \right)\frac{1}{b}\frac{n-b}{n-1}  \right)\frac{1}{\mu}.
\end{equation}
\end{proposition}

\begin{proof}
Since $\xi_i$ is $L_i$--Lipschitz continuous for all $i$, then we have that $\xi$ is $L$--Lipschitz continuous (with $L \leq \frac{1}{n} \sum_i L_i$ by using Jensen's inequality on $\|\cdot\|^2$). That is, $ \forall x \in \R^d$ it holds:
\begin{align}
\label{naosknalanaxkosxanx}
\|\xi_i(x)-\xi_i(y)\| \leq L_i \|x-y\|, \quad \forall x, y \in \R^d\\
\label{naosknalanaxkosxanx2}
\|\xi(x)-\xi(y)\| \leq L \|x-y\|, \quad \forall x, y \in \R^d
\end{align}

We first prove the general case:
\begin{eqnarray*}
	\Exp_{\cD}[\|\xi_v(x) - \xi_v(x^*)\|^2] &=& 
	\Exp_{\cD}\left[\|\frac{1}{n}\sum_{i=1}^n v_i \xi_i(x) - \frac{1}{n}\sum_{i=1}^n v_i \xi_i(x^*)\|^2\right]  \\
	&\overset{\text{Jensen's}}{\leq}& \Exp_{\cD}\left[\frac{1}{n}\sum_{i=1}^n \| v_i [\xi_i(x) - \xi_i(x^*)]\|^2\right] \\
	&=& \Exp_{\cD}\left[\frac{1}{n}\sum_{i=1}^n v_i^2 \| \xi_i(x) - \xi_i(x^*)\|^2\right] \\
	&\overset{\eqref{naosknalanaxkosxanx}}{\leq}& \Exp_{\cD}\left[\frac{1}{n}\sum_{i=1}^n v_i^2 L_i^2\| x - x^*\|^2\right] \\
	&=& \frac{1}{n}\sum_{i=1}^n \Exp_{\cD}\left[v_i^2\right] L_i^2\| x - x^*\|^2\\
	& \overset{\eqref{QSM}}{\leq} &\frac{1}{n}\sum_{i=1}^n  \frac{\Exp_{\cD}\left[v_i^2\right] L_i^2}{\mu} \langle\xi(x),x-x^*\rangle ,
\end{eqnarray*}
yielding~\eqref{eq:generalSchemeEll}.

The use of Jensen's inequality above is the source of looseness in the bound. With the $z$ constant property, we can avoid it with the following derivations

By following the same steps to the proof of Proposition~\ref{PropositionMinibatch}, we obtain \eqref{nxasnxkl}. That is,
\begin{eqnarray*}
\mathbb{E}[\|\xi_v(x) - \xi_v(x^*)\|^2] 
&= &  \sum_{i,j=1}^n z \left\langle \frac{1}{n}(\xi_i(x) - \xi_i(x^*)), \frac{1}{n}(\xi_j(x) - \xi_j(x^*)) \right\rangle \notag\\
& &+ \sum_{i=1}^n\frac{1}{n^2} \frac{1}{p_i}\left(1 -p_i z \right) \norm{\xi_i(x) - \xi_i(x^*))}^2 \, .
\end{eqnarray*}

Now by using \eqref{naosknalanaxkosxanx} and \eqref{naosknalanaxkosxanx2},  we obtain the following

\begin{eqnarray}
\mathbb{E}[\|\xi_v(x) - \xi_v(x^*)\|^2] 
&= &  \sum_{i,j=1}^n z \left\langle \frac{1}{n}(\xi_i(x) - \xi_i(x^*)), \frac{1}{n}(\xi_j(x) - \xi_j(x^*)) \right\rangle \notag\\
& &+ \sum_{i=1}^n\frac{1}{n^2} \frac{1}{p_i}\left(1 -p_i z \right) \norm{\xi_i(x) - \xi_i(x^*))}^2 \notag\\
& \overset{\eqref{naosknalanaxkosxanx}}{\leq} &  z \norm{\xi(x) - \xi(x^*)}^2 \notag\\
& &+  \sum_{i=1}^n\frac{1}{n^2} \frac{L_i^2 }{p_i}\left(1 -p_i z \right) \|x-x^*\|^2 \notag\\
& \overset{\eqref{naosknalanaxkosxanx2}}{\leq} & \left(z L^2 + \sum_{i=1}^n\frac{1}{n^2} \frac{L_i^2 }{p_i}\left(1 -p_i z \right)  \right)  \|x-x^*\|^2. \label{eq:Prop3-6-proof-step1}
\end{eqnarray}

Since $\xi$ is $\mu$-quasi strongly monotone, then~\eqref{eq:Prop3-6-proof-step1} becomes:

\begin{eqnarray}
\mathbb{E}[\|\xi_v(x) - \xi_v(x^*)\|^2] 
& \leq & \left(z L^2 + \sum_{i=1}^n\frac{1}{n^2} \frac{L_i^2 }{p_i}\left(1 -p_i z \right)  \right) \frac{1}{\mu}\langle\xi(x),x-x^*\rangle
\end{eqnarray}
Thus the expected co-coercivity~\eqref{eq:ExpCoCo} is satisfied with
\begin{equation}
\ell_{\xi} =  \left(z L^2 + \sum_{i=1}^n\frac{1}{n^2} \frac{L_i^2 }{p_i}\left(1 -p_i z \right)  \right) \frac{1}{\mu}.
\end{equation}
Similar to the proof of Proposition~\ref{PropositionMinibatch}, using that for  $b$-minibatch sampling, $\Prob{i\in S} = p_i = \frac{b}{n}$ and  $z =\frac{n}{b}\frac{b-1}{n-1}$, we obtain:
\begin{equation}
\ell_\xi = \left(\frac{n}{b}\frac{b-1}{n-1}L^2+ \left( \frac{1}{n} \sum_{i=1}^n L_i^2 \right)\frac{1}{b}\frac{n-b}{n-1}  \right)\frac{1}{\mu}.
\end{equation}
This completes the proof.
\end{proof}

\subsection{Connections of EC to other Assumptions}
\label{Appendix_Connections}
In this section, we present some propositions not included in the main paper showing properties and connections between classical assumptions and our proposed expected co-coercivity~\ref{eq:ExpCoCo}.
\begin{proposition}
In the unconstrained setting, if $\xi$ is $\mu$-quasi strongly monotone, then it is not possible to satisfy the bounded operator assumption that there exists a finite $c$ such that $\Exp\|\xi_v(x)\|^2 \leq c$ for every $x$ in $\R^d$.
\end{proposition}
\begin{proof}
Let us assume that $\Exp\|\xi_v(x)\|^2 \leq c, \forall x \in \R^d.$
Note also that if an operator satisfies the $\mu$-quasi strongly monotone property~\eqref{QSM}, then by using the Cauchy–Schwarz inequality, it satisfies the error bound condition $$ \|\xi(x)\| \geq \mu \|x-x^*\|.$$

By combining the above two inequalities, it holds that:
$$ \mu^2 \|x-x^*\|^2 \leq \|\xi(x)\|^2=\|\Exp[\xi_v(x)]\|^2\leq \Exp \left[\|[\xi_v(x)\|^2\right]\leq c $$
which means that: $$ \|x-x^*\|^2\leq  \frac{c}{\mu^2}.$$

However, for the \emph{unconstrained} stochastic variational inequality problems \eqref{VI}, a point $x$ can be very far from the optimum point $x^*$  and as a result  $\|x-x^*\|^2 \geq  \frac{c }{\mu^2}$. This leads to a contradiction.
\end{proof}

\begin{proposition}
In the single-objective optimization when the stochastic problem $$\min_x \left[f(x) = \tfrac{1}{n} \sum_{i=1}^n f_i(x)\right]$$ has convex and smooth functions $f_i$, then expected smoothness and expected co-coercivity are equivalent (see last row of Table~\ref{TableAssumptions}).
\end{proposition}

\begin{proof}
For simplicity of exposition, let us focus on single-element sampling. 

According to Theorem 2.1.5 in~\citet{nesterov2013introductory}, if $f_i$ is convex and $L_i$ smooth, then the following two conditions are equivalent:
\begin{equation}\label{eq:fvas3}
	\|\nabla f_i(x) - \nabla f_i(y)\|^2 \leq 2L_i \left( f_i(x) - f_i(y) - \langle \nabla f_i(y), x-y \rangle \right).
	\end{equation}
	\begin{equation}\label{eq:fvas4}
	\norm{\nabla f_i(x)-\nabla f_i(y)}^2 \leq L_i  \langle \nabla f_i(x)-  \nabla f_i(y),x-y\rangle
	\end{equation}
If in the above two condition we select $y=x^*$  and take the expectations with respect to $i$, then we obtain the following two equivalent conditions:
\begin{equation}\label{eq:fvas5}
	\Exp\left[\|\nabla f_i(x) - \nabla f_i(x^*)\|^2\right] \leq 2 L_{\max} \left[ f(x) - f(x^*) \right].
	\end{equation}
	\begin{equation}\label{eq:fvas6}
	\Exp\left[\|\nabla f_i(x) - \nabla f_i(x^*)\|^2\right] \leq L_{\max} \langle \nabla f(x),x-x^*\rangle
	\end{equation}
Note that in the above, \eqref{eq:fvas5} is the expected smoothness as proposed in \cite{gower2019sgd} while~\eqref{eq:fvas6} is our expected co-coercivity \eqref{eq:ExpCoCo} for the single element sampling.  Note that for single-objective optimization problems, the operator $\xi$ is simply the gradient vector. As we mentioned in Section~\ref{Section_ExpectedCoCo}, in single-objective optimization, a function is $L$--co-coercive if and only if it is convex and $L$-smooth (i.e. $L$-Lipschitz gradients)~\citep{bauschke2011convex}. Thus, the co-coercivity constant is equivalent to the smoothness parameter. 
\end{proof}

Let us also add a simple remark highlighting the weakness of \ref{eq:ExpCoCo} compare to other previously used assumptions in the literature of stochastic algorithms for solving~\eqref{VI}.

\begin{remark}
As we show in Lemma~\eqref{MainLemma}, by assuming \ref{eq:ExpCoCo} we obtain the following bound \begin{equation}\label{ansxoasxa}
\Exp \| \xi_{v} (x)\|^2 \leq 2\ell_{\xi}\langle \xi (x),x-x^*\rangle  + 2 \sigma^2.
\end{equation} Let us now compare this bound to the assumption of growth condition $\Exp\|\xi_i(x)\|^2 \leq c_1 \|\xi(x)\|^2 +c_2$ (weakest among the other assumptions). 

Note that if an operator $\xi$ is $\ell$--co-coercive, then the growth condition implies:
$$\Exp\|\xi_i(x)\|^2 \leq c_1 \ell \langle \xi (x),x-x^*\rangle +c_2.$$

This has the same form to the bound~\eqref{ansxoasxa}, obtained by \ref{eq:ExpCoCo}. However, the parameters $c_1$ and $c_2$ have unknown values while using \ref{eq:ExpCoCo} these parameters are closed-form problem-dependent expressions. 

Thus, expected co-coercivity is weaker than the growth condition and at the same time more powerful, as in many case it is not really an assumption, but a condition that is satisfied for free. See for example, Propositions~\ref{PropositionMinibatch} and \ref{PropositionExtra}.
\end{remark}

\subsection{Example: Quasi-strongly Monotone Operator that is not Monotone nor Lipschitz}
\label{App:coolNonMonotoneExample}
An operator that is $\mu$-quasi strongly monotone may not even be monotone.  We now give a simple example of such an operator, with the additional property that it is $L$-co-coercive around $x^*$ but is \emph{not} Lipschitz continuous. This highlights the generality of our convergence results beyond the standard monotone setting. Let $L> \mu>0$, we define $\xi(x) = x (\frac{L-\mu}{2} \cos(\|x\|_2) + \frac{L+\mu}{2})$. We have that $x^*=0$ and 
\begin{equation}
  \langle \xi(x) , x-x^*\rangle = \|x-x^*\|^2_2 (\frac{L-\mu}{2} \cos(\|x\|_2) + \frac{L+\mu}{2}) \geq \mu \|x-x^*\|_2^2 \, ,
\end{equation}
and is thus $\mu$-quasi strongly monotone. However, it is \emph{not} monotone. To see this, let us consider the one dimensional case $x \in \R$. In this case, we get, $\xi(x) = x (\frac{L-\mu}{2} \cos(|x|) + \frac{L+\mu}{2})$. To formally violate the monotonicity inequality, we can for instance consider $x = 2 \pi k + \frac\pi2$ and $x'= 2 \pi k $ to get
\begin{equation}
 \langle \xi(x) - \xi(x') , x - x' \rangle  = (\tfrac\pi2\tfrac{L+\mu}{2} -  2 \pi k \tfrac{L-\mu}{2})\frac\pi2 =  \frac{\pi^2}4  (L+\mu -4k(L-\mu)) .
 \end{equation}
This quantity is negative for $k > \tfrac{L+\mu}{4(L-\mu)}$.

This operator is also $L$-co-coercive with respect to $x^*$ since, 
\begin{equation}
  \langle \xi(x) , x-x^*\rangle = \|\xi(x)\|_2^2 (\frac{L-\mu}{2} \cos(\|x\|_2) + \frac{L+\mu}{2})^{-1} \geq L^{-1} \|\xi(x)\|_2^2  \, .
\end{equation}

This operator is \emph{not} Lipschitz continuous for all $x \in \R$, as its derivative is unbounded over $\R$.

\subsection{Example: Co-coercivity for Quadratic Games}
For quadratic games, it is relatively easy to characterize co-coercivity. 
\begin{proposition}
	If $\xi(x) = Ax$ where $A \in \R^{d\times d}$, we have that $\xi$ is co-coercive if and only if $\langle x,Ax\rangle > 0\,,\,\forall x\in\R^d \setminus \nulls(A)$. In that case, we have $\ell = \sup_{\|x\|= 1} \frac{\|A x\|_2^2}{ \langle x,Ax\rangle}$.
\end{proposition}
\begin{proof}
	When $\xi(x) = Ax$, the co-coercivity condition (Definition~\ref{def:cocoercivity})
	\begin{equation}
	\|A(x-x')\|^2 \leq \ell \langle A(x-x'),x-x'\rangle \,, \quad \forall x,x' \in \R^d\,.
	\end{equation}
	Now, we can note that the variable of interest is $x-x' \in \R^d$. Thus we get equivalently, 
	\begin{equation}
	\|A x\|^2 \leq \ell \langle Ax,x\rangle \,, \quad \forall x \in \R^d\,.
	\end{equation}
	This inequality is valid for any $x \in \R^d$ such that $Ax = 0$. Now, let us consider $x\in \R^d$ such that $Ax \neq 0$. 
	
	If there exist $x \in \R^d$ such that $Ax \neq 0$ and $\langle x,Ax\rangle \leq 0$ then, we have that 
	\begin{equation}
	0 < \|A x\|^2 \leq \ell \langle Ax,x\rangle \leq 0
	\end{equation}
	which is not valid. Thus $\xi$ is \emph{not} co-coercive.
	
	On the other hand, if for all $x \in \R^d$ such that $Ax \neq 0$, we have $\langle x,Ax\rangle > 0$ then, the co-coercivity condition would demand
	\begin{equation}
	\frac{\|A x\|^2}{\langle Ax,x\rangle} \leq \ell \,, \quad \forall x \in \R^d\, ,\,Ax \neq 0\,.
	\end{equation}
	Finally, since the left-hand side of the previous equation is scale invariant ($x\mapsto \lambda x$ does not change the LHS), we have 
	\begin{equation}
	\sup_{x \in \R^d}\frac{\|A x\|^2}{\langle Ax,x\rangle} 
	= \sup_{x \in \R^d\setminus\{0\}}\frac{\|A x\|^2}{\langle Ax,x\rangle}
	=  \sup_{\|x\|=1}\frac{\|A x\|^2}{\langle Ax,x\rangle}
	\end{equation}
	So we have $\ell \geq \sup_{\|x\|=1}\frac{\|A x\|^2}{\langle Ax,x\rangle}$. Because the RHS is continuous in $x$, and the unit ball is a compact, this quantity is achieved. Thus there exists $x \in \R^d$ such that 
	\begin{equation}
	\|A x\|^2 = \sup_{\|x\|=1}\frac{\|A x\|^2}{\langle Ax,x\rangle} \langle Ax,x\rangle 
	\end{equation}
	Thus $\ell \leq \sup_{\|x\|=1}\frac{\|A x\|^2}{\langle Ax,x\rangle}$, which concludes the proof.
\end{proof}
For instance, a class of quadratic games that are \emph{not} co-coercive are the ones where $A$ is anti-symmetric ($A^\top = -A$). One the other hand, there is a large class of games that are \emph{not} strongly monotone, like for instance any quadratic game induced by a matrix with a non-zero nullspace. 

\section{Proofs of Main Convergence Analysis Results}
\label{AppendixProofs}

\subsection{Proof of Theorem~\ref{SGDA_ConstantStep}}
\label{ProofSGDA_ConstantStep}
In the main paper, we present the update rule of SGDA in \eqref{SGDA_UpdateRule}. Let us also present here the pseudo-code of SGDA:
\begin{algorithm}[H]
   \caption{Stochastic Gradient Descent Ascent (SGDA)}
   \label{SGDA_Algorithm}
\begin{algorithmic}
  \STATE {\bfseries Input:} Starting stepsize $\gamma_0>0$. Choose initial points $x^0 \in \R^d$. Distribution $\cD$ of samples.
   \FOR{$k=0,1,2,\cdots, K$}
   \STATE Sample $v^k \sim {\cal D}$
   \STATE Set step-size $\gamma_k$ following one of the selected choices (constant, decreasing)
   \STATE Set $x^{k+1}=x^k -\gamma_k \xi_{v^k} (x)$
   \ENDFOR
    \STATE {\bf Output:} The last iterate $x^k$
\end{algorithmic}
\end{algorithm}

Let us present a more general version of Theorem~\ref{SGDA_ConstantStep} that allows convergence with a larger step-size. Due to space limitations, we focus only on the important regime in the main paper, as selecting a larger step-size gives a worse convergence rate. 
\begin{theorem}[Constant Step-size]
\label{Appendix_SGDA_ConstantStep}
Assume that $\xi$ is $\mu-$quasi strongly monotone and that $\xi \in EC( \ell_{\xi})$. Choose $\alpha_k=\alpha < \frac{1}{\ell_\xi}$ for all k. Then, the iterates of SGDA, given by \eqref{SGDA_UpdateRule}, satisfy:
\begin{eqnarray}
\Exp \left[ \|x^{k}-x^*\|^2 \right]&\leq& \left[1-2\alpha \mu(1- \alpha \ell_\xi) \right]^k \|x^0-x^*\|^2  + \frac{\alpha \sigma^2}{ \mu (1-\alpha \ell_\xi)}
\end{eqnarray}
and if $\alpha_k=\alpha \in (0, \frac{1}{2\ell_\xi}]$ then the iterates of SGDA satisfy:
\begin{eqnarray}
\Exp \left[ \|x^{k}-x^*\|^2 \right]&\leq& \left(1-\alpha \mu \right)^k \|x^0-x^*\|^2  + \frac{2 \alpha \sigma^2}{ \mu}
\end{eqnarray}
\end{theorem}

\begin{proof}
\begin{eqnarray}
\|x^{k+1}-x^*\|^2 &=&\left\|x^k-\alpha \xi_{v^k}(x^k)-x^* \right\|^2\notag\\
&=&\|x^k-x^*\|^2-2 \left\langle x^k-x^*, \alpha \xi_{v^k}(x^k) \right\rangle  + \| \alpha \xi_{v^k}(x^k)\|^2\notag\\
&=&\|x^k-x^*\|^2-2 \alpha \left\langle x^k-x^*, \xi_{v^k}(x^k)\right\rangle  + \alpha^2 \| \xi_{v^k}(x^k)\|^2
\end{eqnarray}
By taking expectation condition on $x^k$:
\begin{eqnarray}
\label{ncaslaoaxnalskn2}
\Exp_{\cD} \left[ \|x^{k+1}-x^*\|^2 \right]&=& \|x^k-x^*\|^2-2 \alpha \left\langle x^k-x^*, \xi(x^k)\right\rangle + \alpha^2  \Exp_{\cD}  \left[ \left\|\xi_{v^k}(x^k) \right\|^2 \right] \notag\\
&\overset{Lemma~\ref{MainLemma}}{\leq}& \|x^k-x^*\|^2-2 \alpha \left\langle x^k-x^*, \xi(x^k)\right\rangle + 2 \alpha^2  \ell_{\xi}\left\langle x^k-x^*, \xi(x^k)\right\rangle  + 2  \alpha^2 \sigma^2 \notag\\
&\overset{\eqref{QSM}, \alpha < \frac{1}{\ell_\xi}}{\leq}& \|x^k-x^*\|^2-2\alpha \mu(1- \alpha \ell_\xi) \|x^k-x^*\|^2 + 2  \alpha^2 \sigma^2 
\end{eqnarray}
Recursively applying the above and summing up the resulting geometric series gives:
\begin{eqnarray}
\Exp \left[ \|x^{k}-x^*\|^2 \right]&\leq& [1-2\alpha \mu(1- \alpha \ell_\xi) ]^k \|x^0-x^*\|^2  + 2 \sum_{j=0}^{k-1} (1-2\alpha \mu(1- \alpha \ell_\xi))^j \alpha^2 \sigma^2\notag\\
&\leq & [1-2\alpha \mu(1- \alpha \ell_\xi) ]^k \|x^0-x^*\|^2  + \frac{  \alpha \sigma^2}{\mu(1- \alpha \ell_\xi) }
\end{eqnarray}
If we further take $\alpha \leq \frac{1}{2\ell_\xi}$ then \eqref{ncaslaoaxnalskn2} becomes:
\begin{eqnarray}
\label{ncaslaoaxnalskn3}
\Exp_{\cD} \left[ \|x^{k+1}-x^*\|^2 \right]&\leq& (1-\alpha \mu) \|x^k-x^*\|^2+ 2  \alpha^2 \sigma^2 
\end{eqnarray}
and by recursively applying the above and summing up the resulting geometric series gives:
\begin{eqnarray}
\label{cnaskjdaooaskn}
\Exp \left[ \|x^{k}-x^*\|^2 \right]
&\leq & (1-\alpha \mu)^k \|x^0-x^*\|^2  + \frac{ 2 \alpha \sigma^2}{\mu }
\end{eqnarray}
\end{proof}

\paragraph{Comment on the convergence deterministic Gradient Descent Ascent:} In the main paper, to highlight the generality of Theorem~\ref{SGDA_ConstantStep},  we present Corollary~\ref{CorollaryGDA} on the convergence of deterministic gradient descent ascent. Let us provide some more details of how one can obtain such a result through Proposition~\ref{PropositionMinibatch}.

Let us select the sampling vector $v = (1,1, \dots, 1) \in R^n$ with probability 1 in each step. Note that this is still a sampling vector as $\Exp_{\cD}[v_i]=1$. In this case, at iteration $k$, $\xi_{v^k}(x^k) \eqdef \frac{1}{n} \sum _{i=1}^n v_i \xi_i(x^k)\overset{v_i=1}{=} \frac{1}{n} \sum _{i=1}^n \xi_i(x^k)=\xi(x^k)$ and the update rule becomes equivalent to the deterministic GDA: $$x^{k+1}=x^k -\alpha_k \xi (x^k).$$ In addition, by Proposition~\ref{PropositionMinibatch} we have that if $|S|=n$ with probability one (each iteration of SGDA uses a full batch gradient), then $\ell_\xi=\ell$ and $\sigma^2=0$.
Thus, by combining \eqref{nakns} of Theorem~\ref{SGDA_ConstantStep} with Proposition~\ref{PropositionMinibatch} we obtain the convergence given in Corollary~\ref{CorollaryGDA} for the deterministic gradient descent ascent. 
We highlight that for this case, the expected co-coercivity condition \eqref{eq:ExpCoCo} is equivalent to assuming that operator $\xi$ is $\ell$-co-coercive.
\subsection{Proof of Theorem ~\ref{SGDA_DecreasingStep}}
\label{ProofSGDA_DecreasingStep}
\begin{theorem}
\label{Appendix_SGDA_DecreasingStep}
Assume $\xi$ is $\mu$-quasi-strongly monotone and that $\xi \in EC( \ell_{\xi})$. Let  $\mathcal{K} \eqdef \left.\ell_\xi\right/\mu$ and let 
\begin{equation}
\alpha_k= 
\begin{cases}
\displaystyle \frac{1}{2 \ell_\xi} & \mbox{for}\quad k \leq 4\lceil\mathcal{K} \rceil \\[0.3cm]
\displaystyle \frac{2k+1}{(k+1)^2 \mu} &  \mbox{for}\quad k > 4\lceil\mathcal{K} \rceil.
\end{cases}
\end{equation}
 If $k \geq 4 \lceil\mathcal{K} \rceil$, then iterates of SGDA, given by \eqref{SGDA_UpdateRule} satisfy:
\begin{equation}
\mathbb{E}\| x^{k} - x^*\|^2 \le   \frac{\sigma^2 }{\mu^2 }\frac{8 }{k} + \frac{16 \lceil\mathcal{K} \rceil^2}{e^2 k^2 }  \|x^0 - x^*\|^2 = O\left(\frac{1}{k}\right) \, .
\end{equation}
\end{theorem}

\begin{proof}
Let  $\alpha_k \eqdef \frac{2k+1}{(k+1)^2 \mu}$ and let $k^*$ be an integer that satisfies $\alpha_{k^*} \leq \frac{1}{2\ell_\xi}.$ 
Note that $\alpha_k$ is decreasing in $k$ and  consequently $\alpha_k \leq \frac{1}{2\ell_\xi}$ for all $k \geq k^*.$ This in turn guarantees that~\eqref{ncaslaoaxnalskn3} holds for all $k\geq k^*$ with $\alpha_k$ in place of $\alpha$, that is
\begin{eqnarray}
\Exp_{\cD} \left[ \|x^{k+1}-x^*\|^2 \right]
&\leq & (1-\alpha_k \mu) \|x^k-x^*\|^2+ 2  \alpha_k^2 \sigma^2 
\end{eqnarray}
Hence, if we take expectations and replace  $\alpha_k \eqdef \frac{2k+1}{(k+1)^2 \mu}$ then
\begin{equation}
\mathbb{E}\| x^{k+1}-x^*\|^2 \leq \frac{k^2}{(k+1)^2}\mathbb{E} \|x^k-x^*\|^2 + \frac{2\sigma^2}{\mu^2}\frac{(2k+1)^2}{(k+1)^4 }.
\end{equation}
Multiplying both sides by $(k+1)^2$ we obtain
\begin{eqnarray*}
(k+1)^2 \mathbb{E}\| x^{k+1}-x^*\|^2 &\leq & 
k^2 \mathbb{E} \|x^{k}-x^*\|^2 + \frac{2\sigma^2}{\mu^2} \left(\frac{2k+1}{k+1}\right)^2 \\
 &\leq & k^2 \mathbb{E} \|x^{k}-x^*\|^2 + \frac{8 \sigma^2}{\mu^2},
\end{eqnarray*}
where the second inequality holds because  $\frac{2k+1}{k+1} <2$. Rearranging and summing from $t= k^* \ldots k$ we obtain:
\begin{equation}
\sum_{t=k^*}^{k} \left[ (t+1)^2 \mathbb{E}\| x^{k+1}-x^*\|^2 - t^2 \mathbb{E} \|x^{k}-x^*\|^2 \right] \leq  \sum_{t=k^*}^{k} \frac{8 \sigma^2}{\mu^2}. 
\end{equation}
Using telescopic cancellation gives
\[
(k+1)^2 \mathbb{E}\| x^{k+1}-x^*\|^2 \leq  (k^*)^2 \mathbb{E} \|x^{k^*}-x^*\|^2 +\frac{8 \sigma^2 (k-k^*)}{\mu^2}.
\]
Dividing the above by $(k+1)^2$ gives
\begin{equation}
 \mathbb{E}\| x^{k+1}-x^*\|^2 \leq  \frac{(k^*)^2}{(k+1)^2 } \mathbb{E} \|x^{k^*}-x^*\|^2 +\frac{8 \sigma^2 (k-k^*)}{\mu^2(k+1)^2 }. \label{eq:cndsiu48js}
\end{equation}
For $k \leq k^*$ we have that~\eqref{cnaskjdaooaskn} holds with $\alpha_k=\frac{1}{2 \ell_\xi}$, which combined with~\eqref{eq:cndsiu48js}, gives 
\begin{eqnarray}
 \mathbb{E}\| x^{k+1}-x^*\|^2 &\leq &
  \frac{(k^*)^2}{(k+1)^2 } \left( 1 -  \frac{\mu}{2\ell_{\xi}} \right)^{k^*} \|x^{0}-x^*\|^2 \nonumber \\ &   
  +&\frac{\sigma^2 }{\mu^2 (k+1)^2}\left(8 (k-k^*) +   \frac{(k^*)^2}{\mathcal{K} } \right).  \label{eq:sdaiuna3}
\end{eqnarray}
Choosing $k^*$ that minimizes the second line of the above gives $k^* = 4\lceil\mathcal{K} \rceil$, which when inserted into~\eqref{eq:sdaiuna3} becomes
\begin{eqnarray}
 \mathbb{E}\| x^{k+1}-x^*\|^2 &\leq &
  \frac{16 \lceil\mathcal{K} \rceil^2}{(k+1)^2 } \left( 1 -  \frac{1}{2\mathcal{K}} \right)^{ 4\lceil\mathcal{K} \rceil} \|x^{0}-x^*\|^2  \nonumber \\
 & & +\frac{\sigma^2 }{\mu^2 }\frac{8 (k-2\lceil\mathcal{K} \rceil)}{(k+1)^2} \nonumber \\
  & \leq &  \frac{16 \lceil\mathcal{K} \rceil^2}{e^2(k+1)^2 }  \|x^{0}-x^*\|^2  +  \frac{\sigma^2 }{\mu^2 }\frac{8 }{k+1}, \label{eq:sdaiuna32}
\end{eqnarray}
where we have used that $\left( 1 -  \frac{1}{2x} \right)^{ 4x} \leq e^{-2}$  for all $x \geq 1.$
\end{proof}

\subsection{Proof of Theorem~\ref{SCO_ConstantStep}}
\label{ProofSCO_ConstantStep}
Before providing the proof of Theorem~\ref{SCO_ConstantStep}, let us present the definitions of quasi-strong convexity and expected smoothness condition, together with a lemma that provides a bound to the expected norm of the
stochastic gradients when a function satisfies the expected smoothness. Recall that in the main paper, we assume that 
the Hamiltonian function $\cH(x)$ is quasi-strongly convex and $\cL$---expected smooth (satisfies the expected smoothness condition). Thus, these assumptions are vital for the convergence guarantees of SCO presented in Section~\ref{sec:SCO}.

\paragraph{Technical Background on Optimization.}
Let us consider the optimization problem
\begin{equation}
\label{eq:probOPT}
 x^* = \text{argmin}_{x\in\R^d} \left[ f(x) = \tfrac{1}{n} \sum_{i=1}^n f_i(x) \right], 
\end{equation}
where each $f_i: \R^d \to \R$ is smooth and $f$ has a unique global minimizer $x^*$.

\begin{definition}[Quasi-strong convexity]
\label{DefQSConvex}
We say that a function $f: \R^d \to \R$ is $\mu$--strongly quasi-convex~\citep{karimi2016linear, Necoara-Nesterov-Glineur-2018-linear-without-strong-convexity} if there is $\mu>0$ such that: 
\begin{equation}\label{eq:strconvexcons}
 f(x^*) \geq f(x)+ \dotprod{\nabla f(x) , x^*-x} + \tfrac{\mu}{2} \norm{x^*-x}^2
\end{equation}
for all $x \in \R^d$. Here $x^*$ is the global minimizer of $f$\footnote{In our setting we assume that $x^*$ is unique, but in the more general setting, $x^*$ is the projection of point $x$ onto the solution set $X^*$ minimizing $f$.}.
\end{definition}

Note that we have already presented the expected smoothness condition in Table~\ref{TableAssumptions}. Below we present its formal definition. For this definition we use the stochastic reformulation of the finite-sum problem $f(x) = \tfrac{1}{n} \sum_{i=1}^n f_i(x)$. That is, we define: $ f_v(x) \eqdef \frac{1}{n}\sum_{i=1}^n v_i f_i(x)$ where $v\sim\cD$ is a random sampling vector (see Section~\ref{Sec:sampling} for more details on sampling vectors).

\begin{definition}[Expected Smoothness]
\label{ass:Expsmooth} We say that $f$ is $\cL$---expected smooth with respect to a distribution $\cD$ if there exists  $\cL=\cL(f,\cD)>0$  such that
\begin{equation}
\label{eq:expsmooth}
\EE{\cD}{\norm{\nabla f_v(x)-\nabla f_v(x^*)}^2} \leq 2\cL (f(x)-f(x^*)),
\end{equation}
for all $x\in\R^d$. 
\end{definition}

In the next lemma, by assuming that a function $f(x) = \tfrac{1}{n} \sum_{i=1}^n f_i(x)$ satisfies the expected smoothness we are able to bound the expected norm of its stochastic gradients. This is precisely the result we use in our proofs on the convergence of SCO, to upper bound $\Exp_{\cD}   \left[ \|  \nabla \cH_{v^k,u^k}(x^k)\|^2 \right]$. This bound allows us to avoid the much stronger bounded gradient or bounded variance assumptions.

\begin{lemma}[Lemma 2.4 in \cite{gower2019sgd}]
\label{lem:weakgrowth}
If $f$ is $\cL$---expected smooth, then
\begin{align}
\label{upperbound}
\Exp_{\cD} \left[ \|\nabla f_{v} (x)\|^2 \right] & \leq  4  \cL ( f(x)-f(x^*) ) + 2 \sigma^2.
\end{align}
where  $\sigma^2  \eqdef \Exp_{\cD}[\norm{\nabla f_v(x^*)}^2]$.
\end{lemma}
\begin{proof}
The proof can be easily obtained by following the same steps of the proof of Lemma~\ref{MainLemma}. See also the proof of Lemma 2.4 in \cite{gower2019sgd}.
\end{proof}

Let us now present a more general version of Theorem~\ref{SCO_ConstantStep} that allows convergence with a larger step-size $\alpha$. Due to space limitations, in the main paper, we focus only on the important regime of step-size $\alpha$ as selecting a larger step-size gives a worse convergence rate.
\begin{theorem}[Constant Step-size]
\label{Appendix_SCO_ConstantStep}
Assume $\xi$ is $\mu$-quasi-strongly monotone with $\mu \geq 0$ and that $\xi \in EC( \ell_{\xi})$.
Let us also assume that the Hamiltonian function $\cH$ is $\mu_{\cH}$-quasi strongly convex and $\cL_{\cH}$-expected smooth. Then, for $\gamma_k=\gamma \leq 1/4\cL_{\cH}$ and $\alpha_k=\alpha < 1/2  \ell_{\xi}$  it holds that :
\begin{eqnarray}
\Exp\left[ \|x^{k}-x^*\|^2 \right]\leq(1-\gamma\mu_{\cH}  -2 \alpha\mu + 4\alpha^2 \ell_\xi \mu )^k\|x^0-x^*\|^2 + \frac{4[\alpha^2 \sigma^2 + \gamma^2  \sigma_{\cH}^2]}{\gamma\mu_{\cH} +2 \alpha\mu - 4\alpha^2 \ell_\xi \mu},
\end{eqnarray}
and if $\alpha_k=\alpha \in (0, \frac{1}{4\ell_\xi}]$ then the iterates of SCO satisfy:
\begin{eqnarray}
\Exp\left[ \|x^{k}-x^*\|^2 \right]\leq(1-\gamma\mu_{\cH}  - \alpha\mu )^k\|x^0-x^*\|^2 + \frac{4[\alpha^2 \sigma^2 + \gamma^2  \sigma_{\cH}^2]}{\gamma\mu_{\cH}  +\alpha\mu}.
\end{eqnarray}
If $\mu=0$, that is $\xi$ only satisfies the variational stability condition $\langle\xi(x),x-x^*\rangle \geq0$, then
\begin{eqnarray}
\Exp \left[ \|x^{k}-x^*\|^2 \right]&\leq& (1-\gamma \mu_{\cH})^k \|x^0-x^*\|^2  + \frac{4 [\alpha^2 \sigma^2 +  \gamma^2  \sigma_{\cH}^2]}{\gamma \mu_{\cH} }.
\end{eqnarray}
\end{theorem}
\begin{proof}
\begin{eqnarray}
\label{nannjdalsnaalsna}
\|x^{k+1}-x^*\|^2 &=&\left\|x^k-\alpha \xi_{v^k}(x^k)-\gamma \nabla \cH_{v^k,u^k}(x^k)-x^* \right\|^2\notag\\
&=&\|x^k-x^*\|^2-2 \left\langle x^k-x^*, \alpha \xi_{v^k}(x^k) +\gamma \nabla \cH_{v^k,u^k}(x^k) \right\rangle \notag\\ && + \| \alpha \xi_{v^k}(x^k) +\gamma \nabla \cH_{v^k,u^k}(x^k)\|^2\notag\\
&=&\|x^k-x^*\|^2-2 \alpha \left\langle x^k-x^*, \xi_{v^k}(x^k)\right\rangle -2 \gamma \left\langle x^k-x^*,\nabla \cH_{v^k,u^k}(x^k) \right\rangle \notag\\ && + \| \alpha \xi_{v^k}(x^k)+\gamma \nabla \cH_{v^k,u^k}(x^k)\|^2\notag\\
&\overset{\text{Young's}}{\leq}&\|x^k-x^*\|^2-2 \alpha \left\langle x^k-x^*, \xi_{v^k}(x^k)\right\rangle -2 \gamma \left\langle x^k-x^*,\nabla \cH_{v^k,u^k}(x^k) \right\rangle \notag\\ && + 2\alpha^2 \left\|\xi_{v^k}(x^k) \right\|^2 +2 \gamma^2 \| \nabla \cH_{v^k,u^k}(x^k)\|^2
\end{eqnarray}
By taking expectation condition on $x^k$:
\begin{eqnarray}
\label{ncaslaoaxnalskn}
\Exp_{\cD} \left[ \|x^{k+1}-x^*\|^2 \right]&\leq& \|x^k-x^*\|^2-2 \alpha \left\langle x^k-x^*, \xi(x^k)\right\rangle -2 \gamma \left\langle x^k-x^*,\nabla \cH(x^k) \right\rangle \notag\\ && + 2\alpha^2  \Exp_{\cD}  \left[ \left\|\xi_{v^k}(x^k) \right\|^2 \right] +2 \gamma^2 \Exp_{\cD}   \left[ \|  \nabla \cH_{v^k,u^k}(x^k)\|^2 \right]\notag\\
&\overset{\text{\eqref{eq:strconvexcons}}}{\leq}& \|x^k-x^*\|^2-2 \alpha \left\langle x^k-x^*, \xi(x^k)\right\rangle -2 \gamma[\cH(x^k)-\cH(x^*)]- \gamma \mu_{\cH} \|x^k-x^*\|^2 \notag\\ && + 2\alpha^2  \Exp_{\cD}  \left[ \left\|\xi_{v^k}(x^k) \right\|^2 \right] +2 \gamma^2 \Exp_{\cD}   \left[ \|  \nabla \cH_{v^k,u^k}(x^k)\|^2 \right]\notag\\
&\overset{\eqref{upperbound}}{\leq}& \|x^k-x^*\|^2-2 \alpha \left\langle x^k-x^*, \xi(x^k)\right\rangle -2 \gamma [\cH(x^k)-\cH(x^*)]- \gamma \mu_{\cH} \|x^k-x^*\|^2 \notag\\ && + 2\alpha^2  \Exp_{\cD}  \left[ \left\|\xi_{v^k}(x^k) \right\|^2 \right]  + 8 \gamma^2 \cL_{\cH} ( \cH(x)-\cH(x^*) ) + 4 \gamma^2  \sigma_{\cH}^2 \notag\\
&\overset{\eqref{oansxa}}{\leq}& \|x^k-x^*\|^2-2 \alpha \left\langle x^k-x^*, \xi(x^k)\right\rangle -2 \gamma[\cH(x^k)-\cH(x^*)]- \gamma\mu_{\cH} \|x^k-x^*\|^2 \notag\\ && + 4\alpha^2 \ell_{\xi}\langle\xi(x),x-x^*\rangle + 4\alpha^2 \sigma^2 + 8 \gamma^2 \cL_{\cH} ( \cH(x)-\cH(x^*) ) + 4 \gamma^2  \sigma_{\cH}^2 
\end{eqnarray}

Recall that $\xi$ is $\mu$-quasi strongly monotone, $\left\langle x^k-x^*, \xi(x^k)\right\rangle \geq \mu \|x^k-x^*\|^2$.
Thus, for $\alpha_k < \frac{1}{2 \ell_{\xi}}$, it holds that:
$$(-2 \alpha_k + 4\alpha_k^2  \ell_{\xi} ) \langle\xi(x^k),x-x^*\rangle \leq (-2 \alpha_k + 4\alpha_k^2  \ell_{\xi}) \mu \|x^k-x^*\|^2,$$
and the inequality \eqref{ncaslaoaxnalskn} takes the following form:
\begin{eqnarray}
\label{canjadnakdl}
\Exp_{\cD} \left[ \|x^{k+1}-x^*\|^2 \right] &\leq& (1-\gamma\mu_{\cH} )\|x^k-x^*\|^2 + (-2 \alpha + 4\alpha^2  \ell_{\xi}) \mu \|x^k-x^*\|^2 \notag\\ && + (- 2 \gamma + 8 \gamma_k^2 \cL_{\cH}) [\cH(x^k)-\cH(x^*)] + 4\alpha^2 \sigma^2 + 4 \gamma^2  \sigma_{\cH}^2 \notag\\
&=& (1-\gamma\mu_{\cH} -2 \alpha \mu + 4\alpha^2  \ell_{\xi}\mu )\|x^k-x^*\|^2  \notag\\ && + (- 2 \gamma + 8 \gamma^2 \cL_{\cH}) [\cH(x^k)-\cH(x^*)] + 4\alpha^2 \sigma^2 + 4 \gamma^2  \sigma_{\cH}^2 \notag\\
&\overset{\gamma_k < \frac{1}{4\cL_{\cH}}}{\leq}&(1-\gamma\mu_{\cH}  -2 \alpha\mu + 4\alpha^2  \ell_{\xi}\mu )\|x^k-x^*\|^2 + 4[\alpha^2 \sigma^2 + \gamma^2  \sigma_{\cH}^2] \notag\\
\end{eqnarray}
By taking expectations again and by recursively applying the above and summing up
the resulting geometric series gives:
\begin{eqnarray*}
\Exp\left[ \|x^{k}-x^*\|^2 \right]&\leq& (1-\gamma\mu_{\cH}  -2 \alpha\mu + 4\alpha^2  \ell_{\xi}\mu )^k\|x^0-x^*\|^2 + \frac{4[\alpha^2 \sigma^2 + \gamma^2  \sigma_{\cH}^2]}{\gamma\mu_{\cH} +2 \alpha\mu - 4\alpha^2 \ell_{\xi}\mu}\notag\\
\end{eqnarray*}

If we further assume that $\alpha \leq \frac{1}{4\ell_\xi}$ then $1-\gamma\mu_{\cH}  -2 \alpha\mu_g + 4\alpha^2  \ell_{\xi}\mu_g \leq 1-\gamma\mu_{\cH}  - \alpha\mu $ and the iterates of SCO satisfy:
\begin{eqnarray}
\Exp\left[ \|x^{k}-x^*\|^2 \right]\leq(1-\gamma\mu_{\cH}  - \alpha\mu )^k\|x^0-x^*\|^2 + \frac{4[\alpha^2 \sigma^2 + \gamma^2  \sigma_{\cH}^2]}{\gamma\mu_{\cH}  +\alpha\mu}.
\end{eqnarray}
In addition, if $\mu=0$, that is $\xi$ only satisfies the variational stability condition $\langle\xi(x),x-x^*\rangle \geq0$, then
\begin{eqnarray}
\Exp \left[ \|x^{k}-x^*\|^2 \right]&\leq& (1-\gamma \mu_{\cH} )^k \|x^0-x^*\|^2  + \frac{4 [\alpha^2 \sigma^2 +  \gamma^2  \sigma_{\cH}^2]}{\gamma \mu_{\cH} }.
\end{eqnarray}
This completes the proof.
\end{proof}

\paragraph{On Deterministic Consensus Optimization.} In Corollary~\ref{DeterministicCO} we show the convergence of Deterministic CO as special case of our main Theorem. Here we provide few more details to understand exactly this convergence. 

Let us select the sampling vectors $v = u= (1,1, \dots, 1) \in R^n$ with probability 1 in each step. Note that these are still sampling vectors as $\Exp_{\cD}[v_i]=\Exp_{\cD}[u_i]=1$. In this case, at iteration $k$, $\xi_{v^k}(x^k) \eqdef \frac{1}{n} \sum _{i=1}^n v_i \xi_i(x^k)\overset{v_i=1}{=} \frac{1}{n} \sum _{i=1}^n \xi_i(x^k)=\xi(x^k)$ and $\nabla \cH_{u,v}(x)=\frac{1}{2} \left[ \bJ_u^\top(x) \xi_v(x) +   \bJ_v^\top(x) \xi_u(x) \right]\overset{v_i=1,u_i=1}{=}\frac{1}{2} \left[ \bJ^\top(x) \xi(x) +   \bJ^\top(x) \xi(x) \right]=\bJ^\top(x) \xi(x)=\nabla \cH(x).$ Thus, the update rule becomes equivalent to the deterministic CO: $$x^{k+1}=x^k - \alpha \xi(x^k) - \gamma \nabla \cH(x^k).$$ In addition, by Proposition~\ref{PropositionMinibatch} we have that if $|S|=n$ with probability one, then $\ell_\xi=\ell$ and $\sigma^2=0$. Using also the properties of expected smoothness (see Proposition 3.8 in \cite{gower2019sgd}) we have that $\cL_{\cH}=L_{\cH}$, where $L_{\cH}$ is the smoothness parameter of the Hamiltonian function, and $\sigma_\cH^2=0$. By simply substituting these values to the main theorem we are able to obtain the convergence result presented in Corollary~\ref{DeterministicCO}.

\paragraph{Convergence of SGDA and SHGD as special cases of our Analysis.} As we mentioned in the main paper, SCO
is a weighted combination of SGDA (Algorithm~\ref{SGDA_Algorithm}) and the stochastic Hamiltonian gradient descent (SHGD) of~\cite{loizou2020stochastic}. Thus, it is clear, that if one selects $\alpha^k=0, \forall k >0 $ then the method is equivalent to SHGD and if $ \gamma^k=0, \forall k >0 $ then the method becomes equivalent to the SGDA.
Here we highlight that in these cases the Young's inequality used in \eqref{nannjdalsnaalsna} of the proof of main theorem is not necessary. This is exactly why by specifying the update rule to SHGD we are able to have convergence using the larger bound on the step-size $\gamma \leq 1/2\cL_{\cH}$ (see Corollary~\ref{CorollarySHGD}). Following similar argument the convergence of SGDA presented in Theorem ~\ref{SGDA_ConstantStep} can also obtained as special case of Theorem~\ref{SCO_ConstantStep}.

\subsection{Proof of Theorem~\ref{SCO_DecreasingStep}}
\label{ProofSCO_DecreasingStep}

\begin{theorem}
\label{Appendix_SCO_DecreasingStep}
Assume $\xi$ is $\mu$-quasi-strongly monotone and that $\xi \in EC( \ell_{\xi})$. Assume that the Hamiltonian function $\cH$ is $\mu_{\cH}$-quasi strongly convex and $\cL_{\cH}$-expected smooth. Let $\alpha_k=\gamma_k$,  $\psi= \max\{  \ell_{\xi},\cL_{\cH}\}$ and $k^* \eqdef 8 \frac{\psi}{\mu_{\cH}+\mu}$. Let also, 
\begin{equation}
\gamma_k= 
\begin{cases}
\displaystyle \frac{1}{4 \psi} & \mbox{for}\quad k \leq \lceil k^* \rceil  \\[0.3cm]
\displaystyle \frac{2k+1}{(k+1)^2 [\mu_{\cH}+\mu]}&  \mbox{for}\quad k >  \lceil k^* \rceil.
\end{cases}
\end{equation}
If $k \geq  \lceil k^* \rceil$, then SCO iterates satisfy:
\begin{equation}
\mathbb{E}\| x^{k} - x^*\|^2 \le   \frac{\sigma_{\cH}^2 +\sigma^2}{[\mu + \mu_{\cH}]^2}\frac{16 }{k} + \frac{(k^*)^2 }{e^2 k^2}  \|x^0 - x^*\|^2 = O\left(\frac{1}{k}\right)
\end{equation}
If $\mu=0$, that is $\xi$ only satisfies the variational stability condition $\langle\xi(x),x-x^*\rangle \geq0$, then SCO is still able to converge sublinearly with $O\left(\frac{1}{k}\right)$, to $x^*$.
\end{theorem}

\begin{proof}
Let  $\gamma_k \eqdef \frac{2k+1}{(k+1)^2 [\mu_{\cH}+\mu]}$ and let $k^*$ be an integer that satisfies $$\gamma_{k^*} \leq \min\left\{\frac{1}{4\ell_{\xi}},\frac{1}{4\cL_{\cH}}\right\} = \frac{1}{4 \psi}.$$
Note that $\gamma_k$ is decreasing in $k$ and  consequently $\gamma_k \leq \frac{1}{4 \psi}$ for all $k \geq k^*.$ This in turn guarantees that~\eqref{canjadnakdl} holds for all $k\geq k^*$ with $\gamma_k$ in place of $\gamma$ and $\alpha_k=\gamma_k$ in place of $\alpha$, that is 
\begin{eqnarray*}
\Exp_{\cD} \left[ \|x^{k+1}-x^*\|^2 \right] \leq (1-\gamma_k\mu_{\cH}  -2 \gamma_k\mu + 4\gamma_k^2  \ell_{\xi}\mu )\|x^k-x^*\|^2 + 4[\gamma_k^2 \sigma^2 + \gamma_k^2  \sigma_{\cH}^2]
\end{eqnarray*}
and since $\gamma_k \leq \frac{1}{4\ell_\xi}$ we obtain:
\begin{eqnarray}
\label{qqqq1}
\Exp_{\cD} \left[ \|x^{k+1}-x^*\|^2 \right] \leq (1-\gamma_k [\mu_{\cH} + \mu])\|x^k-x^*\|^2 + 4\gamma_k^2 [\sigma^2 + \sigma_{\cH}^2].
\end{eqnarray}
For simplicity of presentation let us denote $\bar{\mu}= [\mu_{\cH} + \mu]$ and $\bar{\sigma}^2=[\sigma^2 + \sigma_{\cH}^2]$ then \eqref{qqqq1} can be written as:
\begin{eqnarray}
\label{qqqq2}
\Exp_{\cD} \left[ \|x^{k+1}-x^*\|^2 \right] \leq (1-\gamma_k\bar{\mu})\|x^k-x^*\|^2 + 4\gamma_k^2 \bar{\sigma}^2.
\end{eqnarray}
Now let us follow similar steps to the proof of Theorem~\ref{SGDA_DecreasingStep}.

By taking expectations and replacing  $\gamma_k \eqdef \frac{2k+1}{(k+1)^2 [\mu_{\cH}+\mu]}=\frac{2k+1}{(k+1)^2 \bar{\mu}}$ we obtain,
\begin{equation}
\mathbb{E}\| x^{k+1}-x^*\|^2 \leq \frac{k^2}{(k+1)^2}\mathbb{E} \|x^k-x^*\|^2 + \frac{4 \bar{\sigma}^2} {\bar{\mu}^2}\frac{(2k+1)^2}{(k+1)^4 }.
\end{equation}
Multiplying both sides by $(k+1)^2$ we obtain
\begin{eqnarray*}
(k+1)^2 \mathbb{E}\| x^{k+1}-x^*\|^2 &\leq & 
k^2 \mathbb{E} \|x^{k}-x^*\|^2 + \frac{4 \bar{\sigma}^2} {\bar{\mu}^2} \left(\frac{2k+1}{k+1}\right)^2 \\
 &\leq & k^2 \mathbb{E} \|x^{k}-x^*\|^2 + \frac{16 \bar{\sigma}^2} {\bar{\mu}^2},
\end{eqnarray*}
where the second inequality holds because  $\frac{2k+1}{k+1} <2$. Rearranging and summing from $t= k^* \ldots k$ we obtain:
\begin{equation}
\sum_{t=k^*}^{k} \left[ (t+1)^2 \mathbb{E}\| x^{k+1}-x^*\|^2 - t^2 \mathbb{E} \|x^{k}-x^*\|^2 \right] \leq  \sum_{t=k^*}^{k} \frac{16 \bar{\sigma}^2} {\bar{\mu}^2}. 
\end{equation}
Using telescopic cancellation gives
\[
(k+1)^2 \mathbb{E}\| x^{k+1}-x^*\|^2 \leq  (k^*)^2 \mathbb{E} \|x^{k^*}-x^*\|^2 +\frac{8 \bar{\sigma}^2 (k-k^*)}{\bar{\mu}^2}.
\]
Dividing the above by $(k+1)^2$ gives
\begin{equation}
\label{cnaosa2}
 \mathbb{E}\| x^{k+1}-x^*\|^2 \leq  \frac{(k^*)^2}{(k+1)^2 } \mathbb{E} \|x^{k^*}-x^*\|^2 +\frac{8 \bar{\sigma}^2 (k-k^*)}{\bar{\mu}^2(k+1)^2 }. 
\end{equation}

At this point note that for $k \leq k^*$ we have that~\eqref{qqqq2} holds and by using our step-size selection $\gamma_k=\gamma = \frac{1}{4 \psi}$ for $ k \leq \lceil k^* \rceil $ we obtain
\begin{eqnarray}
\label{qqqq3}
\Exp_{\cD} \left[ \|x^{k+1}-x^*\|^2 \right] \leq (1-\gamma\bar{\mu})\|x^k-x^*\|^2 + 4\gamma^2 \bar{\sigma}^2,
\end{eqnarray}
which by taking expectations again and by recursively applying the above and summing up
the resulting geometric series gives (for $k \leq k^*$):
\begin{eqnarray}
\label{qqqq4}
\Exp \left[ \|x^{k+1}-x^*\|^2 \right] \leq (1-\gamma\bar{\mu})^k \|x^0-x^*\|^2 + \frac{4\gamma \bar{\sigma}^2}{\bar{\mu}},
\end{eqnarray}

Thus, for $k \leq k^*$ we have that~\eqref{qqqq4} holds with $\gamma_k=\frac{1}{4 \psi}$, which combined with~\eqref{cnaosa2}, gives 
\begin{eqnarray}
 \mathbb{E}\| x^{k+1}-x^*\|^2 &\leq &
  \frac{(k^*)^2}{(k+1)^2 } \left( 1 -  \frac{\bar{\mu}}{4 \psi} \right)^{k^*} \|x^{0}-x^*\|^2 \nonumber \\ &   
  +&\frac{\bar{\sigma}^2 }{\bar{\mu}^2 (k+1)^2}\left(16 (k-k^*) +   \frac{(k^*)^2 \bar{\mu}}{\psi } \right).  \label{eq:sdaiuna33}
\end{eqnarray}
Choosing $k^*$ that minimizes the second line of the above gives $k^* = 8 \frac{\psi}{\bar{\mu}} $, which when inserted into~\eqref{eq:sdaiuna33} becomes
\begin{eqnarray}
 \mathbb{E}\| x^{k+1}-x^*\|^2 &\leq &
  \frac{(k^*)^2}{(k+1)^2 } \left( 1 -  \frac{2}{k^*} \right)^{k^*} \|x^{0}-x^*\|^2 \nonumber \\ &   
  +&\frac{\bar{\sigma}^2 }{\bar{\mu}^2 (k+1)^2}8 \left(2k- k^* \right)\nonumber \\
  & \leq &  \frac{(k^*)^2}{(k+1)^2 e^2} \|x^{0}-x^*\|^2  +  \frac{\bar{\sigma}^2 }{\bar{\mu}^2 (k+1)^2}8 \left(2k- k^* \right)\nonumber \\
  & \leq &  \frac{(k^*)^2}{(k+1)^2 e^2} \|x^{0}-x^*\|^2  +  \frac{\bar{\sigma}^2 }{\bar{\mu}^2} \frac{16}{k+1}.
\end{eqnarray}
where in the second inequality we have used that $\left( 1 -  \frac{1}{2x} \right)^{4x} \leq \frac{1}{e^2}$  for all $x \geq 1$ and in the last inequality we used that
$\frac{2 k-k^*}{k+1} \leq \frac{2 k}{k+1}  \leq 2. $

Thus by replacing $\bar{\mu}= [\mu_{\cH} + \mu]$ and $\bar{\sigma}^2=[\sigma^2 + \sigma_{\cH}^2]$ we obtain:
$$\mathbb{E}\| x^{k} - x^*\|^2 \le   \frac{\sigma_{\cH}^2 +\sigma^2}{[\mu + \mu_{\cH}]^2}\frac{16 }{k} + \frac{(k^*)^2 }{e^2 k^2}  \|x^0 - x^*\|^2 = O\left(\frac{1}{k}\right).$$

As we mentioned in the statement of the Theorem, if $\mu=0$, that is $\xi$ only satisfies the variational stability condition $\langle\xi(x),x-x^*\rangle \geq0$, then SCO is still able to converge sublinearly with $O\left(\frac{1}{k}\right)$, to $x^*$. In this case, the proof will be exactly the same as above but $k^*\eqdef 8 \frac{\psi}{\mu_{\cH}}$ and $\gamma_k=\frac{2k+1}{(k+1)^2 \mu_{\cH}} \text{for} \quad k >  \lceil k^* \rceil.$ And the convergence will be 
$
\mathbb{E}\| x^{k} - x^*\|^2 \le   \frac{\sigma_{\cH}^2 +\sigma^2}{\mu_{\cH}^2}\frac{16 }{k} + \frac{(k^*)^2 }{e^2 k^2}  \|x^0 - x^*\|^2 = O\left(\frac{1}{k}\right).
$
\end{proof}
\newpage
\section{On Experiments}
\label{Appendix_Experiments}

\subsection{Properties of Hamiltonian Function for Quadratic Games}
\label{app:quadratic_game_proof}
In the next proposition, we explain how the assumptions on the Hamiltonian function used in the main theorems of Section~\ref{sec:SCO} (convergence analysis of SCO) are satisfied for the quadratic min-max problems.
\begin{proposition}
  For quadratic games of the form \eqref{eq:quadratic_games} with $\bA_i$ and $\bC_i$ symmetric with at least one solution $x^*$, the Hamiltonian function $\cH(x)$ is a $L_{\cH}$-smooth and $\mu_{\cH}$–quasi-strongly convex quadratic function with constants $L_{\cH} = \sigma^2_{\max}(\bJ)$ and $\mu_{\cH} = \sigma^2_{\min}(\bJ)$ where
  $\sigma_{\max}$ and $\sigma_{\min}$ are the maximum and minimum non-zero singular values of $\bJ$, and $\bJ=\nabla \xi$ is the Jacobian matrix of the game.
\end{proposition} 

\begin{proof}
  Our approach follows closely the proof of Proposition 4.3 of \cite{loizou2020stochastic} where the properties of the Hamiltonian function for stochastic bilinear games were presented. Let $\bA=\frac{1}{n}\sum_{i=1}^n \bA_i$, $\bB=\frac{1}{n}\sum_{i=1}^n \bB_i$ and $\bC=\frac{1}{n}\sum_{i=1}^n \bC_i$ and let $a=\frac{1}{n}\sum_{i=1}^n a_i$ and $c=\frac{1}{n}\sum_{i=1}^n c_i$.
  
Firstly, note that the stochastic Hamiltonian function of \eqref{eq:quadratic_games} has the following form:
      $$\mathcal{H}(x) = \frac{1}{2} x^\top \bQ x + q^\top x + \ell$$
  where $\bQ = \begin{pmatrix} \bA^2 + \bB \bB^\top & \bA \bB - \bB\bC \\ \bB^\top \bA - \bC \bB^\top & \bC^2 + \bB^\top \bB \end{pmatrix} = \bJ^\top \bJ$, $q = \begin{pmatrix} \bA a - \bB c \\ \bB^\top a + \bC c \end{pmatrix}$, and $\ell = \frac{1}{2} (\|a\|^2 + \|b\|^2)$
  
  The Hamiltonian is thus a smooth quadratic function and the matrix $\bQ$ is symmetric and positive semi-definite.
  In addition, as we assumed that there was at least one solution $x^*$ for the game, i.e. $\xi(x^*)=0$, we have that $x^*$ is also a global minimum of the Hamiltonian function $\mathcal{H}(x)$, and thus we have that $q=-\bQ x^*=-\bJ^\top \bJ x^*$.
  Using this, we can rewrite the Hamiltonian as:
  $$\mathcal{H}(x)=\phi(\bJ x),$$
where the function $\phi(y):=\frac{1}{2}\|y\|^2 - (\bJ x^*)^\top y+\ell$ is 1-strongly convex with 1-Lipschitz continuous gradient.

Thus, using Lemma D.1 in \cite{loizou2020stochastic}, we have that the the Hamiltonian function is a $L_{\cH}$-smooth, $\mu_{\cH}$-quasi-strongly convex function with constants $L_{\cH}=\lambda_{\max}(\bJ^\top \bJ)=\lambda_{\max }(\bQ)=\sigma_{\max}^2(\bJ)$ and $\mu_{\cH}=\sigma_{\min }^2(\bJ)=\lambda_{\min }^+(\bQ)$.
\end{proof}

\subsection{Experimental Details}
\label{app:experimental_details}
We describe here in more details the exact settings we use for evaluating the different algorithms. 
As mentioned in Section~\ref{sec:numerical_eval}, we evaluate the different algorithms on the class of quadratic games:
\begin{equation*}
  \min_{x_1 \in \mathbb{R}^{d}} \max_{x_2 \in \mathbb{R}^{p}} \frac{1}{n}\sum_i \frac{1}{2} x_1^\top \bA_i x_1 + x_1^\top \bB_i x_2 - \frac{1}{2} x_2^\top \bC_i x_2 + a_i^\top x_1 - c_i^\top x_2
\end{equation*}

In all our experiments we choose $d = p = 100$ and $n=100$. To sample the matrices $\bA_i$ (resp. $\bC_i$) we first generate a random orthogonal matrix $\bQ_i$ (resp. $\bQ'_i$), we then sample a random diagonal matrix $\bD_i$ (resp. $\bD'_i$) where the elements on the diagonal are sampled uniformly in $[\mu_A, L_A]$ (resp. $[\mu_C, L_C]$), such that at least one of the matrices has a minimum eigenvalue equal to $\mu_A$ (resp. $\mu_C$) and one matrix has a maximum eigenvalue equal to $L_A$ (resp. $L_B$).  Finally we construct the matrices by computing $\bA_i=\bQ_i\bD_i\bQ_i^\top$ (resp. $\bC_i = \bQ'_i\bD'_i{\bQ'}_i^\top$). This ensures that the matrices $\bA_i$ and $\bC_i$ for all $i \in [n]$, are symmetric and positive definite. We sample the matrices $\bB_i$ in a similar fashion with the diagonal matrix $\bD_i$ to lie between $[\mu_B, L_B]$\footnote{We highlight that matrices $\bB_i$ are not necessarily symmetric.}. In all our experiments we choose $\mu_A=\mu_C$ and $L_A=L_C$. By varying the different constants $\mu_A, L_A, \mu_C, L_C, \mu_B, L_B$ we can get a variety of games with different properties $\mu, \ell_\xi, \mu_{\cH}, \cL_{\cH}$. The bias terms $a_i, c_i$ are sampled from a normal distribution. For further details, see also our source code\footnote{\url{https://github.com/hugobb/StochasticGamesOpt}}.

As we have already mentioned in Section~\ref{sec:numerical_eval}, we pick the step-sizes for the different methods according to our theoretical findings. That is, for constant step-size, we select $\alpha=\frac{1}{2\ell_\xi}$ for SGDA (Theorem~\ref{SGDA_ConstantStep}), $\alpha=\frac{1}{4\ell_\xi}, \gamma=\frac{1}{4\cL_{\cH}}$ for SCO (Theorem~\ref{SCO_ConstantStep}), and $\gamma=\frac{1}{2\cL_{\cH}}$ for SHGD (Corollary~\ref{CorollarySHGD}). For the stepsize-switching rule that guarantees convergence to $x^*$, we use the step-sizes proposed in Theorem~\ref{SGDA_DecreasingStep} for SGDA and Theorem~\ref{SCO_DecreasingStep} for SCO.

In the experiments, we run all methods (SGDA, SCO and SHGD) using uniform single-element sampling. That is, 
$|S|=1$, according to the Definition~\ref{def:minibatch}. Thus, $\Prob{i\in S} = p_i = \frac{1}{n}$. By Proposition~\ref{PropositionMinibatch}, this means that  $\ell_\xi =\ell_{\max}=\max \{\ell_i\}_{i=1}^n$. In addition, by Proposition 3.8 in \cite{gower2019sgd} and the structure of the stochastic Hamiltonian function, $\cH(x)= \frac{1}{2} \|\xi(x)\|^2 = \frac{1}{n} \sum_{i=1}^n \frac{1}{n} \sum_{j=1}^n \frac{1}{2} \langle  \xi_i(x),  \xi_j(x)\rangle$, we have that $\cL_{\cH}=\max\{L_{\cH_{i,j}}\}_{i=1, j=1}^n$, where $L_{\cH_{i,j}}$ is the smoothness parameter of $\cH_{i,j}= \frac{1}{2} \langle  \xi_i(x),  \xi_j(x)\rangle$. 

The values of the co-coercive parameters $\ell_i$ for all $i \in [n]$ are computed using $\frac{1}{\ell_i} = \min_{\lambda \in Sp(\bJ_i)} \Re(\frac{1}{\lambda})$ \citep{azizian2019tight}. Here $Sp(\bJ_i)$ denotes the spectrum of the Jacobian matrix $J_i$ for all $i \in [n]$ and $\Re$ denotes the real part of a complex number.

\subsection{Additional Experiment: Influence of the Step-size on Convergence}
\label{app:step_size_study}
In this section we provide further experiments exploring the performance of SGDA, SHGD and SCO for different constant step-sizes (see Fig.~\ref{fig:lr_grid}). This experiment aims to understand better how the convergence rate and the size of the neighborhood the algorithms converge to depend on the step size. This also enables us to assess if the optimal step-size suggested by the theory is tight.
When the step-size is too large we observe that the methods diverges, we do not include them in the plots but provide them below for completeness. We observe that SHGD diverges with a step-size of $\gamma=3\gamma^*$ with $\gamma^*=\frac{1}{2\cL_{\cH}}$, SGDA diverges with a step-size of $\alpha=4\alpha^*$ with $\alpha^*=\frac{1}{2\ell_\xi}$, SCO diverges with a step-size of $\alpha=4\alpha^*$ and $\gamma=4\gamma^*$ with $\alpha^*=\frac{1}{4\ell_\xi}$ and $\gamma^*=\frac{1}{4\cL_{\cH}}$.
We also observe that SCO is less sensitive to the choice of $\alpha$ than to the choice of $\gamma$

\begin{figure}[H]
  \centering
  \begin{subfigure}[b]{0.32\textwidth}
    \includegraphics[width=\textwidth]{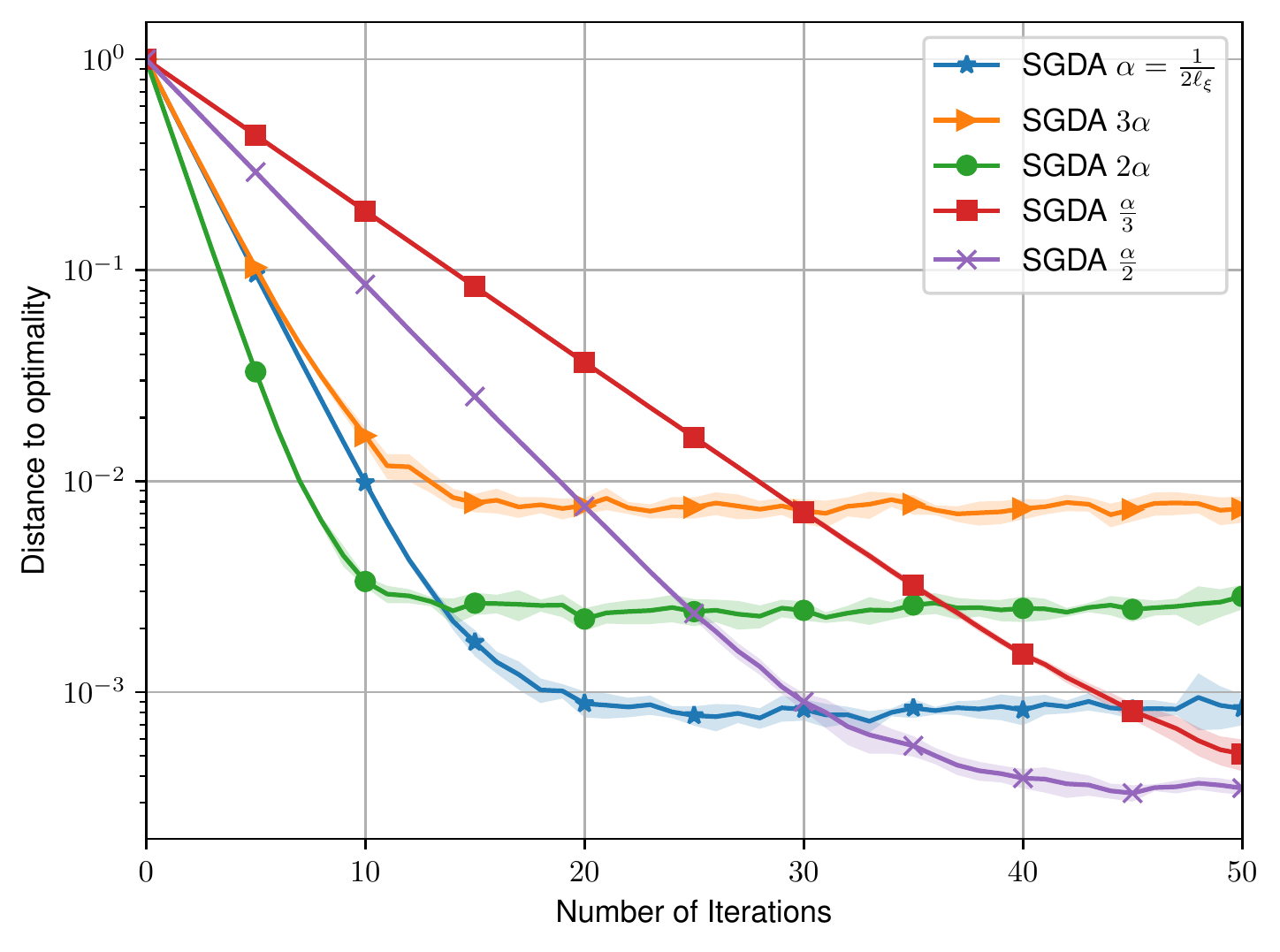}
    \caption{SGDA}
  \end{subfigure}
  \begin{subfigure}[b]{0.32\textwidth}
    \includegraphics[width=\textwidth]{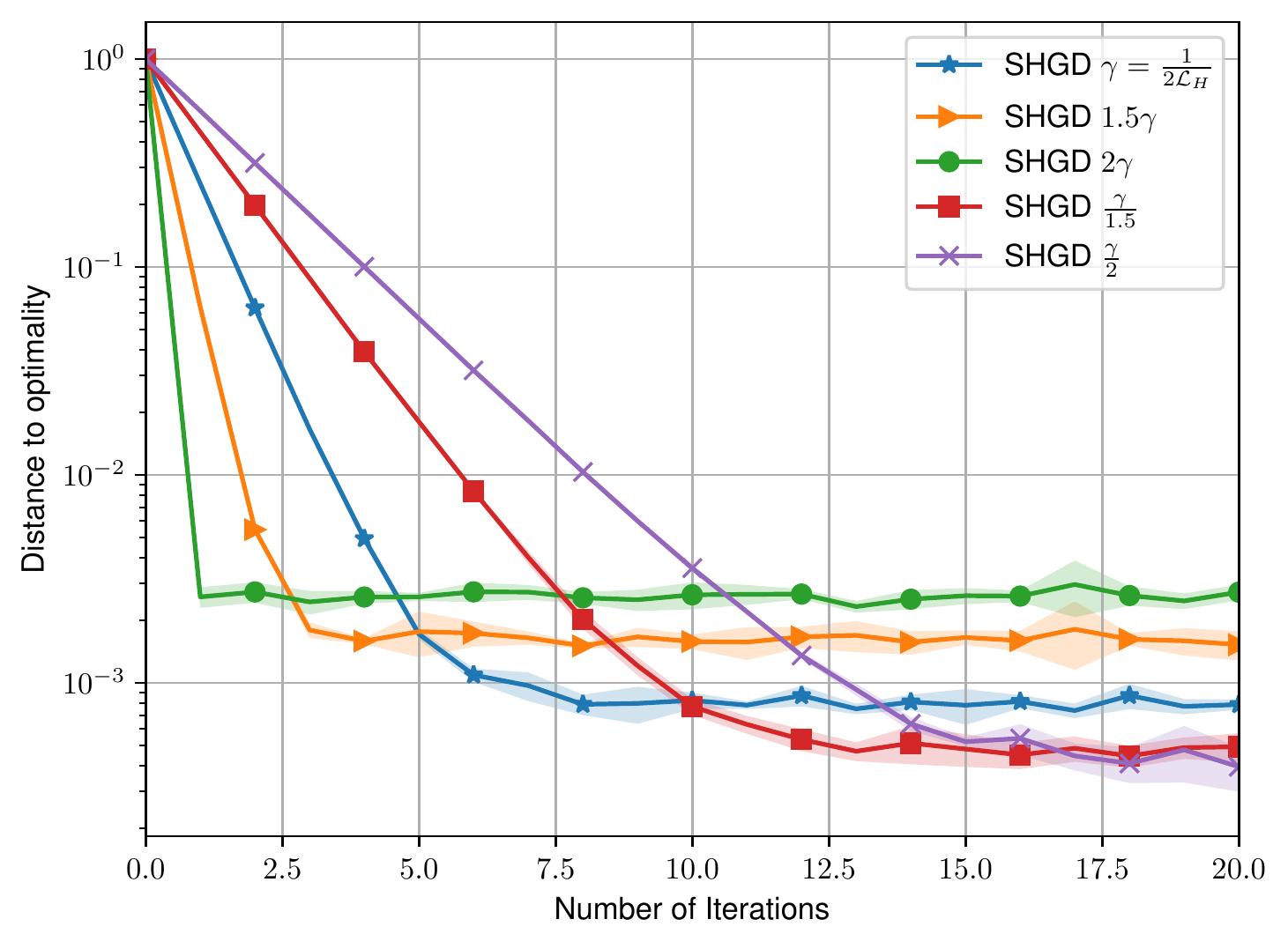}
    \caption{SHGD}
  \end{subfigure}
  \begin{subfigure}[b]{0.32\textwidth}
    \includegraphics[width=\textwidth]{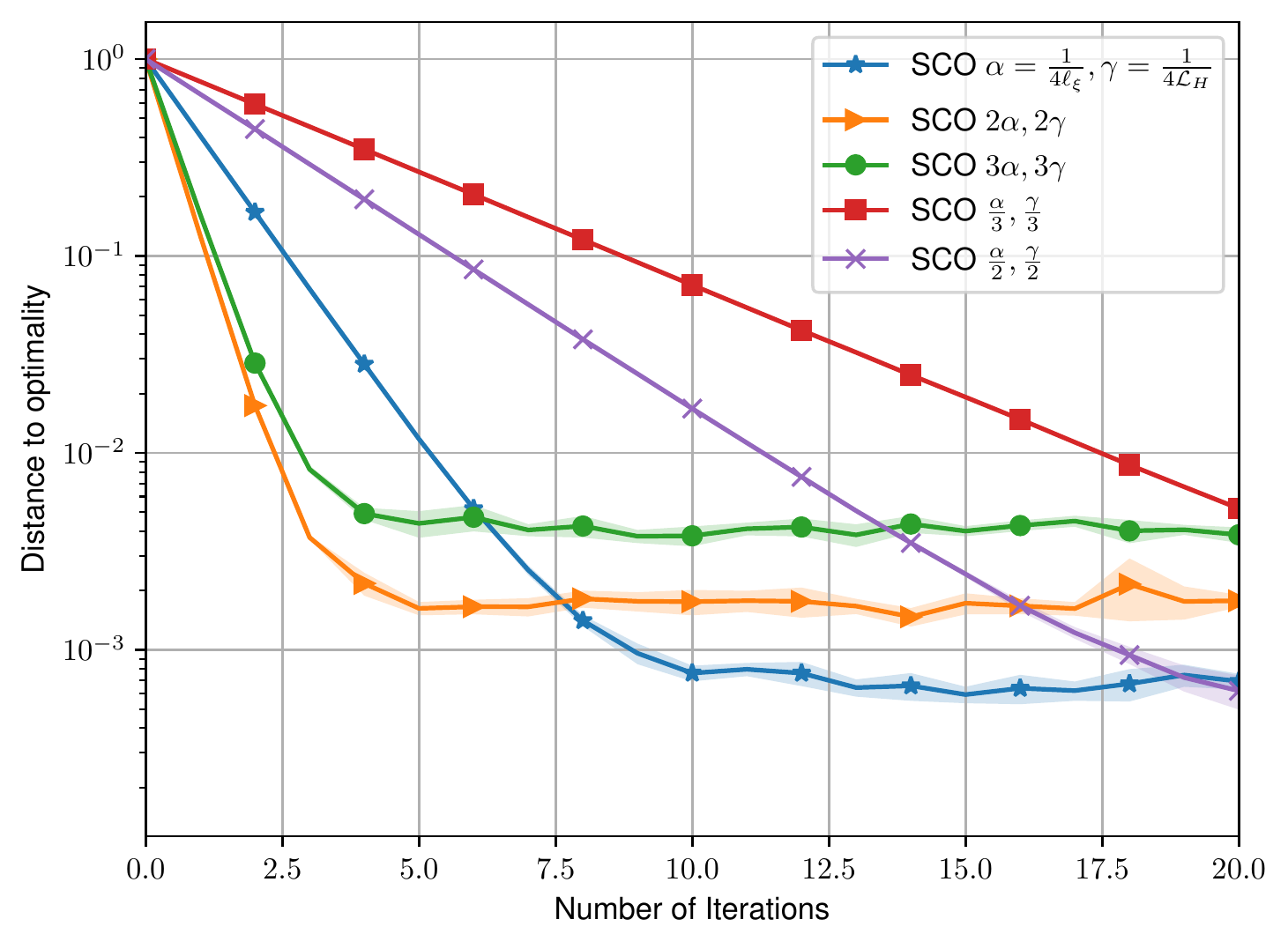}
    \caption{SCO varying both $\alpha$ \& $\gamma$}
  \end{subfigure}
  \begin{subfigure}[b]{0.32\textwidth}
    \includegraphics[width=\textwidth]{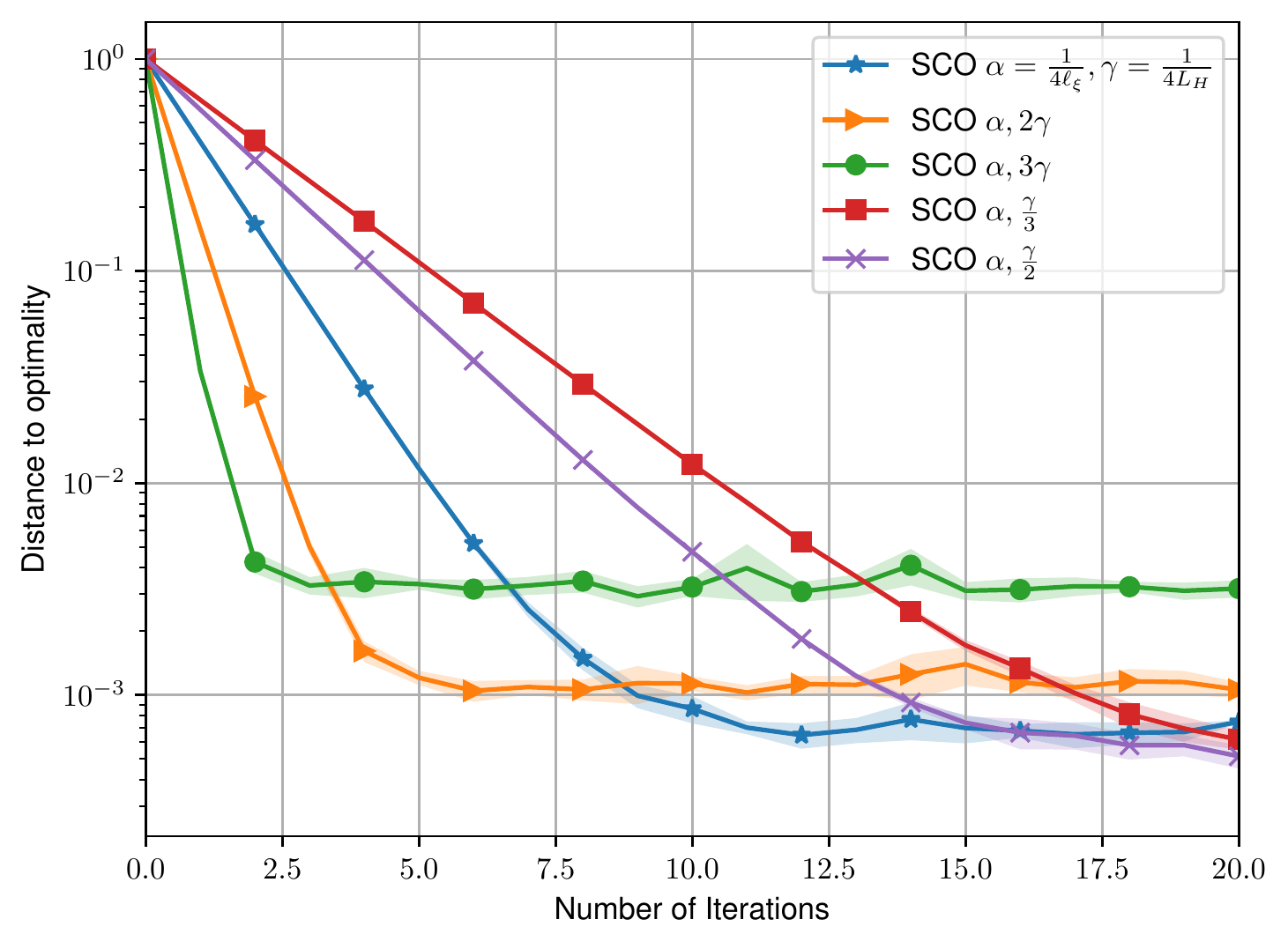}
    \caption{SCO varying only $\gamma$}
  \end{subfigure}
  \begin{subfigure}[b]{0.32\textwidth}
    \includegraphics[width=\textwidth]{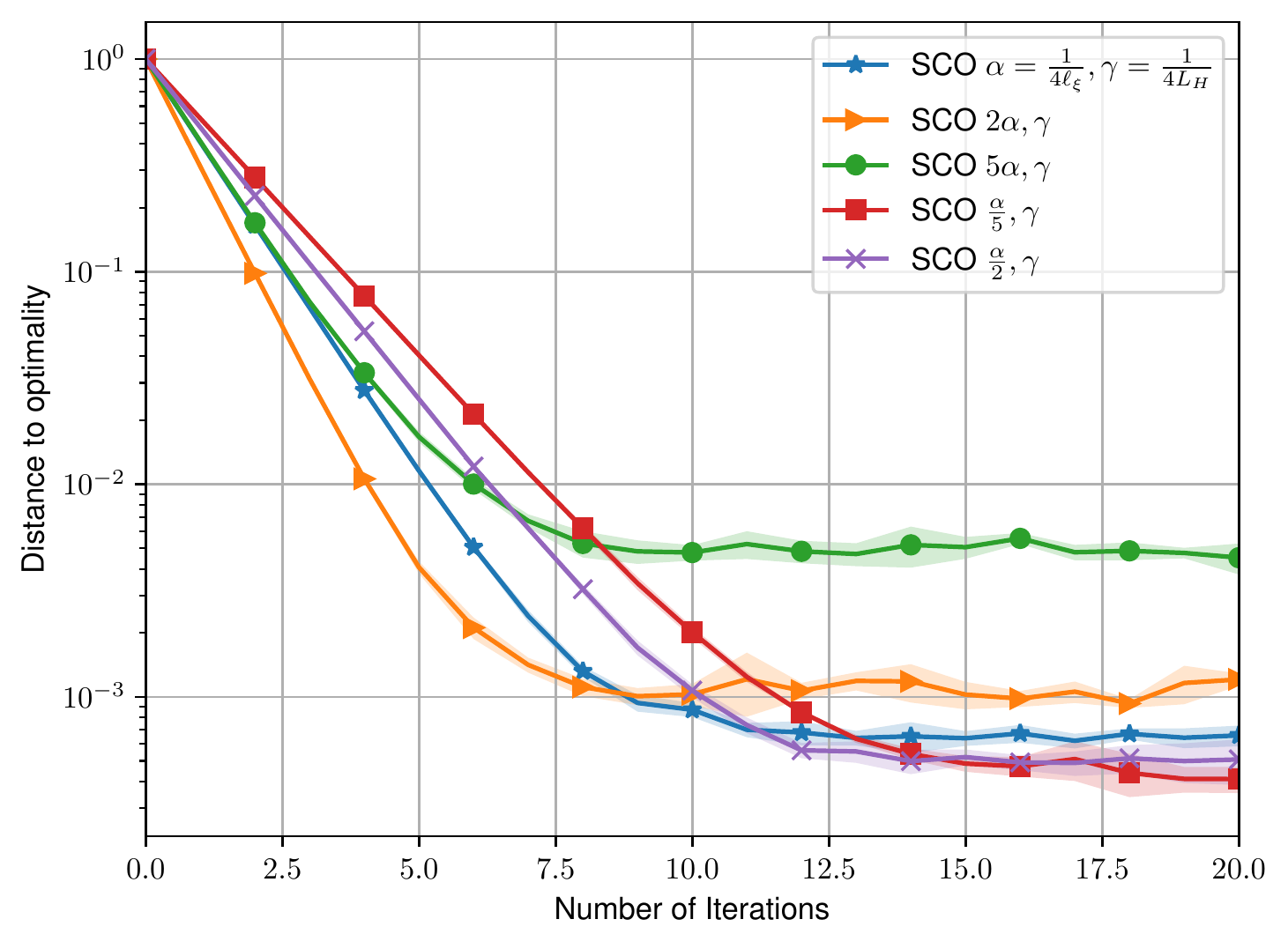}
    \caption{SCO varying only $\alpha$}
  \end{subfigure}
  \caption{Performance of the methods, SGDA, SHGD and SCO with different constant step-sizes. Distance to optimality: $\frac{\|x^k-x^*\|^2}{\|x^0-x^*\|^2}$. The step-sizes used are expressed as the optimal theoretical step-size predicted by the theory times a constant. The step-sizes for which the methods diverge are not included in the plots.}
  \label{fig:lr_grid}
\end{figure}
\newpage
\section{Beyond Finite-sum Structure}
\label{Appendix_BeyondFiniteSum}

We explain in this section how all our convergence results also hold for the more general (non finite-sum) stochastic approximation setting~\citep{Nemirovski-Juditsky-Lan-Shapiro-2009}, where we define $\xi: \R^d \rightarrow \R^d$ as:
\begin{equation} \label{eq:stochasticXi}
\xi(x) := \Exp_U \tilde{\xi}(x, U) ,
\end{equation}
where $U$ is a random vector in $\R^p$, and $\tilde{\xi} : \R^d\times \R^p \rightarrow \R^d$ is used to define $\xi$ and is assumed to be well-behaved enough so that~\eqref{eq:stochasticXi} exists for all $x$. Setting $U$ to be a uniform random variable in $\{1, \ldots, n\}$ gives back the finite sum setting with uniform weights which was covered in the paper in~\eqref{VI}.

We present here only results for the singleton sampling regime, as generalizing to mini-batching and other sampling schemes in the continuous regime is non-trivial and beyond the scope of this paper. This means that we restrict here our estimator for $\xi(x)$ to simply be $\tilde{\xi}(x,U)$, with $U$ sampled according to its distribution. Any appearance of $\xi_v$ in the paper can then be replaced with $\tilde{\xi}(x,U)$ to be able to re-interpret the algorithms and proofs in this setting. In particular, we have that our estimator is  $g(x) = \xi_v(x) := \tilde{\xi}(x,U)$, and so $\Exp[ g(x) ] = \Exp_U \tilde{\xi}(x,U) = \xi(x)$ trivially. Also, the expected co-coercivity of $\xi$ is defined as the existence of a $\ell_\xi > 0$ such that
\begin{equation} \label{eq:EC_stochastic}
\Exp_U \left[ \| \tilde{\xi}(x,U) - \tilde{\xi}(x^*,U)\|^2 \right] \leq \ell_\xi \langle \xi(x) , x - x^* \rangle \quad \forall x \in \R^d \, .
\end{equation}
An important quantity appearing in our convergence results is 
\begin{equation}
\sigma^2 := \Exp_U \| \tilde{\xi}(x,U) \|^2 \, ,
\end{equation}
and we assume that $U$ and $\tilde{\xi}$ are such that $\sigma^2 < \infty$.\footnote{In the finite sum setting, this is always true. But for general random variable $U$, this is not always the case.} Similarly, $\sigma_\cH^2$ is defined by letting $u$ and $v$ represent independent singleton sampling in~\eqref{StochasticHamiltonianGradient}, and is assumed to be finite.

Under the assumption of expected co-coercivity of~$\xi$ and that $\sigma^2$ is finite, all the convergence theorems of Section~\ref{Section_SGDA} and~\ref{sec:SCO} hold as is for $\xi$ more generally defined by~\eqref{eq:stochasticXi}. This is because all the convergence proofs did not use the finite sum structure explicitly, only the linearity of the expectation operator.

Finally, we can easily generalize Proposition~\ref{PropositionMinibatch} for $b=1$ with $\ell_\xi = \ell_{\max}$ (we assume each $\tilde{\xi}(x,u)$ is $\ell_u$-co-coercive in $x$ around $x^*$, and assume that $\ell_{\max} := \sup_{u} \ell_u$ is finite).

We can also generalize Proposition~\ref{PropositionExtra} (for singleton sampling) with $\ell_\xi := \Exp_U L_U^2 / \mu$ (we assume each $\tilde{\xi}(x,u)$ is $L_u$-Lipschitz continuous in $x$, for each $u$; and that $\Exp_U L_U^2$ is finite).

\section{More Related Work}
\label{Appendix_MoreRelatedWork}

The references necessary to motivate our work and connect it to the most relevant literature is included in the appropriate sections of the main body of the paper. Here we present a broader view of the literature, including some more references to papers of the area that are not directly related with our work.

\paragraph{Smooth monotone games.} 
The deterministic version of the problem has been studied extensively both in past and recent work, initially focusing on strongly monotone problems
\citep{tseng1995linear, gidel2018variational, liang2019interaction, azizian2019tight, zhou2021robust}.
\citet{mokhtari2020unified} recently gave rates for extragradient and optimistic gradient through a proximal point approach. 

For monotone problems, in the absence of strong monotonicity but assuming a lower bound of the singular values of the coupling between players, \citet{azizian2019tight} produce tight results for the extragradient and optimistic gradient methods.

For general monotone problems, \citet{mertikopoulos2019games} establish last-iterate convergence, but requires a decreasing step-size schedule; it also does not guarantee convergence to fixed points of non-strictly monotone problems like bilinears. 	
\citet{golowich2020last} show that last-iterate methods are slower than averaging in this setting.
Overcoming the above limitations, \citet{golowich2020last, golowich2020tight} establish tight upper bounds, $O(1/\sqrt{T})$, for general monotone problems under a weak smoothness assumption for extra-gradient and optimistic gradient respectively.

\paragraph{Stochastic smooth monotone games.} 
Contrary to the deterministic family of problems, in the stochastic setting (where available gradients are noisy), progress has occurred mostly in the past five years. 
Without variance reduction, some work establishes last-iterate convergence results with a slow (sublinear) rate \citep{rosasco2014stochastic, mishchenko2020revisiting}.
For {\em pseudomonotone problems}, \citet{kannan2019optimal} give last-iterate convergence of stochastic extragradient with decreasing step sizes, assuming a uniform bound on the variance.
Using variance reduction techniques, \citet{palaniappan2016stochastic} yield fast linear rates without requiring a bound on the variance, however the step size and inner loop length need to be tuned using the modulus of local strong monotonicity. This limitation was later lifted with a variance reduced extragradient method proposed by  \citet{chavdarova2019reducing}, which is adaptive to local strong monotonicity.

For monotone problems, in the absence of strong monotonicity but assuming a lower bound of the singular values of the coupling between players, \citet{chavdarova2019reducing} show that stochastic extragradient with constant step-size does not always yield last-iterate convergence.
Such a result for extragradient was shown to be possible using a {\em double step-size scheme} in \citet{hsieh2020explore}, using one step size for the extrapolation step and a different step size for the update step. 

For general monotone problems, \citet{mertikopoulos2019games} use a dual averaging approach to get last-iterate convergence for no-regret algorithms making a bounded variance assumption.
This bounded variance assumption is lifted in \citet{lin2020finite}; the authors give finite-time last-iterate convergence of optimistic gradient under other strong conditions: 
i) for a {\em relative random noise} model, the noise variance is assumed to be proportional to the squared norm of the vector of gradients, according to the sequence of multiplicative coefficients, $\tau_k$. That is, $\Exp_i\|\xi_i(x_k)-\xi(x_k)\|^2 < \tau_k \|\xi(x_k)\|^2$.
This kind of condition with a constant $\tau_k=\tau$ can occur in some machine learning problems, in particular in supervised learning learning problems when the model is overparametrized and capable to perfectly fitting all training points (so the gradient variance becomes zero at the stationary point of the solution). 
On the other hand, this assumption is satisfied rarely for adversarial formulations.
Even more, Theorem 4.6 in \citet{lin2020finite} requires that the sequence, $\tau_k$, goes to zero with $k$ in order to get convergence;
ii) for an {\em absolute random noise} model, noise variance is uniformly bounded by $\sigma_k$, a bounded variance assumption. 
Again, for the result of Theorem A.4 in \citet{lin2020finite} to hold, the sequence $\sigma_k$ is assumed to go to zero with $k$.
To the best of our knowledge, both assumptions are very strong: 
in order to get either $\tau_k$ or $\sigma_k$ to decay, one will have to use a mini-batch size that grows up to $n$, which is impractical.
It should be noted that the same work contains almost sure (non finite-time) last-iterate results without the above restrictive assumptions.

\paragraph{Structured non-monotone problems.} The recent works of \cite{daskalakis2021complexity} and \cite{diakonikolas2021efficient} show that, for general smooth objectives, the computation of even approximate first-order locally optimal min-max solutions is intractable, motivating the identification of structural assumptions on the objective function for which these intractability barriers can be bypassed.

In this work we focus on a setting (structured non-monotone operators) that one can provide tight convergence guarantees and avoid the standard issues (cycling and divergence of the methods) appearing in the more general non-monotone regime, like . That is, problems satisfying quasi-strong monotonicity \eqref{QSM} or the variational stability condition.

The two classes of VI problems we consider, quasi-strongly monotone and problems satisfying the variational stability condition, have already been used in several papers (under different names). For example, \cite{song2020optimistic} also focuses on quasi-strongly monotone VI but they called them strong coherent VI and show that quasi strong monotonicity is weaker than the strongly pseudo-monotone \citep{kannan2019optimal} and strongly monotone assumption. In addition, as presented in \cite{song2020optimistic}, the non-monotone subsets of quasi-strong monotonicity (a.k.a., strongly pseudomonotone), have many real applications in competitive exchange economy \citep{brighi2002characterizations}, fractional programming \citep{elizarov2009maximization,rousseau2005trade}, and product pricing \citep{choi1990product}. Meanwhile, the restriction of quasi-strong monotonicity in minimization problems such as one-point convexity \citep{li2017convergence} is also used in analyzing neural networks.

In addition in \cite{zhou2017stochastic}, the strongly coherent optimization problems have been studied, which are special cases of the quasi-strongly monotone VI. We refer the interested reader to \cite{zhou2017stochastic} for a list of examples of strongly coherent optimization problems. As we mentioned in the main paper, the assumption of quasi-strong monotonicity is equivalent to the concept of strong stability described in \cite{mertikopoulos2019games}.

\end{document}